\documentclass[runningheads]{llncs}
\usepackage{graphicx}
\usepackage{amsmath}
\usepackage{amsfonts}
\usepackage{amssymb}
\usepackage{stmaryrd}
\usepackage[T1]{fontenc}
\usepackage{xcolor}
\usepackage[square,comma,numbers,sort&compress]{natbib}
\usepackage{algorithm}
\usepackage{algorithmic}
\usepackage{enumerate}
\usepackage{tabularray}
\usepackage{hyperref}

\newcommand{\repr}[1]{\llbracket #1\rrbracket}

\begin{document}
\title{Automata Learning from Preference and Equivalence Queries}

\author{Eric Hsiung\orcidID{0000-0002-4188-4127} \and
Joydeep Biswas\orcidID{0000-0002-1211-1731} \and
Swarat Chaudhuri\orcidID{0000-0002-6859-1391}}
\institute{Univesity of Texas at Austin, Austin TX 78712, USA\\
\email{\{ehsiung,joydeepb,swarat\}@cs.utexas.edu}}

\authorrunning{E. Hsiung et al.}

\maketitle

\newcommand{\ouralgorithm}{\textsc{Remap}}
\newcommand{\obstable}{\langle S, E, T; \mathcal{C}, \Gamma\rangle}
\newcommand{\boldemph}[1]{\textbf{\emph{#1}}}
\newcommand{\orow}[1]{\textbf{\textit{row}}(#1)}
\newcommand{\orows}[1]{\textbf{\textit{rows}}(#1)}
\newcommand{\lstar}{\ensuremath{\textrm{L}^*}}
\begin{abstract}
Active automata learning from membership and equivalence queries is a foundational problem with numerous applications. We propose a novel variant of the active automata learning problem:
actively learn finite automata using \emph{preference queries}---i.e., queries about the relative position of two sequences in a total preorder---instead of membership queries. 
Our solution is \ouralgorithm{}, a novel algorithm 
which leverages a symbolic observation table along with unification and constraint solving to navigate a space of symbolic hypotheses (each representing a set of automata), and uses satisfiability-solving to construct a concrete automaton (specifically a Moore machine) from a symbolic hypothesis.
\ouralgorithm{} is guaranteed to correctly infer the minimal automaton with polynomial query complexity under exact equivalence queries, and achieves PAC--identification ($\varepsilon$-approximate, with high probability) of the minimal automaton using sampling-based equivalence queries.
Our empirical evaluations of \ouralgorithm{} on the task of learning reward machines for two reinforcement learning domains indicate \ouralgorithm{} scales to large automata and is effective at learning correct automata from consistent teachers, under both exact and sampling-based equivalence queries.
\end{abstract}

\section{Introduction}

Active automata learning has applications from software engineering~\cite{Schuts2016Refactoring,Aarts2012LearningAT} and verification~\cite{Lin2014Learning} to interpretable machine learning~\cite{Weiss2019} and learning reward machines~\cite{tappler2019based,GaonB20_nonmarkovian,xu_lstar,dohmen-2022-icaps}. The classical problem formulation involves a teacher with access to a regular language and a learner which asks \emph{membership} and \emph{equivalence} queries to infer a finite automaton describing the regular language~\cite{Angluin87,Gold1978ComplexityOA}. 

Consider an alternative formulation: \emph{learning a finite automaton from 
preference and 
equivalence queries}.\footnote{\citet{shah2023learning} investigates choosing between membership and preference queries.} A preference query resolves the relative position of two sequences in a total preorder available to the teacher. The motivation for learning from preferences stems from leveraging human preferences as a rich source of information. In fact, comparative feedback, such as preferences, is a modality which humans are apt at providing in comparison to giving specific numerical values, as shown by \citet{macglashan2017coach} and \citet{christiano2017}, likely due to \textit{choice overload} \cite{iyengar2000choice}. But preferences need not be directly obtained from individual humans: preferences can also be obtained from automated systems in a frequency-based sense. For example, consider an obstacle avoidance scenario of a vehicle avoiding large debris in the middle of a roadway. Consider a dataset of a population of driver responses (vehicle trajectories) to such a scenario, procured from dashcam or traffic camera footage. The population preference of a given driving response can be automatically determined by ranking each response by the frequency of occurrence in the dataset. In fact, learning from human preference data has applications in fine-tuning language models~\cite{Ouyang2022TrainingLM}, learning conditional preference networks~\cite{Koriche2009LearningCP, Guerin2013LearningCP}, inferring reinforcement learning policies~\cite{Christiano2017DeepRL}, and learning Markovian reward functions~\cite{biyki2022aprel,Sadigh2017ActivePL,Bewley2021InterpretablePR, Kalra2023CanDD}. Possible applications for learning finite automata from preferences over sequences include inferring sequence classifications (e.g. program executions~\cite{Giannakopoulou2012SymbolicLO}, vehicle maneuvers, human-robot interactions) using ordered classes (e.g. \emph{safe, risky, dangerous, fatal}) with each automaton state labeled by a class, 
distilling interpretable preference models~\cite{Weiss2019}, and inferring reward machines. However, no method currently exists for learning finite automata from preferences with the termination and correctness guarantees enjoyed by classical automata learning algorithms such as \lstar{}~\cite{Angluin87}.
\begin{figure}[t]
    \centering
    \includegraphics[width=0.65\columnwidth]{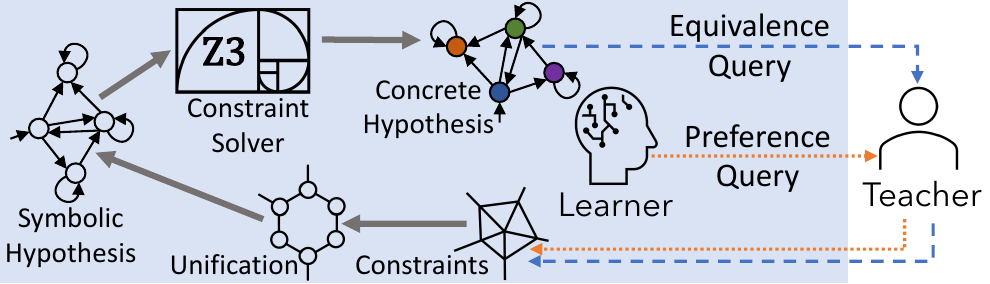}
    \caption{REMAP Algorithm Overview}
    \label{fig:diagram}
\end{figure}

Unfortunately, adapting \lstar{} to the preference-based setting is challenging. Preference queries do not directly provide the concrete observations available from membership queries, as required by \lstar{}, so our solution \ouralgorithm{} addresses this challenge through a \emph{symbolic approach} to \lstar{}.

\ouralgorithm{}\footnote{Code available at \url{https://eric-hsiung.github.io/remap/}, as well as supplementary material.} (Figure \ref{fig:diagram}) features termination and correctness guarantees for \emph{exact} and \emph{probably approximately correct} \cite{valiant1984} (PAC) \emph{identification} \cite{angluin1988queries} of the desired automaton, with a strong learner capable of symbolic reasoning and constraint-solving to offset the weaker preference-based signals from the teacher. While approximate learners have been proposed to learn from a combination of membership and preference queries~\cite{shah2023learning}, \ouralgorithm{} is the first exact learner with formal guarantees of minimalism, correctness, and query complexity. By using unification\footnote{Symbolically obtaining sets of equivalent variables from sets of equations, and substituting variables in the equations with their representatives. See Section \hyperref[sec:unification]{4.2}.} to navigate the symbolic space of hypotheses, and constraint-solving to construct a concrete automaton from a symbolic hypothesis, \ouralgorithm{} identifies, in a polynomial number of queries, the minimal Moore machine isomorphic to one describing the teacher’s total preorder over input sequences when using \emph{exact} equivalence queries. However, exact equivalence queries may be infeasible if the learner and teacher lack a common representation.\footnote{Either a pair of identical representations, or between which a translation exists.} This motivates \emph{sampling-based} equivalence queries under \ouralgorithm{}, which achieves \emph{PAC--identification}.\footnote{See Definition \ref{def:pacid} for PAC--identification and related Theorems \ref{thm:pacid1} and \ref{thm:pacid2}.} Our empirical evaluations apply \ouralgorithm{} to learn reward machines for sequential-decision making domains in the reward machine literature. We measure query complexity for the exact and PAC--identification settings, and measure empirical correctness for the PAC--identification setting.

We contribute (a) \ouralgorithm{}, a novel \lstar{} style algorithm for learning Moore machines using preference and equivalence queries; under \emph{exact} and \textit{PAC-identifica-tion} settings we provide (b) theoretical analysis of query complexity, correctness, and minimalism, and (c) supporting empirical results, demonstrating the efficacy of the algorithm's ability to learn reward machines from preference and equivalence queries.

\section{Background}
Prior to introducing \ouralgorithm{}, we provide preliminaries on orders and finite automata, followed by a discussion of Angluin's \lstar{} algorithm in order to highlight how \ouralgorithm{} differs from \lstar{}.

\begin{definition}[Total Preorder] The binary relation $\precsim$ on a set $A$ is a \emph{total preorder} if $\precsim$ satisfies: $a\precsim a$ for all $a\in A$ \emph{(reflexivity)}; $a\precsim b \land b\precsim c \implies a\prec c$ for all $a,b,c \in A$ \emph{(transitivity)}; and $a\precsim b \lor b\precsim a$ for all $a,b\in A$ \emph{(total)}.
\end{definition}
\begin{definition}[Total Order]  The binary relation $\precsim$ on a set $A$ is a \emph{total order} if $\precsim$ satisfies: total preorder conditions; and $a\precsim b \land b\precsim a\implies a=b$ for all $a,b\in A$ \emph{(antisymmetric)}; $a=b$ means $a$ and $b$ are the same element from $A$.
\end{definition}

\boldemph{Finite Automata} \ Automata describe sets of sequences. An \textit{alphabet} $\Sigma$ is a set whose elements can be used to construct sequences; $\Sigma^*$ represents the set of sequences of any length created from elements of $\Sigma$. A sequence $s\in \Sigma^*$ has integer length $|s|\geq 0$; if $|s|=0$, then $s$ is the empty sequence $\varepsilon$. An element $\sigma\in\Sigma$ has length 1; if $s,t\in\Sigma^*$, then $s\cdot t$ represents $s$ concatenated with $t$, with length $|s|+|t|$. 
Different types of finite automata have different semantics.
Deterministic automata feature a deterministic transition function $\delta$ defined over a set of states $Q$ and an input alphabet $\Sigma^I$. An output alphabet $\Sigma^O$ may be present to label states or transitions using a labeling function $L$. 
The tuple $\langle Q, q_0, \Sigma^I, \Sigma^O, \delta, L\rangle$ is a \textit{Moore machine}~\cite{Moore56}, where $q_0 \in Q$ is the initial state, $\delta: Q\times \Sigma^I \rightarrow Q$ describes transitions, and $L: Q\rightarrow \Sigma^O$ associates outputs with states. Extended, $\delta: Q\times (\Sigma^I)^*\rightarrow Q$, where $\delta(q,\varepsilon)=q$ and $\delta(q,\sigma\cdot s)=\delta(\delta(q,\sigma), s)$. In \textit{Mealy machines}~\cite{mealy}, outputs are associated with transitions instead, so $L: Q\times \Sigma^I\rightarrow \Sigma^O$. Mealy and Moore machines are equivalent~\cite{FLEISCHNER19771} and can be converted between one another.
Reward machines are an application of Mealy machines in reinforcement learning and are used to express a class of non-Markovian reward functions.

\boldemph{Active Automaton Learning} \ Consider the problem of actively learning a Moore machine $\langle Q, q_0, \Sigma^I, \Sigma^O, \delta, L\rangle$ to exactly model a function $f: (\Sigma^I)^*\rightarrow \Sigma^O$, where $\Sigma^I$ and $\Sigma^O$ are input and output alphabets of finite size known to both teacher and learner. We desire a learner which learns a model $\hat{f}$ of $f$ exactly; that is for all $s\in (\Sigma^I)^*$, we require $\hat{f}(s)=f(s)$, where $\hat{f}(s)=L(\delta(q_0,s))$, with the assistance of a teacher $\mathcal{T}$ that can answer questions about $f$.

In Angluin's seminal active learning algorithm, \lstar{}~\cite{Angluin87}, 
the learner learns $\hat{f}$ as a binary classifier, where $|\Sigma^O|=2$, to determine sequence membership of a regular language by querying $\mathcal{T}$ with: \textit{(i) membership queries}, where the learner asks $\mathcal{T}$ for the value of $f(s)$ for a particular sequence $s$, and \textit{(ii) equivalence queries}, where the learner asks $\mathcal{T}$ to evaluate whether for all $s\in(\Sigma^I)^*$, $\hat{f}(s)=f(s)$. For the latter query, $\mathcal{T}$ returns True if the statement holds; otherwise a counterexample $c$ for which~$\hat{f}(c)\neq f(c)$ is returned. An \textit{observation table} $\langle S, E, T\rangle$ records the concrete observations acquired by the learner's queries (see Figure \ref{fig:example-lstar}). Here, $S$ is a set of prefixes, $E$ is a set of suffixes, and $T$ is the empirical observation function that maps sequences to output values---$T: (S \cup (S \cdot \Sigma^I))\cdot E \rightarrow \Sigma^O$. The observation table is a two-dimensional array; $s\in(S\cup (S\cdot \Sigma^I))$ indexes rows, and $e\in E$ indexes columns, with entries given by $T(s\cdot e)$. Any proposed hypothesis $\hat{f}$ must be consistent with $T$. The algorithm operates by construction: a deterministic transition function must be found exhibiting consistency (deterministic transitions) and operates in a closed manner over the set of states. If the consistency or closure requirements are violated, then membership queries are executed to expand the observation table. Once a suitable transition function is found, a hypothesis $\hat{f}$ can be made and checked via the equivalence query. The algorithm terminates if $\hat{f}(s)=f(s)$ for all $s\in(\Sigma^I)^*$; otherwise \lstar{} continues on by adding counterexample $c$ and all its prefixes to the table, and finds another transition function satisfying the consistency and closure requirements.
\begin{figure*}[t]
    \centering
    \includegraphics[width=0.75\textwidth]{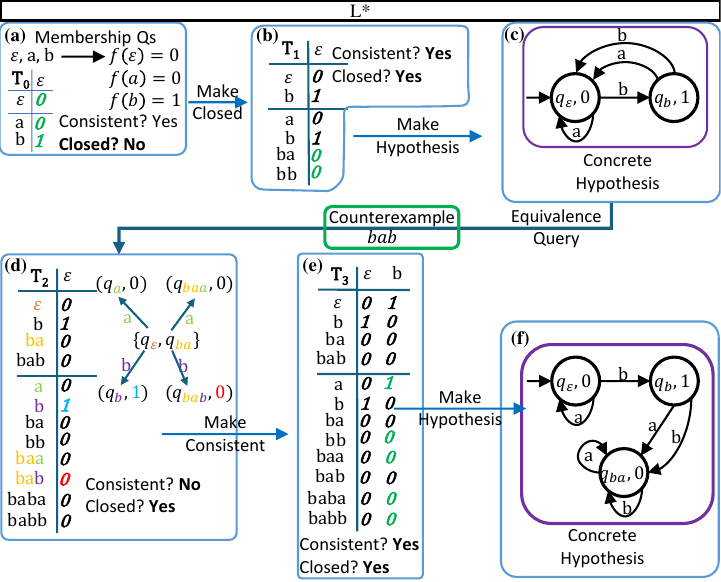}\\
    \caption{\lstar{} example. Figure \ref{fig:example} illustrates \ouralgorithm{}. \lstar{} records concrete values from membership queries in observation table. Green symbolizes new information. Colors in \emph{(d)} visually highlight transition inconsistencies.}
    \label{fig:example-lstar}
\end{figure*}

Consequently, we consider how \lstar{}-style learning can be used to learn Moore machines from \emph{preference queries} over sequences, as a foray into understanding how finite automaton structure can be learned from comparison information.

\section{Problem Statement} We consider how a Moore machine $\langle Q, q_0, \Sigma^I, \Sigma^O, \delta, L\rangle$ can be actively learned from \emph{preference queries} over $(\Sigma^I)^*$. 
We focus on the case of a finite sized $\Sigma^I$ and $\Sigma^O$ known to both teacher and learner. 
With a \emph{preference query}, the learner asks the teacher $\mathcal{T}$ which of two sequences $s_1$ and $s_2$ is preferred, or if both are equally preferable.
This requires a preference model, which we assume represents a total preordering over $(\Sigma^I)^*$; we also assume $\Sigma^O$ is totally ordered. Thus, we consider a preference model for $\mathcal{T}$ where $f:(\Sigma^I)^*\rightarrow\Sigma^O$ is consistent with both orderings, i.e., $\mathcal{T}$ prefers $s_1$ over $s_2$ if $f(s_1)>f(s_2)$, or otherwise has equal preference if $f(s_1)=f(s_2)$.

Several options for evaluating hypothesis equivalence (is $\hat{f}\equiv f$?) can be defined for this problem formulation. We first review the definition of exact equivalence, followed by an alternative notion of equivalence which respects ordering:
\begin{definition}[Exact Equivalence] Given a hypothesis $\hat{f}$ and a reference $f$, $\hat{f}$ is \emph{exactly equivalent} to $f$ if $\hat{f}(s)=f(s)$ for all $s$ in $(\Sigma^I)^*$.
\end{definition}
\begin{definition}[Order-Respecting Equivalence] Given a hypothesis $\hat{f}$ and a reference $f$, $\hat{f}$ is \emph{order-respecting equivalent} to $f$ if, for all $s,t$ in $(\Sigma^I)^*$, there exists a relation $\mathbf{R}_{s,t}\in\{=,>,<\}$ such that $\hat{f}(s)\mathbf{R}_{s,t}\hat{f}(t) \iff f(s)\mathbf{R}_{s,t}f(t)$.
\end{definition}
%The standard option requires exact equivalence: $\hat{f}(s)=f(s)$ for all $s$ in $(\Sigma^I)^*$. An alternative requires the total ordering implied by $\hat{f}$ and $f$ to match. That is, for all $s,t$ in $(\Sigma^I)^*$, both $f(s)\mathbf{R}_{s,t}f(t)$ and $\hat{f}(s)\mathbf{R}_{s,t}\hat{f}(t)$ are satisfied, where $\mathbf{R}_{s,t}\in\{=,>,<\}$.
Since a hypothesis $\hat{f}$ satisfying exact equivalence is also always order-respect-ing, we focus on exact equivalence queries where $\mathcal{T}$ returns feedback along with counterexample $c$. 
\begin{definition}[Exact Equivalence Query with Feedback]\label{exact-eq-def} Given a hypothesis $\hat{f}$ and a reference $f$, an exact equivalence query EQ returns the following triple: \[\text{\boldemph{EQ}}(\hat{f})=\langle\forall s\in(\Sigma^I)^*:\hat{f}(s)=f(s), c : \exists c \text{ \emph{s.t.} } \hat{f}(c)\neq f(c), \phi(c)\rangle\] where $c$ is a counterexample if one exists, and $\phi(c)$ is feedback associated with the counterexample, interpreted as a constraint.
\end{definition}

The strength of feedback $\phi$ directly impacts how many hypotheses per counterexample can be eliminated by the learner: returning $\phi(c):=\hat{f}(c)\neq f(c)$, which means the value $f(c)$ is not $\hat{f}(c)$, is weak feedback compared to returning $\phi(c):=\hat{f}(c)=f(c)$, which means the value of $\hat{f}(c)$ should be $f(c)$, or returning $\phi(c):=\hat{f}(c)\in X$ (where $X\subset\Sigma^O$, which means the value of $\hat{f}(c)$ should be in a subset of $\Sigma^O$). 
Although \ouralgorithm{} outputs a concrete Moore machine, \ouralgorithm{} navigates symbolic Moore machine\footnote{i.e., a Moore machine with symbolic values as outputs.} space using solely preference information, while concrete information assists in selecting the concrete hypothesis. 
Thus, our theoretical analysis 
considers equivalence queries that provide counterexamples with strong feedback: $\phi(c) := \hat{f}(c) = f(c)$.

\section{The \ouralgorithm{} Algorithm}
\ouralgorithm{} is a \lstar{}-based algorithm employing preference and equivalence queries to gather constraints. By first leveraging unification to navigate the symbolic hypothesis space of Moore machines, solving the constraints yields a concrete Moore machine. In particular, \ouralgorithm{} (Algorithm~\ref{alg:pref-demo}) learns a Moore machine $\langle \hat{Q}, \hat{q}_0, \Sigma^I, \Sigma^O, \hat{\delta}, \hat{L}\rangle$, representing a multiclass classifier $\hat{f}(s)=\hat{L}(\hat{\delta}(\hat{q}_0,s))$.
\begin{figure*}[tp]
    \centering
    \includegraphics[width=0.9\textwidth]{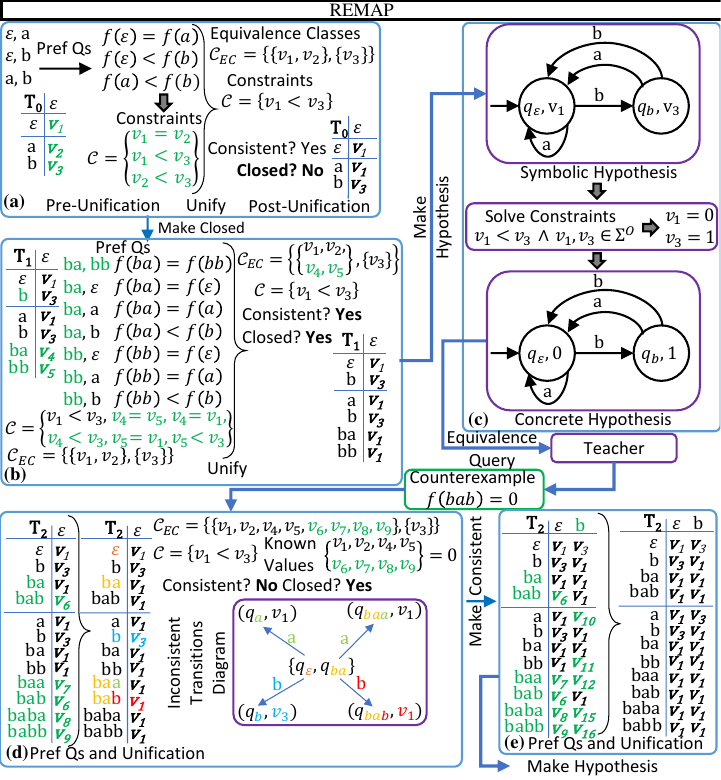}\\
    \caption{\ouralgorithm{} example. Figure \ref{fig:example-lstar} illustrates \lstar{}. \ouralgorithm{} performs preference queries and records variables in symbolic observation table $\obstable$ within a \textsc{SymbolicFill}. \emph{(a)} initializes $\obstable$; \emph{(b)} expands $S$ with sequence $b$, followed by \textsc{SymbolicFill}; yields unified, closed, and consistent $\obstable$; \emph{(c)} \textsc{MakeHypothesis} yields concrete hypothesis $h_1$ from symbolic hypothesis and constraints in $\mathcal{C}$; \emph{(d)} submits $h_1$ via equivalence query; receives and processes counterexample $bab$ with feedback $f(bab)=0$, adding $bab$, $ba$, and $b$ to $S$, performs a \textsc{SymbolicFill}, sets the value of equivalence class for $T(bab)$ to $f(bab)$; yields an inconsistent table. \emph{(e)} shows resulting consistent table. Figure \ref{fig:example-lstar}f shows ground truth Moore machine with $\Sigma^I=\{a,b\}, \Sigma^O=\{0,1\}$. Green symbolizes new information. Colors in \emph{(d)} visually highlight transition inconsistencies.}
    \label{fig:example}
\end{figure*}
Central to \ouralgorithm{} are four core components: \textbf{(1)} a new construct called a \textit{symbolic observation table} (Section \hyperref[sec:obstable]{4.1}), shown in Figure \ref{fig:example}a, \ref{fig:example}b, \ref{fig:example}d, and \ref{fig:example}e; \textbf{(2)} as well as an associated algorithm for \textit{unification}~\cite{Robinson1965AML} (Section \hyperref[sec:unification]{4.2}) inspired by~\citet{Martelli1976UnificationIL} to contain the fresh variable explosion; both enable the learner to \textbf{(3)} generate symbolic hypotheses (Section \hyperref[sec:symbolic-hypothesis]{4.3}) purely from observed symbolic constraints, along with \textbf{(4)} a constraint solver for obtaining a concrete hypothesis (Section \hyperref[sec:concrete-hypothesis]{4.4}). We first discuss the four components and how they fit together into \ouralgorithm{} (Section \hyperref[sec:remap]{4.5}). Afterwards, we illustrate the correctness and termination guarantees.

\subsubsection{4.1 Symbolic Observation Tables}\label{sec:obstable} Recall that in classic \lstar{}, the observation table entries are concrete values obtained from membership queries (Figure \ref{fig:example-lstar}a). However, observations obtained from \textit{preference queries} are constraints, rather than concrete values, indicating that for a pair of sequences $s_1$ and $s_2$, one of $f(s_1)>f(s_2)$, $f(s_1)<f(s_2)$, or $f(s_1)=f(s_2)$ holds. We therefore introduce the \boldemph{symbolic observation table} $\obstable$, where $S$ is a set of prefixes, $E$ is a set of suffixes, $\mathcal{C}$ is the set of known constraints, $\Gamma$ is a \boldemph{context} which uniquely maps sequences to variables, with the set of variables $\mathcal{V}$ in the context $\Gamma$ given by $\mathcal{V}=\{\Gamma[s\cdot e]|s\cdot e \in (S\cup (S \cdot \Sigma^I))\cdot E\}$, and $T: (S \cup (S \cdot \Sigma^I))\cdot E \rightarrow \mathcal{V}$ maps queried sequences to variables. Thus, in a symbolic observation table, the entry for each prefix-suffix pair is a variable, rather than a concrete value; and the constraints over those variables are stored in $\mathcal{C}$ (Figure \ref{fig:example}a, \ref{fig:example}b, \ref{fig:example}d, \ref{fig:example}e). The constraints from the preferences of $\mathcal{T}$ over $f$ about $s_1$ and $s_2$ correspond with $T(s_1)>T(s_2)$, $T(s_1)<T(s_2)$, or $T(s_1)=T(s_2)$, respectively.
\begin{definition} An \boldemph{equivalence class} $\mathbb{C}$ of variables is a set with the property that all members of $\mathbb{C}$ are equivalent to each other. The \boldemph{representative} $\repr{\mathbb{C}}$ of $\mathbb{C}$ is a deterministically elected member of $\mathbb{C}$. The set of variables $\mathcal{V}$ can be partitioned into disjoint equivalence classes. The set of equivalence classes $\mathcal{C}_{EC}$ corresponds with the partitioning of $\mathcal{V}$ into the smallest possible number of equivalence classes consistent with equality constraints. Let $\mathcal{R}=\{\repr{\mathbb{C}}|\mathbb{C}\in\mathcal{C}_{EC}\}$ be the set of representatives.
\end{definition}

\subsubsection{4.2 Unification and Constraints}\label{sec:unification} 
Since \textit{preference queries} return observations comparing the values of two sequences, we leverage a simple unification algorithm to ensure the number of unique variables in the table remains bounded by $|\Sigma^O|$. As a reminder, unification involves symbolically solving for representatives and rewriting all constraints and variables in terms of those representatives. Whenever we observe the constraint $f(s_1)=f(s_2)$, this implies $T(s_1)=T(s_2)$, so we add $\Gamma[s_1]$ and $\Gamma[s_2]$ to an \boldemph{equivalence class} $\mathbb{C}$, and elect a \boldemph{representative} from $\mathbb{C}$. The unification algorithm is presented in Appendix Algorithm \ref{alg:unification-a}.

If $\Gamma[s_1]$ already belongs to an equivalence class $\mathbb{C}\in\mathcal{C}_{EC}$, but $\Gamma[s_2]$ is an \textit{orphan variable}---belonging to no class in $\mathcal{C}_{EC}$---then $\Gamma[s_2]$ is merged into $\mathbb{C}$. Swap $s_1$ and $s_2$ for the other case. If $\Gamma[s_1]$ and $\Gamma[s_2]$ belong to separate classes $\mathbb{C}_1$ and $\mathbb{C}_2$, with $\mathbb{C}_1\neq\mathbb{C}_2$, then $\mathbb{C}_1$ and $\mathbb{C}_2$ are merged into one via $\mathbb{C}\leftarrow\mathbb{C}_1\cup\mathbb{C}_2$, and one of $\repr{\mathbb{C}_1}$ and $\repr{\mathbb{C}_2}$ is deterministically elected as the representative~$\repr{\mathbb{C}}$.

When this unification process is applied to a symbolic observation table $\obstable$, we are unifying $\mathcal{V}$ (the set of variables in the context $\Gamma$) according to the set of known equivalence constraints $\mathcal{C}_{EQ}$ and equivalence classes $\mathcal{C}_{EC}$ available in $\mathcal{C}$; and each variable in the table is replaced with its equivalence class representative (see Figure \ref{fig:example}abde after the large right curly brace). That is, for all $s$ in $(S\cup S\cdot \Sigma^I)\cdot E$, we substitute $T(s)\leftarrow\repr{\mathbb{C}}$ if $T(s)\in\mathbb{C}$. In the resulting \textit{unified} symbolic observation table, $T$ maps queried sequences to~the set of representatives $\mathcal{R}$. Note $|\mathcal{R}|=|\mathcal{C}_{EC}|\leq |\Sigma^O|$.

Besides constraints from preferences queries, the learner obtains constraints about the value of $T(c)$ from \textit{equivalence queries} (Figure \ref{fig:example}c to \ref{fig:example}d) and adds them to $\mathcal{C}$. 
When first obtained, these constraints may possibly be expressed in terms of orphan variables, but during the process of unification, each orphan variable joins an equivalence class and is then replaced by its equivalence class representative in the constraint. Thus, after unification, all constraints in $\mathcal{C}$ are expressed in terms of equivalence class representatives. 

Finally, unifying the symbolic observation table is critical for making a symbolic hypothesis without knowledge of concrete value assignments: unification permits $\obstable$ to become \textit{closed} and \textit{consistent}---prerequisites for generating a symbolic hypothesis.
\begin{definition} Let $\orows{S}=\{\orow{s}|s\in S\}$, where the row in $\obstable$ indexed by $s$ is $\orow{s}$. $\obstable$ is \boldemph{closed} if $\orows{S\cdot \Sigma^I} \subseteq \orows{S}$.
\end{definition}
\begin{definition} $\obstable$ is \boldemph{consistent} if for all sequence pairs $s_1$ and $s_2$ where $\orow{s_1}\equiv\orow{s_2}$, all their transitions also remain equivalent with each other: $\orow{s_1\cdot \sigma}\equiv\orow{s_2\cdot\sigma}$ for all $\sigma \in \Sigma^I$.
\end{definition}
\begin{definition} Table $\obstable$ is \boldemph{unified} if $T(s\cdot e) \in \mathcal{R}$ for all $s\cdot e \in (S\cup (S\cdot \Sigma^I)) \cdot E$.
\end{definition}
\subsubsection{4.3 Making a Symbolic Hypothesis}\label{sec:symbolic-hypothesis}
If symbolic observation table $\obstable$ is \textit{unified}, \textit{closed}, and \textit{consistent}, a symbolic hypothesis can be made (Figure \ref{fig:example}c). 
This construction is identical to \lstar{}, except that the outputs are symbolic:
\begin{align*}
\hat{Q} &= \{\orow{s} | \forall s\in S\} \text{ \ is the set of states}\\
\hat{q}_0 &= \orow{\varepsilon} \text{ \ is the initial state}\\
\hat{\delta}(\orow{s},\sigma) &= \orow{s\cdot \sigma} \text{ \ for all $s\in S$ and $\sigma \in \Sigma^I$}\\
\hat{L}(\orow{s}) &= T(s\cdot\varepsilon) \text{ \ is the sequence to output function}\\
&\langle \hat{Q}, \hat{q}_0, \Sigma^I, \Sigma^O, \hat{\delta}, \hat{L}\rangle \text{ is a symbolic hypothesis.}
\end{align*}
\subsubsection{4.4 Making a Concrete Hypothesis}\label{sec:concrete-hypothesis} 
The learner finds a satisfying solution $\Lambda$ to the set of constraints $\mathcal{C}$, while subject to the global constraint requiring the value of each representative to be in $\Sigma^O$ (Figure \ref{fig:example}c). Thus, $\hat{L}$ becomes concrete.
\begin{align}
&\Lambda\longleftarrow\textsc{FindSolution}(\obstable, \Sigma^O)\\
&\hat{L}(\orow{s}) = \Lambda[T(s\cdot\varepsilon)]
\end{align} In particular, $\Lambda$ finds satisfying values for each member of $\mathcal{R}$, the set of equivalence class representatives. Since $\obstable$ is unified, we are guaranteed that $T$ maps from queried sequences to $\mathcal{R}$; hence $\Lambda[T(s\cdot \varepsilon)]$ is guaranteed to resolve to a concrete value as long as the teacher provides consistent preferences.
\begin{algorithm}[t]
\caption{\ouralgorithm{}}
\label{alg:pref-demo}
\textbf{Input}: Alphabets $\Sigma^I$ (input) and $\Sigma^O$ (output), teacher $\mathcal{T}$\\
\textbf{Output}: Moore Machine $\mathcal{H} = \langle \hat{Q}, \Sigma^I, \Sigma^O, \hat{q}_0, \hat{\delta}, \hat{L}\rangle$\\
\begin{algorithmic}[1]
\STATE Initialize $\mathcal{O} = \obstable$ with $S=\{\varepsilon\},E=\{\varepsilon\}$, $\mathcal{C}=\{\}$, $\Gamma=\emptyset$\\
\STATE $\mathcal{O} \longleftarrow$ \textsc{SymbolicFill}$\left(\mathcal{O}| \mathcal{T}\right)$
\REPEAT
\STATE $\mathcal{O} \longleftarrow$ \textsc{MakeClosedAndConsistent}$\left(\mathcal{O}| \mathcal{T}\right)$
\STATE $\mathcal{H}\longleftarrow$\textsc{MakeHypothesis}$\left(\mathcal{O}, \Sigma^I, \Sigma^O\right)$
\STATE \textit{result} $\longleftarrow$ \textsc{EquivalenceQuery}$(\mathcal{H}| \mathcal{T})$
\STATE $\mathcal{O}\longleftarrow$\textsc{ProcessCex}$(\mathcal{O},\text{\textit{result}})$ \textbf{if} \textit{result} $\neq$ \textit{correct}
\UNTIL{\textit{result} $=$ \textit{correct}}
\RETURN $\mathcal{H}$
\end{algorithmic}
\end{algorithm}
\begin{algorithm}[t]
\caption{A Query Efficient Symbolic Fill Procedure}
\label{alg:efficient-symbolic-fill-abbrv}
\textbf{procedure}~\textsc{SymbolicFill}$\left(\obstable| \mathcal{T}\right)$\\
\begin{algorithmic}[1]
    \STATE \textit{seqs} $= \{\}$; Let $\mathcal{O}=\obstable$; let \textit{oldsortedseqs} be a sorted list of sequences.
    \STATE \textit{seqs} $=$ \textsc{PopulateMissingFreshVars}($\mathcal{O}$)
    \STATE $\mathcal{O}$ $\longleftarrow$ \textsc{PrefQsByRandomizedQuicksortFollowedByLinearMerge}(\textit{seqs, oldsortedseqs}, $\mathcal{O}|\mathcal{T}$)
    \STATE $\mathcal{O}\longleftarrow$ \textsc{Unification}$(\mathcal{O})$
    \RETURN $\mathcal{O}$
\end{algorithmic}
\end{algorithm}

\subsubsection{4.5 \ouralgorithm{}}\label{sec:remap} We now describe \ouralgorithm{} (Algorithm~\ref{alg:pref-demo}) in terms of the previously discussed components. In order to make a symbolic hypothesis, \ouralgorithm{} must first obtain a \textit{unified, closed, and consistent} $\obstable$. To perform closed and consistency checks, the table must be unified. Therefore, \ouralgorithm{} must \textit{symbolically fill} $\obstable$ by asking preference queries and performing unification to obtain a unified table. If the unified table is not closed or not consistent, then the table is alternately expanded and symbolically filled until the table becomes unified, closed, and consistent.

\textbf{\textsc{SymbolicFill}} \ A symbolic fill (Algorithm~\ref{alg:efficient-symbolic-fill-abbrv} and Appendix Algorithm \ref{alg:efficient-symbolic-fill} and \ref{alg:symbolic-fill}) (\emph{i}) creates fresh variables for empty entries in the table, (\emph{ii}) asks preference queries, and (\emph{iii}) performs unification. If a sequence $s\cdot e\in(S\cup(S\cdot\Sigma^I))\cdot\,E$ does not have an associated variable in the context $\Gamma$, then a fresh variable $\Gamma[s\cdot e]$ is created. Preference queries are executed to obtain the total preordering of $(S \cup (S\cdot \Sigma^I))\cdot E$. In our implementation, we use preference queries in place of comparisons in randomized quicksort and linear merge. Once every sequence in $(S \cup (S\cdot \Sigma^I))\cdot E$ has a variable, and once the preference queries have been completed, unification is performed. \textsc{SymbolicFill} is called on lines 2, 4 (in \textsc{MakeClosedAndConsistent}), and 7 (in \textsc{ProcessCex}.). Unification is shown in Appendix Algorithm~\ref{alg:unification-a}.

\boldemph{Ensuring Consistency} \  If the unified table is not consistent, then there exists a pair $s_1, s_2\in S$ and $\sigma\in\Sigma^I$ for which $\orow{s_1}\equiv\orow{s_2}$ and $\orow{s_1\cdot\sigma}\not\equiv\orow{s_2\cdot\sigma}$, implying there is an $e\in E$ such that $T(s_1\cdot\sigma\cdot e)\not\equiv T(s_2\cdot\sigma\cdot e)$. To attempt to make the table consistent, add $\sigma \cdot e$ to $E$, and then perform a symbolic fill. Figure \ref{fig:example}d shows inconsistency, and Figure \ref{fig:example}e shows a table made consistent through expansion of the suffix set.

\boldemph{Ensuring Closedness} \  If the unified table is not closed, then $\orows{S\cdot \Sigma^I}\not\subseteq \orows{S}$. To attempt to make the table closed, find a row $\orow{s'}$ in $\orows{S\cdot\Sigma^I}$ but not in $\orows{S}$. Add $s'$ to $S$, update $S\cdot \Sigma^I$, then fill symbolically. Figure \ref{fig:example}a to \ref{fig:example}b shows a closure process.

The closed and consistency checks occur in a loop (consistency first, closed second) inside \textsc{MakeClosedAndConsistent} until the table becomes unified, closed, and consistent. Then hypothesis $h=\langle \hat{Q}, \hat{q}_0, \Sigma^I, \Sigma^O, \hat{\delta}, \hat{L}\rangle$ is generated by \textsc{MakeHypothesis} (Figure \ref{fig:example}c) and is sent to the teacher via \textsc{EquivalenceQuery} (Figure \ref{fig:example}c to \ref{fig:example}d). If $h$ is wrong, then a counterexample $c$ is returned, as well as feedback $\phi(c)$ which is interpreted as a new constraint on the value of $\hat{f}(c)$. The counterexample $c$ and all its prefixes are added to $S$, and then a symbolic fill is performed. Then the constraint on the value of $\hat{f}(c)$ is added to $\mathcal{C}$ as a constraint on the value of the representative at $T(c)$.

\section{Theoretical Guarantees of \ouralgorithm{}}\label{theory}
We now cover the algorithmic guarantees of \ouralgorithm{} when $\mathcal{T}$ uses exact equivalence queries, and show how sampling-based equivalence queries achieves PAC--identification.
We first detail how \ouralgorithm{} guarantees termination and yields a correct, minimal Moore machine that classifies sequence equivalently to $f$.
If \ouralgorithm{} terminates, then the final hypothesis must be correct, since termination occurs only if no counterexamples exist for the final hypothesis. Therefore, if the hypothesized Moore machine classifies all sequences correctly according to the teacher, it must be correct. Thus, proving termination implies correctness. See Appendix for sketches and proofs. Here, we assume the teacher provides feedback $\hat{f}(c)=f(c)$ with counterexample $c$.

\begin{theorem} If $\obstable$ is unified, closed, and consistent, and the range of \textsc{MakeHypothesis}$(\obstable)$ is $\mathcal{H}$, then every hypothesis $h \in \mathcal{H}$ is consistent with constraints $\mathcal{C}$. Any other hypothesis consistent with $\mathcal{C}$, but not contained in $\mathcal{H}$, must have more states.
\end{theorem}
\begin{proof} (Sketch) A given unified, closed, and consistent symbolic observation table $\obstable$ corresponds to $(\mathcal{S}, \mathcal{R}, \mathcal{C})$, where $\mathcal{S}$ is a symbolic hypothesis, $\mathcal{R}$ is the set of representatives used in the table, and $\mathcal{C}$ are the constraints expressed over $\mathcal{R}$. All hypotheses in $\mathcal{H}$ have states and transitions identical to $\mathcal{S}$. Each satisfying solution $\Lambda$ to $\mathcal{C}$ corresponds to a unique concrete hypothesis in $\mathcal{H}$. Therefore every concrete hypothesis in $\mathcal{H}$ is consistent with $\mathcal{C}$. Let $|h|$ represent the number of states in $h$. We know for all $h\in \mathcal{H}$, $|h|=|\mathcal{S}|$. Let $\overline{\mathcal{H}}$ be the set of concrete hypotheses \emph{not in} $\mathcal{H}$. Note $\overline{\mathcal{H}}$ can be partitioned into three sets---concrete hypotheses with (a) fewer states than $\mathcal{S}$, (b) more states than $\mathcal{S}$, and (c) same number of states as $\mathcal{S}$ but inconsistent with $\mathcal{C}$. We ignore (c) because we care only about hypotheses consistent with $\mathcal{C}$. Consider any concrete hypothesis $h$ in $\mathcal{H}$ and its corresponding satisfying solution $\Lambda$. Suppose we desire another hypothesis $h'$ to be consistent with $h$. If $|h'|<|h|$, then $h'$ cannot be consistent with $h$ because at least one sequence will be misclassified. Therefore, if $h'$ must be consistent with $h$, then we require $|h'|\geq |h|$. Thus, any other hypothesis consistent with $\mathcal{C}$, but not in $\mathcal{H}$, must have more states.\qed\end{proof}

Theorem 1 establishes that the output of \ouralgorithm{} will be the smallest Moore machine consistent with all the constraints in $\mathcal{C}$. This is necessary to prove termination.

\begin{lemma}
Whenever a counterexample $c$ is processed, either $0$ or $1$ additional representative values becomes known.
\end{lemma}
\begin{theorem}
Suppose $\obstable$ is unified, closed, and consistent. Let $\hat{h}=\textsc{MakeHypothesis}(\obstable)$ be the hypothesis induced by $\Lambda$, a satisfying solution to $\mathcal{C}$. If the teacher returns a counterexample $c$ as the result of an equivalence query on $\hat{h}$, then at least one of the following is true about $\hat{h}$: (a) $\hat{h}$ contains too few states, or (b) the satisfying solution $\Lambda$ inducing $\hat{h}$ is either incomplete or incorrect.
\end{theorem}
\begin{corollary} \ouralgorithm{} must terminate when the number of states and number of known representative values in a concrete hypothesis reach their respective upper bounds.
\end{corollary}
\begin{proof} (Sketch) This sketch applies to the above lemma, theorem, and corollary about termination. Consider the sequence $\dots,h_{k-1}, h_k,\dots$ of hypotheses that \ouralgorithm{} makes. For a given pair of consecutive hypotheses $(h_{k-1}, h_k)$, consider how the number of states $n$, and the number of \emph{known} representative values $n_\bullet$ changes. Let $n^*$ be the number of states of the minimal Moore machine correctly classifying all sequences. Let $V^*\leq |\Sigma^O|$ be the upper bound on $|\mathcal{R}|$. Note that $0\leq n_\bullet \leq |\mathcal{R}| \leq V^* \leq |\Sigma^O|$ always holds. Through detailed case analysis on returned counterexamples, we can show that the change in $n_\bullet$, denoted by $\Delta n_\bullet$, must always be either $0$ or $1$, and furthermore, if $\Delta n_\bullet = 0$, then we must have $\Delta n \geq 1$. By the case analysis and tracking $n$ and $n_\bullet$, observe that if a counterexample $c$ is received from the teacher due to hypothesis $h$, then \emph{at least one of} (a) $n < n^*$ or (b) $n_\bullet < V^*$ must be true. Since $\Delta n_\bullet$ and $\Delta n$ cannot simultaneously be $0$, whenever a new hypothesis is made, progress must be made towards the upper bound of $(n^*, V^*)$. If the upper bound is reached, then the algorithm must terminate, since it is impossible to progress from the point $(n^*, V^*)$.\qed
\end{proof}
\begin{theorem}[Query Complexity] If $n$ is the number of states of the minimal automaton isomorphic to the target automaton, and $m$ is the maximum length of any counterexample sequence that the teacher returns, then (a) \ouralgorithm{} executes at most $n+|\Sigma^O|-1$ equivalence queries, and (b) the preference query complexity is $\mathcal{O}(mn^2 \ln (mn^2))$, which is polynomial in the number of unique sequences queried.
\end{theorem}
\begin{proof} Based on Theorem 2, we know that the maximum number of equivalence queries is the taxi distance from the point $(1,0)$ to $(n, |\Sigma^O|)$, which is $n+|\Sigma^O|-1$. From counterexample processing, we know there will be at most $m(n+|\Sigma^O|-1)$ sequences added to the prefix set $S$, since a counterexample $c$ of length $m$ results in at most $m$ sequences added to the prefix set $S$. The maximum number of times the table can be found inconsistent is at most $n-1$ times, since there can be at most $n$ states, and the learner starts with $1$ state. Whenever a sequence is added to the suffix set $E$, the maximum length of sequences in $E$ increases by at most $1$, implying the maximum sequence length in $E$ is $n-1$. Similarly, closure operations can be performed at most $n-1$ times, so the total number of sequences in $E$ is at most $n$; the maximum number of sequences in $S$ is $n+m(n+|\Sigma^O|-1)$. The maximum number of unique sequences queried in the table is the maximum cardinality of $(S\cup S\cdot \Sigma^I)\cdot E$, which is $$(n+m(n+|\Sigma^O|-1))(1+|\Sigma^I|)n = \mathcal{O}(mn^2).$$
%$$|(S\cup S\cdot \Sigma^I)\cdot E|=(n+m(n+|\Sigma^O|-1))(1+|\Sigma^I|)n = O(mn^2)$$
Therefore, the preference query complexity of \ouralgorithm{} is $\mathcal{O}(mn^2\ln(mn^2))$ due to randomized quicksort.\qed
\end{proof}

Lemma 1, Theorem 2, and Corollary 1 imply \ouralgorithm{} makes progress towards termination with every hypothesis, and termination occurs when specific conditions are satisfied; therefore its output must be correct. Theorem 3 indicates that REMAP learns the correct minimal automaton isomorphic to the target automaton in polynomial time.
Next, we show how \ouralgorithm{} achieves PAC--identification when sampling-based equivalence queries are used.

\begin{definition}[Probably Approximately Correct Identification]\label{def:pacid} \ Given Moore machine $M=\langle Q, q_0, \Sigma^I, \Sigma^O, \delta, L\rangle$, let the classification function $f:(\Sigma^I)^*\rightarrow\Sigma^O$ be represented by $f(s)=L(\delta(q_0,s))$ for all $s\in(\Sigma^I)^*$. Let $\mathcal{D}$ be any probability distribution over $(\Sigma^I)^*$. An algorithm $\mathcal{A}$ probably approximately correctly identifies $f$ if and only if for any choice of $0< \epsilon \leq 1$ and $0 < d < 1$, $\mathcal{A}$ always terminates and outputs an $\epsilon$-approximate sequence classifier $\hat{f}:(\Sigma^I)^*\rightarrow\Sigma^O$, such that with probability at least $1-d$, the probability of misclassification is $P(\hat{f}(s)\neq f(s)) \le \epsilon$ when $s$ is drawn according to the distribution $\mathcal{D}$.
\end{definition}
\begin{theorem}\label{thm:pacid1} \ouralgorithm{} achieves probably approximately correct identification of any Moore machine when the teacher $\mathcal{T}$ uses sampling-based equivalence queries with at least $m_k \geq \left\lceil\frac{1}{\epsilon}\left(\ln\frac{1}{d}+k\ln 2\right)\right\rceil$ samples drawn i.i.d. from $\mathcal{D}$ for the $k$th equivalence query.
\end{theorem}
\begin{proof} (Sketch) The probability $1-\epsilon_k$ of a sequence sampled from an arbitrary distribution $\mathcal{D}$ over $(\Sigma^I)^*$ depends upon the distribution and the intersections of sets of sequences of the teacher and the learner's $k$th hypothesis with the same classification values. The probability that the $k$th hypothesis misclassifies a sequence is $\epsilon_k$. If the teacher samples $m_k$ samples for the $k$th equivalence query, then an upper bound can be established for the case when $\epsilon_k\leq \epsilon$ for a given $\epsilon$. Since we know \ouralgorithm{} executes at most $n+|\Sigma^O|-1$ equivalence queries, one can upper bound the probability that \ouralgorithm{} terminates with an error by summing all probabilities of events that the teacher does not detect an error in at most $n+|\Sigma^O|-1$ equivalence queries. An exponential decaying upper bound can be found, and a lower bound for $m_k$ can be found in terms of $\epsilon,d,$ and $k$.\qed
\end{proof}
\begin{theorem}\label{thm:pacid2} To achieve PAC-identification under \ouralgorithm{}, given parameters $\epsilon$ and $d$, and if $f$ can be represented by a minimal Moore machine with $n$ states and $|\Sigma^O|$ classes, then teacher $\mathcal{T}$ needs to sample at least $$\mathcal{O}\left(n+|\Sigma^O| + \frac{1}{\epsilon}\left((n+|\Sigma^O|)\ln\frac{1}{d} + (n+|\Sigma^O|)^2\right)\right)$$ sequences i.i.d. from $\mathcal{D}$ over the entire run of \ouralgorithm{}.
\end{theorem}
\begin{proof}
Since for the $k$th equivalence query, the teacher must sample at least $m_k\geq \left\lceil \frac{1}{\epsilon}(\ln\frac{1}{d} + k\ln2)\right\rceil$ sequences in order to achieve PAC-identification, if the total number of samples is to be minimized while still achieving PAC-identification, then the teacher can just sample a quantity of sequences i.i.d. from $\mathcal{D}$ equal to the following total
\[\sum_{k=1}^{n+|\Sigma^O|-1}\left[\frac{1}{\epsilon}(\ln\frac{1}{d} + k\ln2) + 1\right]\]\[
=n+|\Sigma^O|-1+\frac{1}{\epsilon}\left[(\ln\frac{1}{d})(n+|\Sigma^O|-1) + \ln2\sum_{k=1}^{n+|\Sigma^O|-1}k\right]\]\[
=\mathcal{O}\left((n+|\Sigma^O|)+\frac{1}{\epsilon}((n+|\Sigma^O|)\ln\frac{1}{d} + (n+|\Sigma^O|)^2)\right)
\]\qed
\end{proof}

Theorem 4 and Theorem 5 imply \ouralgorithm{} achieves PAC--identification for a choice of $\epsilon$ and $d$ as long as the teacher samples sufficient sequences per equivalence query. In particular, Theorem 5 indicates the total quantity of sequences sampled by the teacher to achieve PAC-identification depends on both $n$ the number of states, and $|\Sigma^O|$ the number of output classes. This contrasts with the result for PAC-identification of DFAs (Theorem 7 \cite{Angluin87}), which depends only on $n$. Finally, since \ouralgorithm{} outputs a Moore machine, one can leverage Moore and Mealy machine equivalence in order to convert the final hypothesis into a reward machine, as defined and covered in the next section.

\section{Learning Reward Machines from Preferences}\label{rl}

We consider applying \ouralgorithm{} to learn reward machines from preferences. Reward machines are Mealy machines with propositional and reward semantics. Equivalence between Mealy and Moore machines allows the output of \ouralgorithm{} to be converted to a reward machine.
We first review reinforcement learning and reward machine semantics.

\boldemph{Markov Decision Processes} \ Decision making problems are often modeled by a Markov Decision Process (MDP), which is a tuple $\langle \mathcal{S},\mathcal{A},P,R,\gamma\rangle$ where $\mathcal{S}$ is the set of states, $\mathcal{A}$ is the set of actions, $P:\mathcal{S}\times \mathcal{A}\times \mathcal{S}\rightarrow [0,1]$ represents the transition probability from state $s$ to $s'$ via action $a$. The reward function $R:\mathcal{S}\times \mathcal{A}\times\mathcal{S}\rightarrow \mathbb{R}$ provides the associated scalar reward, and $0\leq\gamma \leq 1$ is a discount factor. In MDPs, the Markovian assumption is  that transitions and rewards depend only upon the current state-action pair and the next state. However, not all tasks are expressible using Markovian reward~\cite{abel2021}.

\boldemph{Non-Markovian Reward} \ Non-Markovian Reward Decision Processes are identical to MDPs, except that $R: (\mathcal{S}\times \mathcal{A})^*\rightarrow \mathbb{R}$ is non-Markovian a reward function that depends on state-action history. This allows reward machines to model a class of non-Markovian reward functions.

\boldemph{Reward Machines} \ A reward machine (RM)
is a Mealy machine where $\Sigma^I=2^\mathcal{P}$, $\Sigma^O$ is a set of reward emitting objects, and $\mathcal{P}$ is a set of propositions describing states and actions. A \emph{labeling} function $\mathbb{L}:\mathcal{S}\times \mathcal{A}\times \mathcal{S}\rightarrow 2^\mathcal{P}$ with $\mathbb{L}(s_{k-1},a_k,s_k)=l_k$ labels a state-action sequence $s_0a_1s_1a_2s_2\dots a_ns_n$ with label sequence $l_1l_2\dots l_n$. Thus, reward machines operate over label sequences. 

A single disjunctive normal formula (DNF) labeled transition can summarize multiple transitions with identical $\Sigma^O$ labels (connecting a pair of states), since the elements of $2^\mathcal{P}$ are sets of propositions. Reward machines map label sequences to reward outputs and can be represented as $f: (2^\mathcal{P})^*\rightarrow \Sigma^O$, so \ouralgorithm{} learns a reward machine by converting the output Moore machine to a reward machine. 

\boldemph{Sequential Tasks in OfficeWorld and CraftWorld}\label{rmdetails} \ \citet{icarte2018} and \citet{andreas17a} introduced the OfficeWorld and CraftWorld gridworld domains, respectively, and 
feature sequential tasks encoded as reward machines. OfficeWorld features 4 sequential tasks across several rooms with various objects available for an agent to interact with. Example tasks include (1) picking up coffee and mail and delivering them to a certain room, or (2) continuously patrolling between a set of rooms. CraftWorld is a 2D version of MineCraft, where the 10 sequential tasks involve the agent collecting materials and constructing tools or objects in a certain order while avoiding hazards.

\section{Empirical Results}
We evaluate the exact and PAC--identification (PAC-ID) versions of \ouralgorithm{} and consider: first, \boldemph{how often is PAC-ID \ouralgorithm{} correct?} (Exact \ouralgorithm{} is guaranteed to be correct). To answer this, we run experiments by applying PAC-ID \ouralgorithm{} to learn reward machines (RMs), by converting the Moore machine into a RM. We measure empirical correctness with \emph{empirical probability of isomorphism} and \emph{average regret}. Second, \boldemph{how do exact and PAC-ID \ouralgorithm{} scale}? We measure preference query complexity as a function of the number of unique sequences queried, and present an example phase diagram of algorithm execution. We use Z3~\cite{demoura-z3-2008} for the constraint solver.

\boldemph{Setup} \ We investigate these questions on 14 sequential tasks in the OfficeWorld and CraftWorld domains. 
The Appendix~\ref{appendix:experimentalsetup} contains domain specific details. We implement exact equivalence queries (Definition \ref{exact-eq-def}) using a variant of the Hopcroft-Karp algorithm~\cite{Almeida2009TestingTE, Hopcroft1971ALA}. We implement i.i.d. sequence sampling in sampling-based equivalence queries with the following process per sample: sample a length $L$ from a geometric distribution; then, construct an $L$-length sequence by drawing $L$ elements i.i.d. from a uniform distribution over $\Sigma^I$. In both the exact and sampling-based equivalence queries, strong feedback is used.

\subsection{PAC--Identification Correctness Experiments}
\boldemph{Reproducibility} \ PAC-ID \ouralgorithm{} was run 100 times per ground truth reward machine. We measure correctness based on (1) empirical probability that the learned RM is isomorphic to the ground truth RM, based on classification accuracy, and (2) average policy regret between the learned RM policy and the ground truth RM policy.

\begin{figure}[t]
    \centering
    \includegraphics[height=0.29\columnwidth]{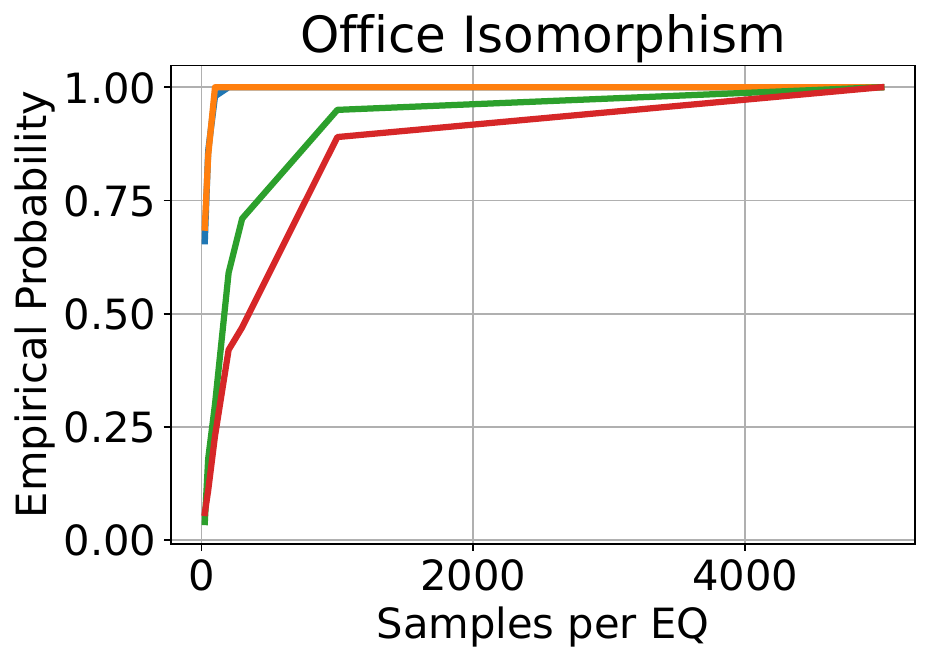}
    \includegraphics[height=0.29\columnwidth]{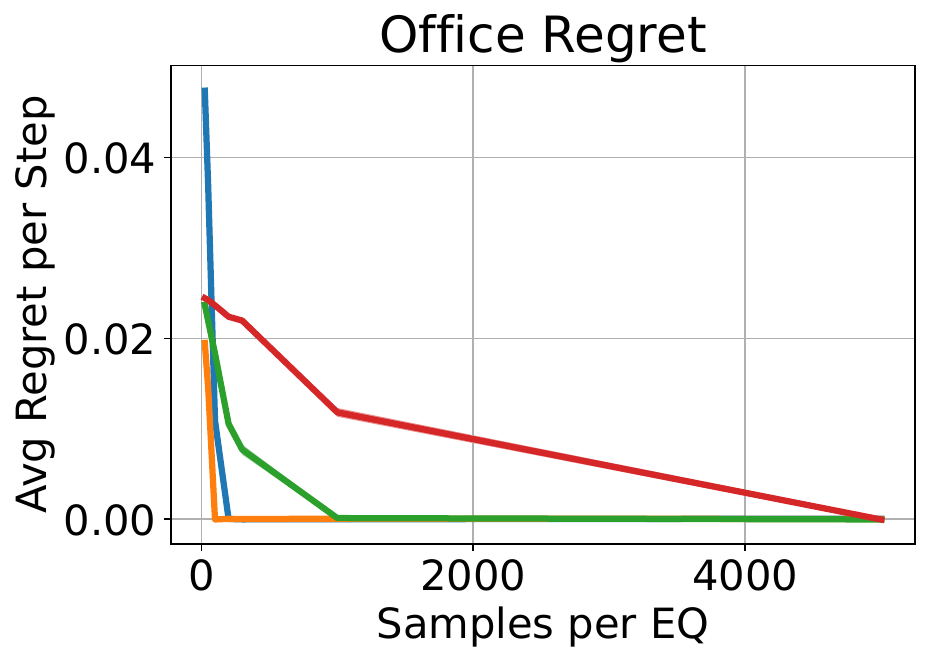}\\
    \includegraphics[height=0.29\columnwidth]{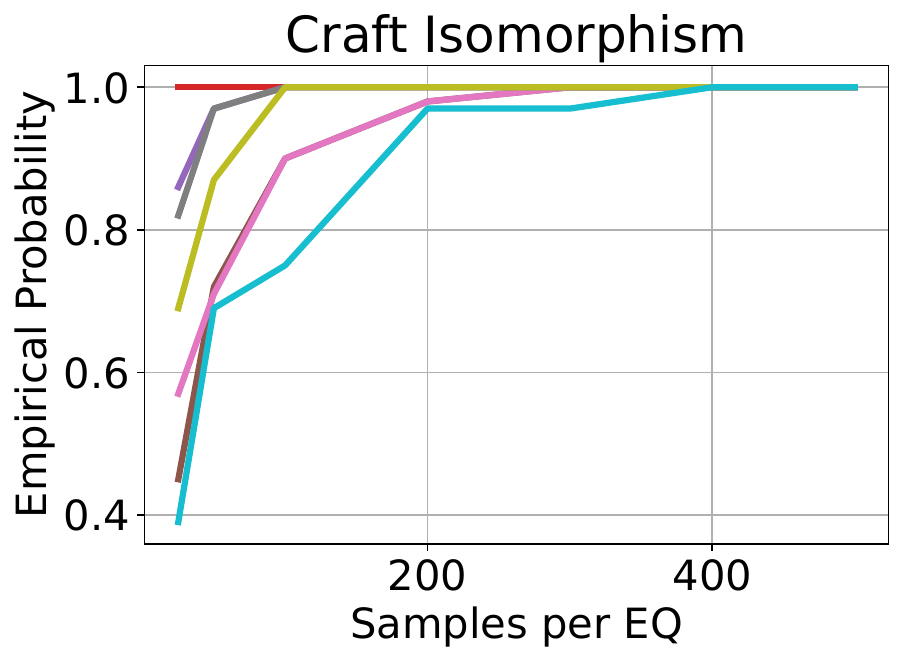}
    \includegraphics[height=0.29\columnwidth]{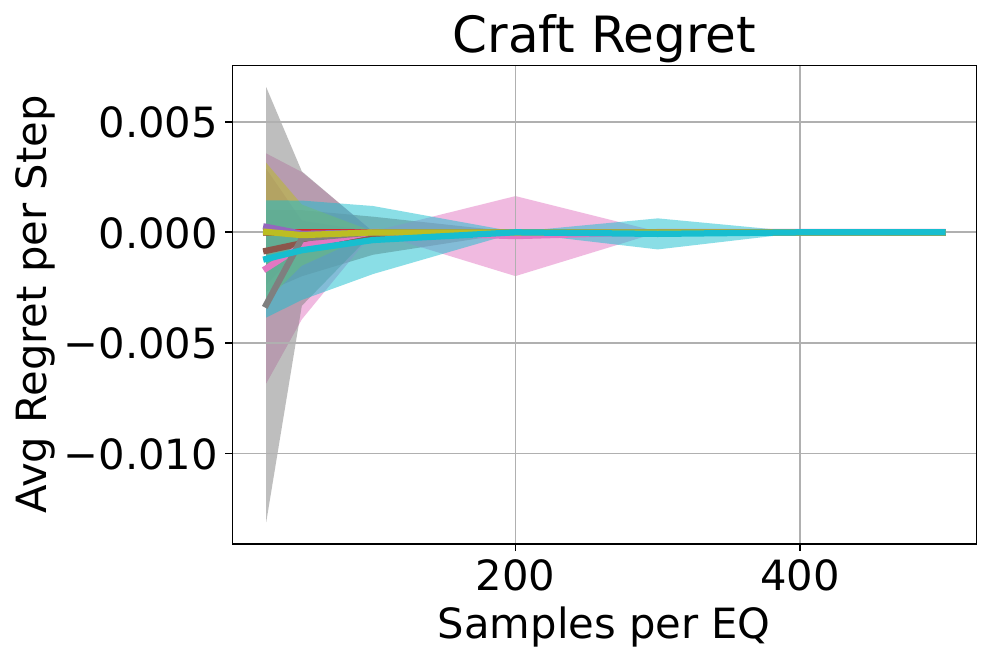}\\
    \includegraphics[width=0.49\columnwidth]{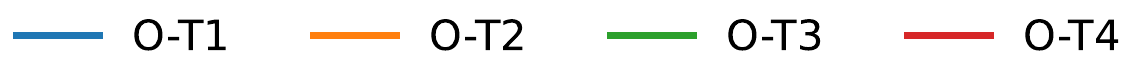}
    \includegraphics[width=0.49\columnwidth]{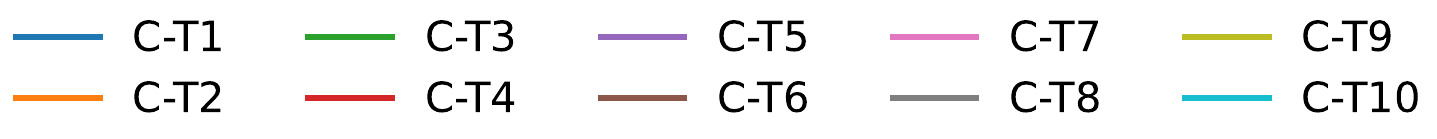}
    \caption{PAC--identification \ouralgorithm{}: (\emph{left}) empirical isomorphism probability, (\emph{right}) 
    average regret as functions of the number of samples per equivalence query for 4 OfficeWorld tasks (O-T1-4) and 10 CraftWorld tasks (C-T1-10).}
    \label{fig:correctness}
\end{figure}
\begin{figure}[!t]
    \centering
    \includegraphics[height=0.34\columnwidth]{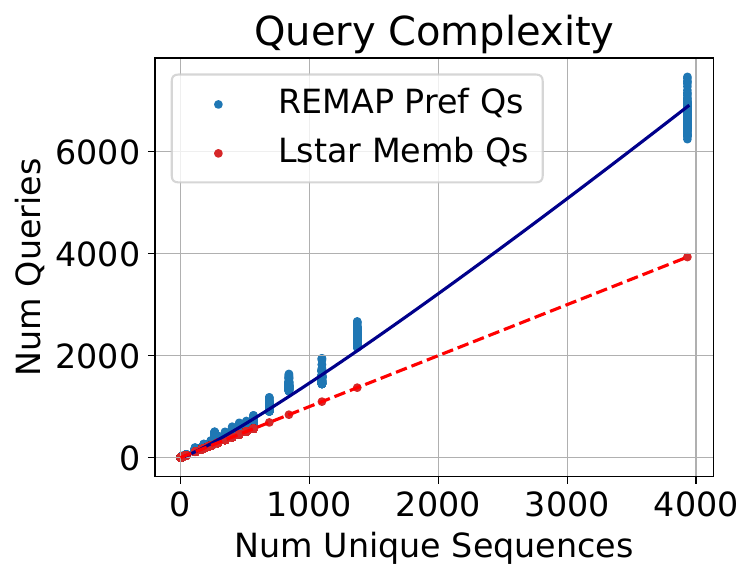}\hfill
    \includegraphics[height=0.34\columnwidth]{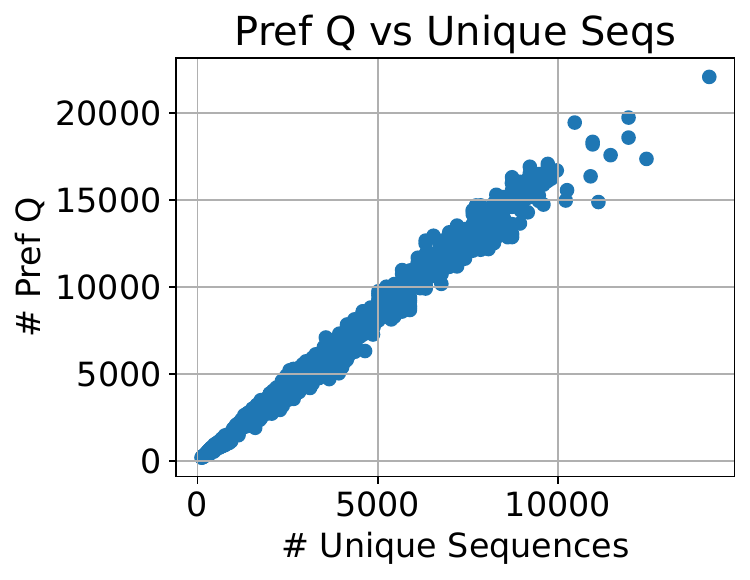}\\
    \includegraphics[height=0.34\columnwidth]{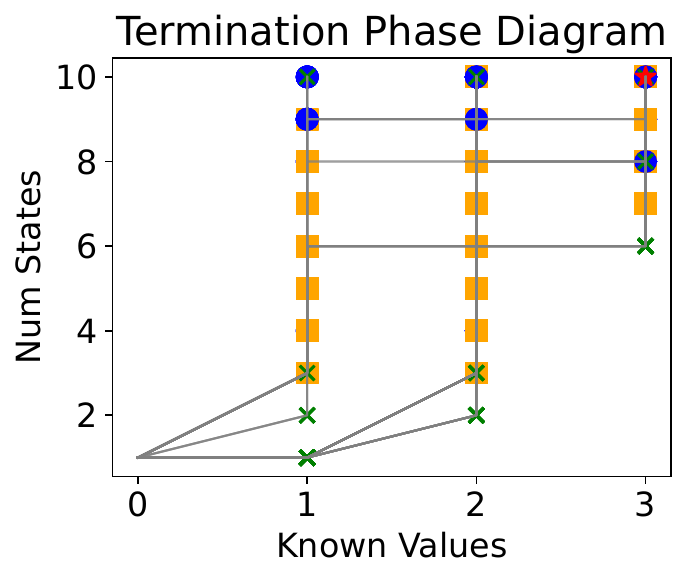}
    \includegraphics[height=0.19\columnwidth]{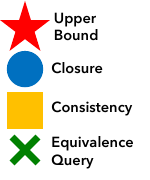}
    \includegraphics[height=0.34\columnwidth]{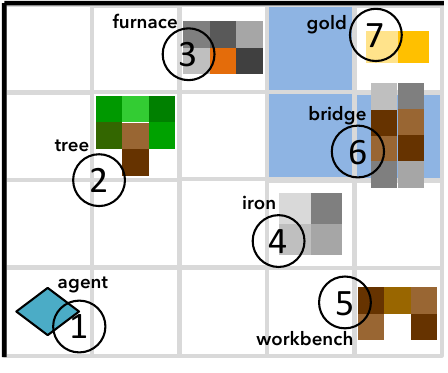}
    \caption{Query Complexity. \textit{Top left}: Exact \ouralgorithm{} preference query complexity. Mean of 100 trials per ground truth reward machine (blue dots) $\pm1$ standard deviation (orange, grey bars). 
    \textit{Top right}: PAC-ID \ouralgorithm{}, preference query complexity is $\mathcal{O}(n\ln n)$ in the number of unique sequences in the table. \textit{Bottom left}: Example termination phase diagram. \textit{Bottom right}: CraftWorld environment depiction.}
    \label{fig:query-complexity}
\end{figure}

\boldemph{Empirical Probability of Isomorphism} is the \emph{fraction of learned RMs with 100\% classification accuracy}. As the number of sample sequences tested by the teacher per equivalence query increases, the probability that the learner outputs a RM isomorphic to the ground truth RM upon termination goes to~$1$ (Figure \ref{fig:correctness}, left column). Classification accuracy is defined as the \emph{fraction of a test set of sequences that are identically classified} by the learned and ground truth RMs. The Appendix describes the distribution over $(\Sigma^I)^*$ that the test set is drawn from.

\boldemph{Average Policy Regret} \ We employ \textit{Q-learning with counterfactual experiences for reward machines} (CRM)~\cite{toro-etal-jair22} to obtain optimal policies for ground truth and learned RMs. We measured the empirical expected return of optimal policies learned from each type of RM. Average regret for a given task was measured as the \emph{difference between the empirical return under the ground truth RM for that task} (averaged over 100 CRM trials) \emph{and the empirical return under the learned RM} (with 10 CRM trials per learned RM, then averaging all $100\times 10 = 1000$ trials). Regret goes to 0 as the number of samples tested by the teacher per equivalence query increases (Figure~\ref{fig:correctness}, right column). The Appendix \ref{regretcomputation} describes regret computation details.

\boldemph{Correctness Conclusion} \ Exact \ouralgorithm{} learns the correct automaton 100\% of the time. Additionally, PAC-ID \ouralgorithm{} is more likely to be correct as the number of samples per equivalence query increases: isomorphism probability goes to 1 and regret goes to 0 for all tasks in both domains.

\subsection{Scaling Experiments} Figure \ref{fig:query-complexity} shows query complexity results. We measure preference query complexity of exact and PAC-ID \ouralgorithm{}, as a function of the number of unique sequences stored in the table upon termination. Exact \ouralgorithm{} (upper left) displays a trendline of $C=0.2114 N\ln N$ with $R^2=0.99268$, where $C$ is the number of queries, and $N$ is the number of unique sequences in the observation table. PAC-ID \ouralgorithm{} (upper right) tends to make significantly more preference queries about unique sequences compared to exact \ouralgorithm{} due to the sampling process. However, the number of preference queries is still $\mathcal{O}(N\ln N)$ 
due to randomized quicksort and linear merge comparison complexity. In comparison, an \lstar{}-based approach would use exactly $N$ membership queries (exactly linear in $N$ with a coefficient of $1$).

The maximum number of equivalence queries \ouralgorithm{} makes (Theorem 3) is the taxi distance from $(0,1)$ to $(|\Sigma^O|,n)$ in the termination phase diagram of Figure \ref{fig:query-complexity}. Progress (Lemma 1 and Theorem 2) towards termination (Corollary 1) occurs whenever a new hypothesis is made. \ouralgorithm{} can terminate early when all variables have correct values and the required number of states is reached.

\section{Related Work}
Active approaches for learning automata are variations or improvements of Angluin’s seminal \lstar{} algorithm~\cite{Angluin87}, featuring \emph{membership} and \emph{equivalence queries}. We consider an alternative formulation: \emph{actively learning automata from preference and equivalence queries} featuring feedback. We first discuss adaptations of \lstar{} for learning variants of finite automata, including reward machine variants.

\boldemph{Learning Finite Automata} \ \citet{Angluin87} introduced \lstar{} to learn deterministic finite automata (DFAs). \ouralgorithm{} has similar theoretical guarantees as \lstar{}, but utilizes a symbolic observation table, rather than an evidence-based one. Other algorithms adopt the evidence-based table of \lstar{} to learn: symbolic automata~\cite{symbolicAutomata,Argyros2018TheLO}, where Boolean predicates summarize state transitions; weighted automata~\cite{Bergadano1994LearningBO,Balle2015LearningWA} which feature valuation semantics for sequences on non-deterministic automata; probabilistic DFAs~\cite{Weiss2019}, a weighted automata that models distributions of sequences. None of these approaches uses preference queries.

However,~\citet{shah2023learning} considers active, cost-based selection between membership and preference queries to learn DFAs, relying on a satisfiability encoding of the problem. They assume a \emph{fixed hypothesis space} and have \emph{probabilistic} guarantees for termination and correctness. \ouralgorithm{}, through unification, \emph{navigates a sequence of hypothesis spaces}, each guaranteed to contain a concrete hypothesis satisfying current constraints, and has theoretical guarantees of correctness, minimalism, and termination under \emph{exact} and \emph{PAC--identification} settings.

Furthermore, learning finite automata from preference information relates to the novel problem of \emph{learning reward machines}~\cite{icarte2018} \emph{from preferences}. Learning Markovian reward functions from preferences has be studied extensively using neural~\cite{biyki2022aprel, Sadigh2017ActivePL} and interpretable decision tree~\cite{Bewley2021InterpretablePR,Kalra2023CanDD} representations, but approaches for learning reward machines primarily adapt evidence-based finite automata learning approaches.

\boldemph{Reward Machine Variants} \ Several reward machine (RM) variants have been proposed. Classical RMs~\cite{icarte2018,toro-icarte-2019-nips} have deterministic transitions and rewards; probabilistic RMs~\cite{dohmen-2022-icaps} model probabilistic transitions and deterministic rewards; and stochastic RMs~\cite{corazza-2022-aaai} pair deterministic transitions with stochastic rewards. Symbolic RMs~\cite{zhou-2022-icml} are deterministic like classical RMs, but feature symbolic reward values in place of concrete values. \citet{zhou-2022-icml} apply Bayesian inverse reinforcement learning (BIRL) to infer optimal reward values and actualize symbolic RMs into classical RMs, and require a symbolic RM sketch. \ouralgorithm{} requires no sketch, since it navigates over a hypothesis space of symbolic RMs and outputs a concrete classical RM upon termination.

\boldemph{Learning Reward Machines} \ Many RM learning algorithms assume access to explicit reward samples via environment interaction. Given a maximum RM size, \citet{toro-icarte-2019-nips} apply discrete optimization to arrive at a perfect classical RM. \citet{xu_jirp} learn a minimal classical RM by combining regular positive negative inference~\cite{Dupont94} with Q-learning for RMs~\cite{icarte2018}, and apply constraint solving to ensure each hypothesis RM is consistent with observed reward samples. \citet{corazza-2022-aaai} extended the method to learn stochastic RMs. \citet{topper2024bayesian} extends BIRL to learn classical RMs using simulated annealing, but needs the number of states to be supplied, and requires empirical tuning of hyperparameters. \lstar{} based approaches have also been used to learn classical~\cite{tappler2019based,GaonB20_nonmarkovian,xu_lstar} and probabilistic~\cite{dohmen-2022-icaps} RMs, relying on concrete observation tables.~\citet{GaonB20_nonmarkovian} and \citet{xu_lstar} use a binary observation table, while~\citet{tappler2019based} and \citet{dohmen-2022-icaps} record empirical reward distribution table entries.

In contrast, \ouralgorithm{} uses a symbolic observation table, and uses preferences information in place of explicit reward values. \ouralgorithm{} navigates symbolic hypothesis space, with constraint solving enabling a concrete classical RM.

\section{Conclusion}
We introduce the problem of learning Moore machines from preferences and propose \ouralgorithm{}, an \lstar{} based algorithm, wherein a strong learner with access to a constraint solver is paired with a weak teacher capable of answering preference queries and providing counterexample feedback in the form of a constraint. Unification applied to a symbolic observation table permits symbolic hypothesis space navigation; the constraint solver enables concrete hypotheses. \ouralgorithm{} has theoretical guarantees for correctness, termination, and minimalism under both exact and PAC--identification settings, and it has been empirically verified under both settings when applied to learning reward machines. Future work will expound on more realistic preference models,  variable strength feedback, and inconsistency.

\begin{credits}
\subsubsection{\discintname} The authors have no competing interests to declare that are relevant to the content of this article.
\end{credits}

\bibliographystyle{splncs04nat}
\bibliography{references}

\newpage
\appendix
\setcounter{theorem}{0}
\setcounter{lemma}{0}
\setcounter{corollary}{0}
\section{Technical Appendix}
In this technical appendix, we present empirical data, examples, experimental setup information, termination plots, detailed algorithms, and proofs. We also present the reproducibility checklist at the end of the appendix.

\subsection{Experimental Details}\label{appendix:experimentalsetup}
\subsubsection{Hardware}
Experiments were run on a server with 512GB of memory and 2 Intel(R) Xeon(R) Gold 6248R CPUs. Each experimental run was executed on a single core. The operating system was Ubuntu 22.04.4. The relevant software libraries are listed in the code appendix remap/README file.

\subsubsection{OfficeWorld and CraftWorld Domains}
The computational experiments in this paper involved the OfficeWorld and CraftWorld domains. Both are gridworld environments that have been used in the reward machine literature. They feature sequential tasks which can be represented as a reward machine. Figure \ref{fig:officeworldrm} illustrates an example OfficeWorld domain, along with a corresponding reward machine representing a task where the agent must bring paper to the desk while avoiding obstacles.

\subsubsection{Handling Reward Machine Incompleteness.} The classical reward machines specified by \citet{icarte2018} in the OfficeWorld and CraftWorld domains are incomplete automata, in that not all states have transitions defined. Specifically, classical reward machines have a terminal state from which no transitions can occur. However, an incomplete reward machine can be converted into a complete one by adding transitions from all terminal states to a single, special absorbing ``HALT'' state. 

Therefore, to handle reward machine incompleteness, we first convert the incomplete reward machines into complete Mealy machines with terminal states, and then convert the Mealy machine to a complete Moore machine with a single absorbing HALT state which all terminal states transition to. This latter Moore machine is used by the teacher as the ground truth non-Markovian reward function. Next, \ouralgorithm{} is run, resulting in a learned Moore machine. The learned Moore machine is converted to a Mealy machine, and the absorbing state is identified and removed. 

We do not merge terminal states in order to remain consistent with the original implementation reward machines, but we do collapse pairs of transitions into single summary transitions if a pair of transitions share the same start, end, and output values. This is accomplished by constructing a truth table, constructing a disjunctive normal formula (DNF) for the truth table, and then simplifying the DNF and using the result as the summarizing transition label. Once this process is complete, the result is a learned reward machine, which can now be evaluated.

\subsubsection{Handling Reward Machine Nondeterminism.} Classical reward machines have a deterministic definition. However, some of the classical reward machines specified by \citet{icarte2018} are \emph{nondeterministic}, in that from a given state, multiple transitions can be satisfied using under a given set of true propositions. This was the case for the reward machines for tasks 5 through 10 in the CraftWorld domain.

The common nondeterminism in those reward machines was the following type: assume the proposition set is $\mathcal{P}=\{a,b\}$, and the transitions have been summarized according to some set of Boolean formulae for which a subset is $\{\phi_a, \phi_b\}$, where $\phi_a$ is satisfied whenever $a$ holds, and $\phi_b$ is satisfied whenever $b$ holds. Suppose we have a state $q_1$, and two of its summarized transitions are the following: $q_2=\delta(q_1, \phi_a)$ and $q_3=\delta(q_1, \phi_b)$, meaning that a transition from $q_1$ to $q_2$ occurs if $\phi_a$ is satisfied, and a transition from $q_1$ to $q_3$ occurs if proposition $\phi_b$ holds. Clearly, $\phi_a$ and $\phi_b$ can simultaneously be satisfied if $a\land b$ holds, implying nondeterminism. Additionally, all the nondeterministic reward machine specifications also contained the following style of transitions: $q_4=\delta(q_1,\phi_a\phi_b)=\delta(q_1,\phi_b\phi_a)$, where $\phi_a\phi_b$ and $\phi_b\phi_a$ are sequences of length 2. This type of nondeterminism can be corrected by adding an additional state $q_5$ and modifying the transitions from $q_1$ to $q_2$ and $q_3$, while still maintaining the intended behabior of reaching state $q_4$. Specifically, the conversion is, given
\begin{figure}[t]
    \centering
    \includegraphics[width=0.495\columnwidth]{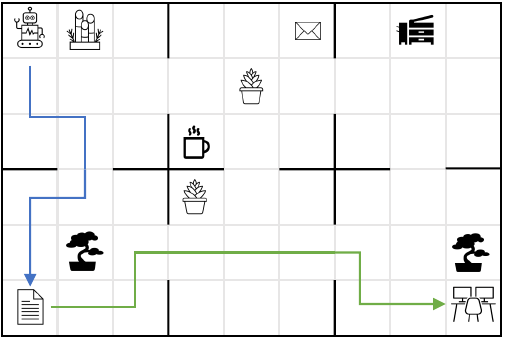}
    \includegraphics[width=0.495\columnwidth]{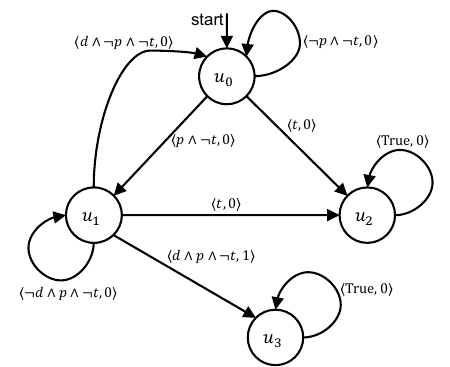}
    \caption{Example OfficeWorld domain \emph{(top)} and a reward machine \emph{(bottom)} encoding the sequential task ``\emph{bring the paper to the desk without running into any obstacles.}'' A transition from state $u_i$ to $u_j$, labeled by propositional formula $\phi$ and scalar reward $r$ as $\langle \phi, r\rangle$, occurs only if $\phi$ is satisfied; reward $r$ is emitted upon transition. Propositions $p, t$, and $d$ are the agent: possessing paper, running into an obstacle, and being located at the desk.}
    \label{fig:officeworldrm}
\end{figure}
\begin{align*}
q_1&=\delta(q_1,\varepsilon)=\delta(q_1, \neg\phi_a\land\neg\phi_b)\\
q_2&=\delta(q_1, \phi_a)=\delta(q_2,\neg\phi_b)\\
q_3&=\delta(q_1, \phi_b)=\delta(q_3,\neg\phi_a)\\
q_4&=\delta(q_1,\phi_a\phi_b)=\delta(q_1,\phi_b\phi_a),
\end{align*} we can make the following adjustments and additions:
\begin{align*}
q_2=\delta(q_1, \phi_a)&\Longrightarrow q_2=\delta(q_1, \phi_a\land \neg \phi_b)\\
q_3=\delta(q_1, \phi_b)&\Longrightarrow q_3=\delta(q_1, \phi_b\land \neg \phi_a)\\
\text{Add new state } &\Longrightarrow q_5=\delta(q_1, \phi_a\land \phi_b)\\ &\phantom{{}\Longrightarrow q_5}= \delta(q_5, \neg(\phi_a \lor \phi_b))\\
\text{Add new transition } &\Longrightarrow q_4=\delta(q_5, \phi_a \lor \phi_b).
\end{align*} These changes make the reward machine deterministic while still maintaining desired behavior by explicitly providing three different transitions away from state $q_1$ for processing the input proposition sets $\{a\},\{b\},$ and $\{a,b\}$ separately.

\begin{figure*}[!t]
    \centering
    \includegraphics[width=\textwidth]{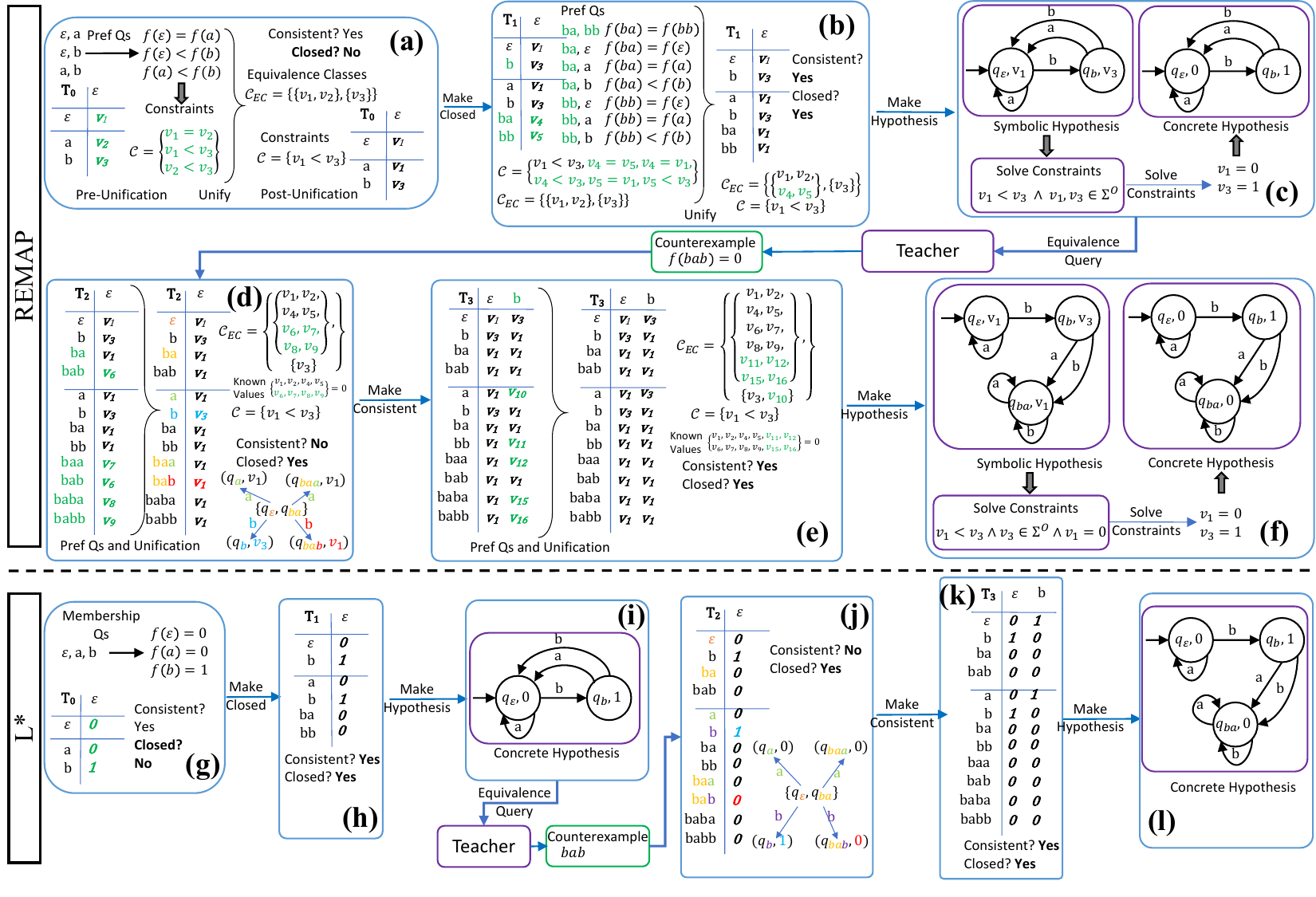}
    \caption{Learning a Moore machine with \ouralgorithm{} (top) vs \lstar{} (bottom). \lstar{} employs concrete values in its observation table, whereas \ouralgorithm{} uses a symbolic observation table $\obstable$. The function $f(s)$ returns $1$ if $s$ is in $a^*b$ and returns $0$ otherwise. (a) initializing $\obstable$ and performing a \textsc{SymbolicFill}, unclosed table; (b) expand $S$ with sequence $b$ to close the table, followed by a \textsc{SymbolicFill} yielding a unified, closed, and consistent $\obstable$; (c) \textsc{MakeHypothesis} yields concrete hypothesis $h_1$ from the symbolic hypothesis and solving constraints in $\mathcal{C}$; (d) submit $h_1$ via equivalence query, receive and process the counterexample $bab$ with feedback $f(bab)=0$ from teacher by adding $bab$, $ba$, and $b$ to $S$, perform a \textsc{SymbolicFill}, set the value equivalence class of $T(bab)$ to $f(bab)$, yields an inconsistent table; (e) expand suffixes $E$ with $b$ and perform a \textsc{SymbolicFill}, yielding a unified, closed, and consistent table; (f) \textsc{MakeHypothesis} yields concrete hypothesis $h_2$ for an equivalence query, wherein the teacher establishes $h_2$ is correct, and the algorithm terminates. Learning the target Moore machine with \lstar{} is shown is parts (g)-(l). (g) shows using membership queries to populate the table with concrete values, resulting in an unclosed table, which is then made closed in (h) by adding $b$ to $S$. Since the result is closed and consistent, a concrete hypothesis can be made in part (i). Sending this hypothesis to the teacher via an equivalence query results in the teacher sending a counterexample $bab$ back, which must be processed in (j). Here, the $bab$ and all its prefixes are added to $S$, resulting in an inconsistent table. Consistency is resolved by adding $b$ to $E$, resulting in a closed and consistent table. This allows the final hypothesis to be made in part $(l)$.}
    \label{fig:appendix_example}
\end{figure*}

\subsubsection{Test Set Generation for Classification Accuracy} Here, we describe the process inducing the distribution $\mathcal{D}$ over $(\Sigma^I)^*$ from which the test set is generated. To generate the test set, we follow the procedure described for sampling-based equivalence queries: we first sample a random variable $L$ from a geometric distribution to represent desired sequence length, and then populate each of the $L$ sequence elements i.i.d. from the uniform distribution over $\Sigma^I$. The geometric distribution we use is $\text{Pr}(L=l)=(1-p)^{l-1}p$, where $p=0.2$. Thus, the average sampled sequence length is 5.

Next, we amend this sample sequence set with the set of sequences guaranteed to be composed only from explicitly specified transitions in the incomplete ground truth reward machine. This latter sequence sample set is generated by finding all sequences of Boolean formula from the initial state to all other states via \emph{iterative deepening search}, resulting in all sequences with length \emph{at most} the number of states in the reward machine. Each Boolean formula sequence generates multiple sequences with elements from $2^\mathcal{P}$, by uniformly sampling elements from $2^\mathcal{P}$ that make the formula true. Explicitly, for a given path of length $d$ from the start state to depth $d$, where $d$ ranges from $2$ to $N$ the number of states in the automaton, we generate $d|2^\mathcal{P}|s$ sequences, where $s$ is a positive integer. Thus, for an given automaton with $N$ states, we have at most $\mathcal{O}(|2^\mathcal{P}|^N)$ paths, so we generate a test set of size $\mathcal{O}(|2^\mathcal{P}|^{N+1}ds)$. Each of the elements from the set $\{25,50,100,200,300,400,500, 1000, 2000, 5000\}$ was used as the value of $s$.

For tasks 1 through 4 of OfficeWorld and CraftWorld, this corresponded to between 35.4k to 36.2k samples. For tasks 5 through 10 of CraftWorld, this corresponded to 1.4m to 24.6m samples.

We evaluate each ground truth and learned reward machine pair with these sequences. Classification accuracy is the fraction of sequences in the sample set identically classified by the learned and ground truth reward machines.

\subsubsection{Regret Computation Details}\label{regretcomputation}
For the regret experiments, we needed to train and evaluate policies under the ground truth and learned reward machines. For a given reward machine, REMAP was run 100 times, producing a set of 100 reward machines per ground truth. Any differences in reward machines stemmed from how the teacher sampled test sequences and the order and length of counterexamples presented.

To compute regret, we computed \emph{empirical regret} between the optimal policy under the ground truth reward machine and the optimal policy of each of the learned reward machines.

We utilized Q-learning with counterfactual experiences for reward machines (CRM) \cite{toro-etal-jair22} to obtain optimal policies for each reward machine. We set the discount factor to $\gamma=0.9$, and for the OfficeWorld reward machines, each policy was trained for a total of $2\times 10^5$ steps, and $2\times 10^6$ steps for CraftWorld reward machines. Each reward machine was trained for at least 10 seeds. Observing the return curves, we concluded that by $1\times 10^5$ steps, the policy was already optimal for OfficeWorld domains, and by $1\times10^6$ steps for CraftWorld domains. We compute the average reward per step of the optimal policy by summing the total return of the policy over the last $10^5$ (OfficeWorld) or $10^6$ (CraftWorld) steps:\begin{align*}
\text{Average Reward per Step for Seed $k$} &= \frac{1}{\Delta s}\int_s^{s+\Delta s} r_{k,t}dt\\ &= \hat{R}_k\\
\text{Empirical Average Reward per Step} &= \frac{1}{N}\sum_{k=1}^N\hat{R}_k\\
&= \hat{\mathbb{E}}[\hat{R}],
\end{align*} where $r_{k,t}$ represents the reward received for seed $k$ at step $t$, and $s$ and $s+\Delta s$ represent the interval of steps the average is taken over, and $N$ represents the total number of seeds. Then, the average regret plotted in the paper was the difference between the empirical average reward per step of the optimal policy induced by the ground truth reward machine, and the empirical average reward per step of the optimal policies from the corresponding learned reward machines. Sample variance was computed for the ground truth reward machines via $$\text{Var}(\hat{R}) = \frac{1}{N-1}\sum_{k=1}^N (\hat{R}_k - \hat{\mathbb{E}}[\hat{R}])^2$$ with standard deviation computed via taking the square root of the sample variance.

\subsection{Example Comparing REMAP and \lstar{}}
Figures \ref{fig:example-lstar} and \ref{fig:example} of the main paper contained an abridged example comparing REMAP and \lstar{}. We show the full example in Figure \ref{fig:appendix_example} of the Appendix showing how REMAP and \lstar{} learn the Moore machine shown in part \emph{(l)}.

If the ground truth is to be interpreted as the Moore equivalent of a reward machine, then the ground truth Moore machine has 3 states, with $q_b$ as the terminal state, $q_{ba}$ is the absorbing HALT state, and $q_\varepsilon$ as the initial state; the corresponding ground truth reward machine would have only two states---$q_\varepsilon$ as the initial state, and $q_{b}$ as the terminal state, with no transitions out of $q_b$; there would be no absorbing $q_{ba}$ state. Additionally, the transition from $q_\varepsilon$ to $q_b$ via $b$ would have a reward of $1$ associated with it, and the self-transition $q_\varepsilon$ to $q_\varepsilon$ via $a$ would have a reward of $0$ associated with it; the states would no longer have rewards associated with them.

If the input alphabet is interpreted as the powerset of a proposition set $\mathcal{P}=\{p\}$, with $2^\mathcal{P}=\Sigma^I=\{\{\},\{p\}\}$, then we can use $a=\{\}$ and $b=\{p\}$ for convenience, and where $\Sigma^O=\{0,1\}$. As shown in the example, REMAP uses a symbolic observation table and performs preference queries, and closedness and consistency tests are evaluated with respect to a unified table. \lstar{} uses a concrete observation table and uses membership queries.
\begin{figure*}[t]
    \centering
    \includegraphics[height=0.24\textwidth]{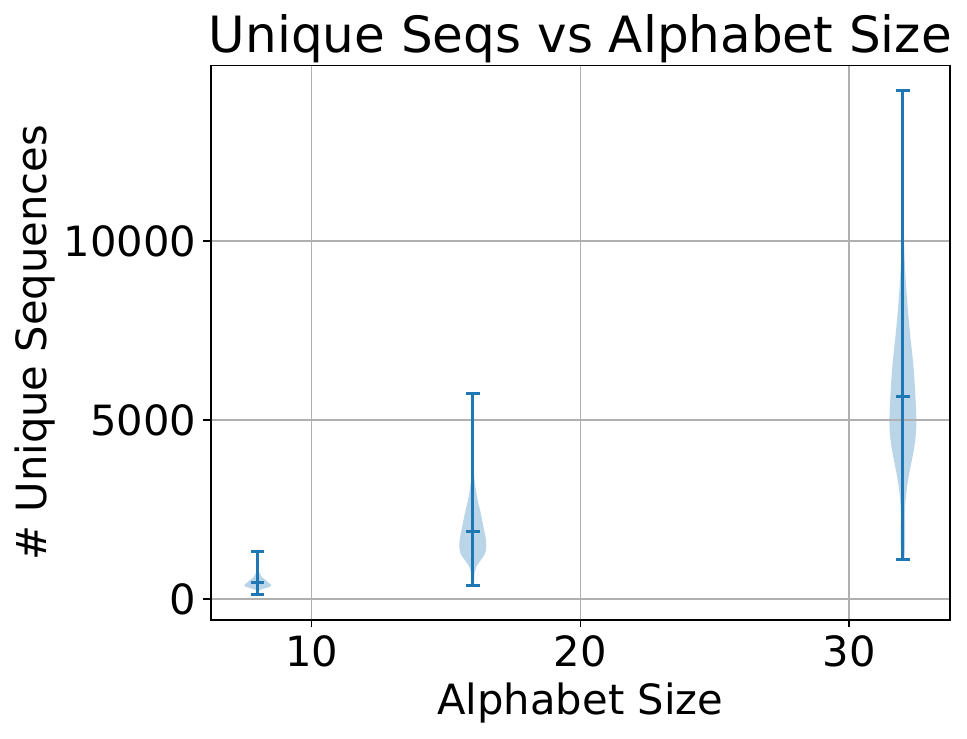}
    \includegraphics[height=0.24\textwidth]{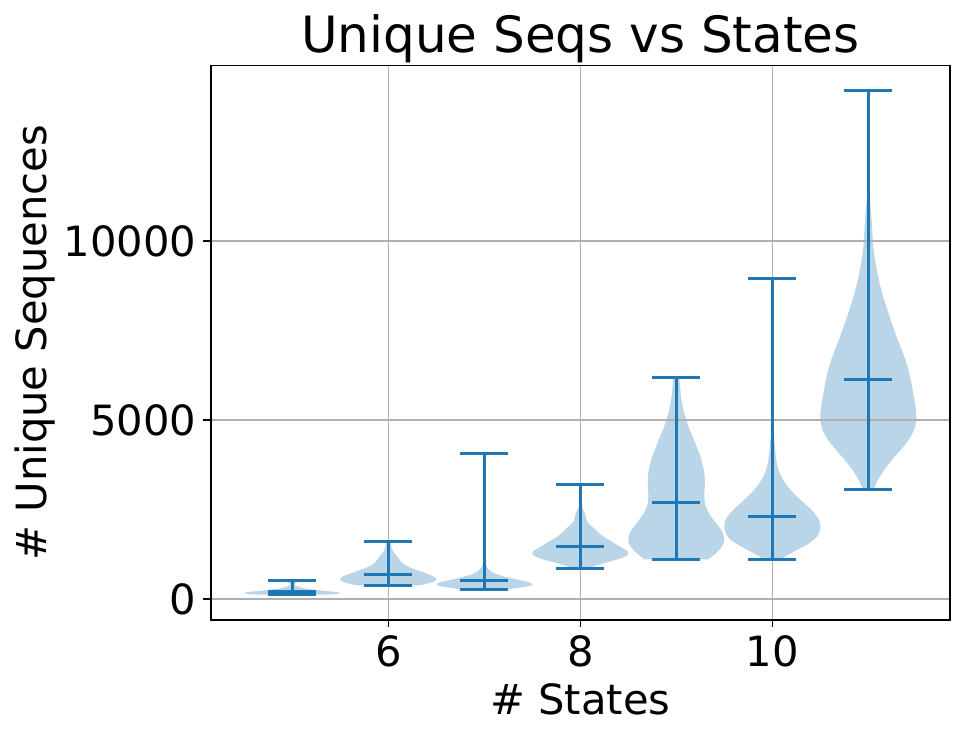}
    \includegraphics[height=0.24\textwidth]{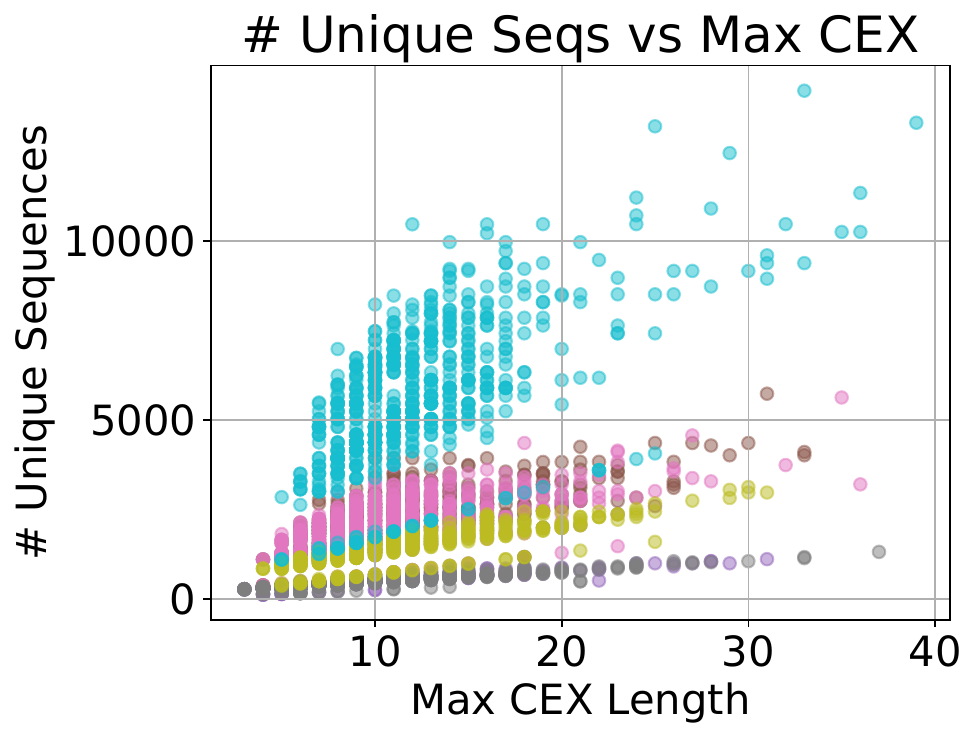}
    \caption{Empirical Scaling Plots. \textit{Left to right}: Under an inexact PAC-identification teacher, empirical distributions of how the total number of sequences in the table upon termination of \ouralgorithm{} depends on alphabet size, number of states, and maximum counterexample length (color coded using data from CraftWorld tasks C-T5 through C-T10, see Figure~\ref{fig:correctness} legend).}
    \label{fig:scaling}
\end{figure*}
\subsection{Additional Scaling Measurements}
Since we measured query complexity as a function of the number of unique sequences present in the table upon termination, we also measured the number of unique sequences as functions of input alphabet size, number of states, and length of the maximum counterexample received. These results are shown in Figure \ref{fig:scaling}.

\subsection{Termination Plots}
We present a full set of termination phase diagrams for learning the CraftWorld reward machines for tasks 5 through 10 in Figures \ref{fig:termination_plots_appendix_1}, \ref{fig:termination_plots_appendix_2}, \ref{fig:termination_plots_appendix_3}, and \ref{fig:termination_plots_appendix_4}. Columns correspond to tasks, while rows correspond to number of samples the teacher make per equivalence query. Each plot contains 100 paths through the termination phase space, where the $x$-axis is the number of known representative values, and the $y$-axis is the number of states in the hypothesis. Each node along the path corresponds with an event in \ouralgorithm{}. The start of each path always starts at $(0,1)$, since there is always an initial state, but no known representatives. Green $X$'s represent when an equivalence query is made. Blue circles represent the immediate result of a closure operation, while orange squares represent the immediate result of a consistency operation. The red star represents the upper bound on number of states and number of known representative values. Observe that between consecutive equivalence queries, at least one of the number of states or the number of known representative values must increase: in particular, if the number of known representative values does not increase as the result of an equivalence query, the number of states must increase. Additionally, it is possible for \ouralgorithm{} to terminate early, when the number of states has reached the upper bound, and when the satisfying solution to the constraints has all correct values.
\begin{figure*}
    \centering
    \includegraphics[height=0.245\textwidth]{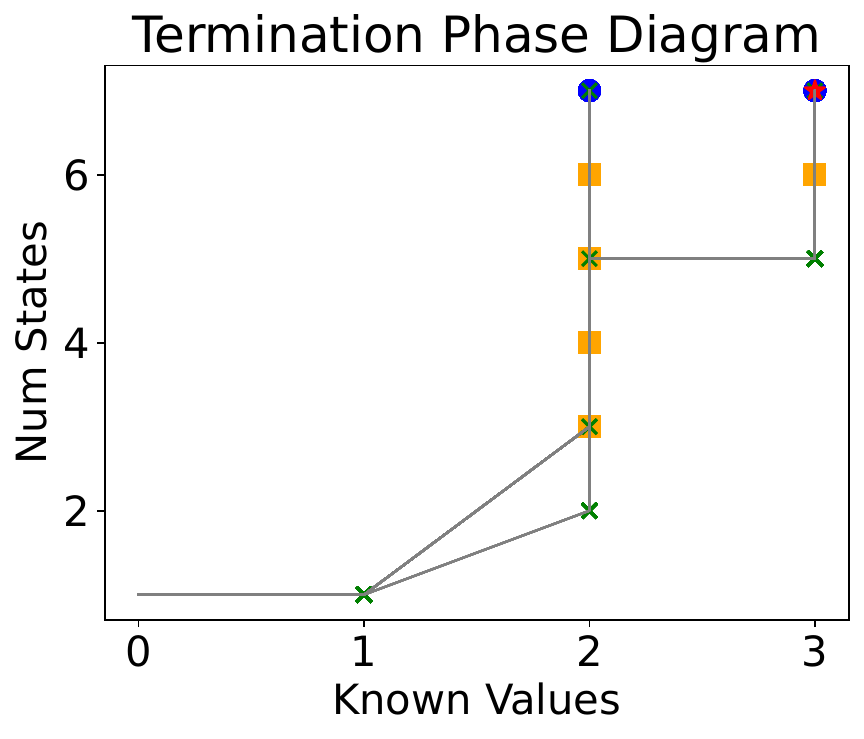}
    \includegraphics[height=0.245\textwidth]{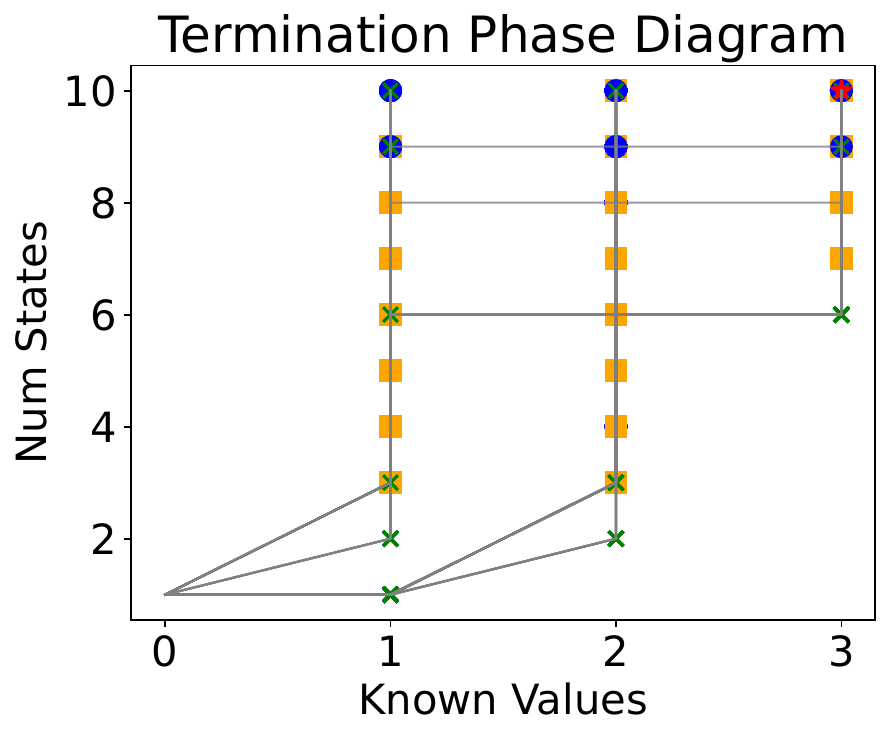}
    \includegraphics[height=0.245\textwidth]{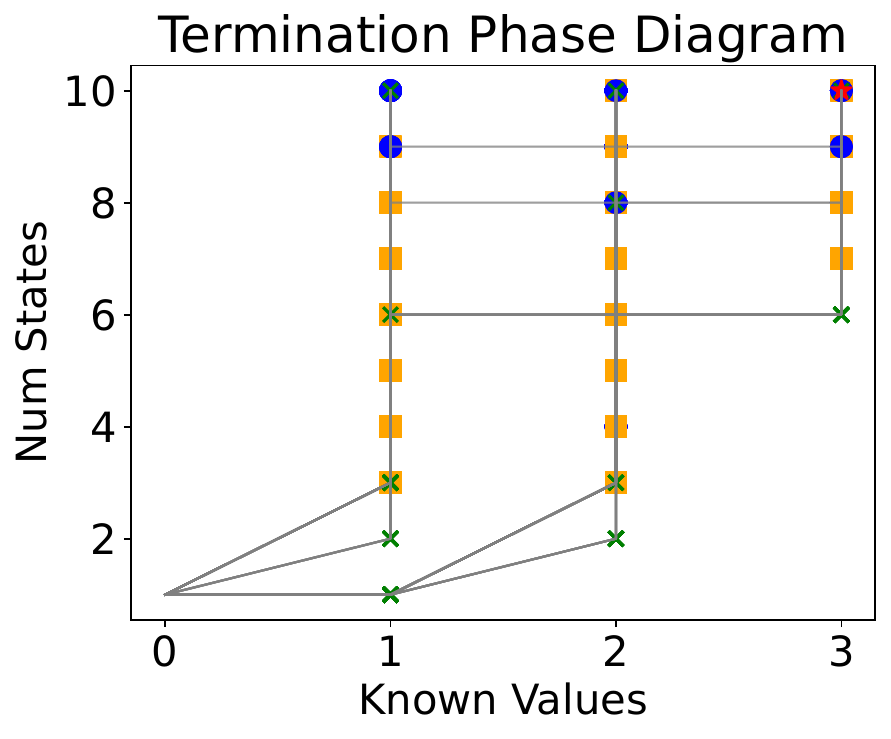}\\
    \includegraphics[height=0.245\textwidth]{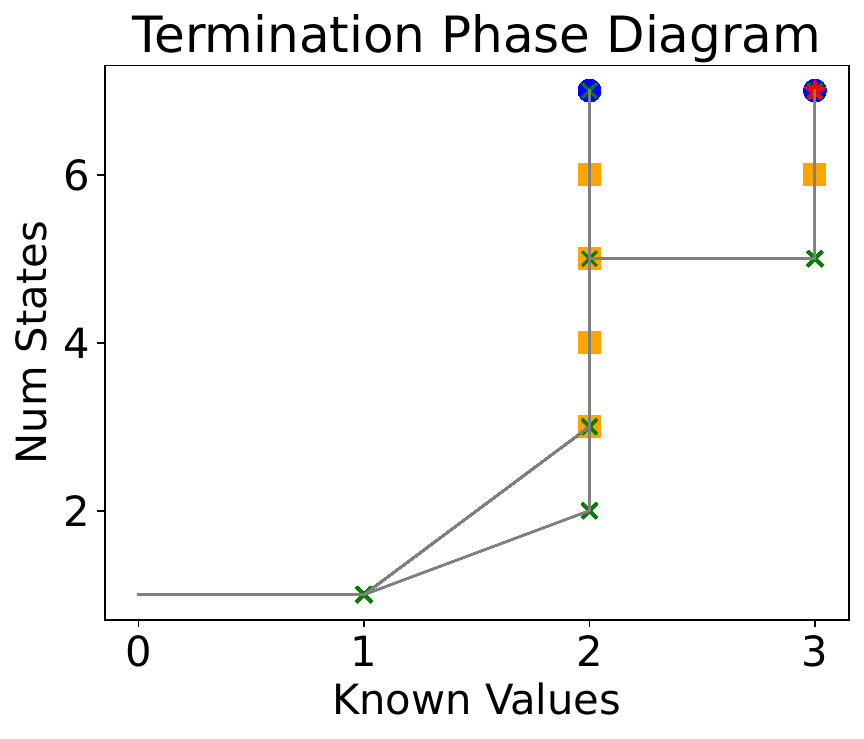}
    \includegraphics[height=0.245\textwidth]{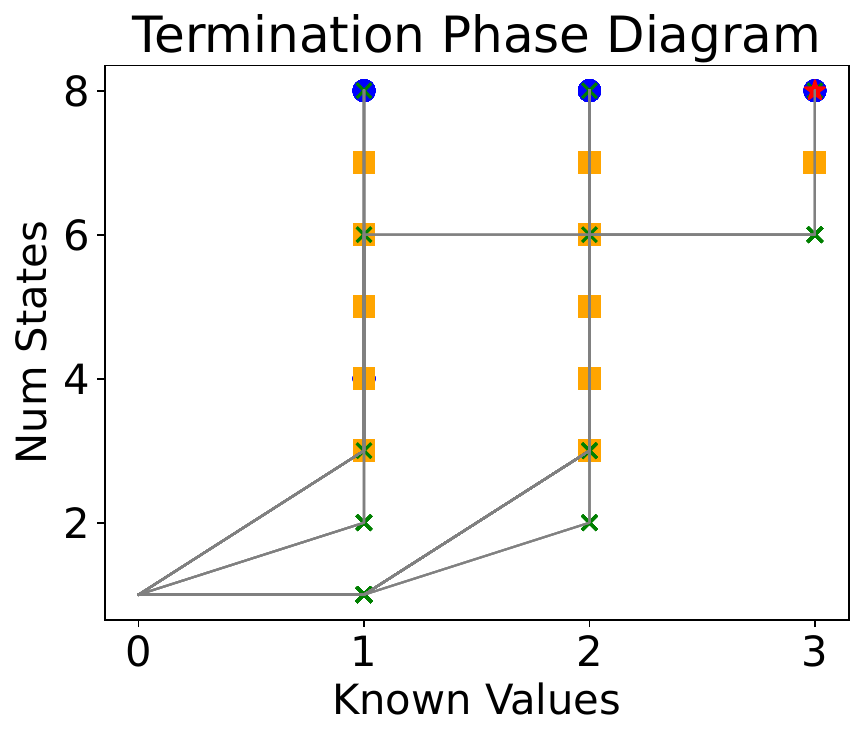}
    \includegraphics[height=0.245\textwidth]{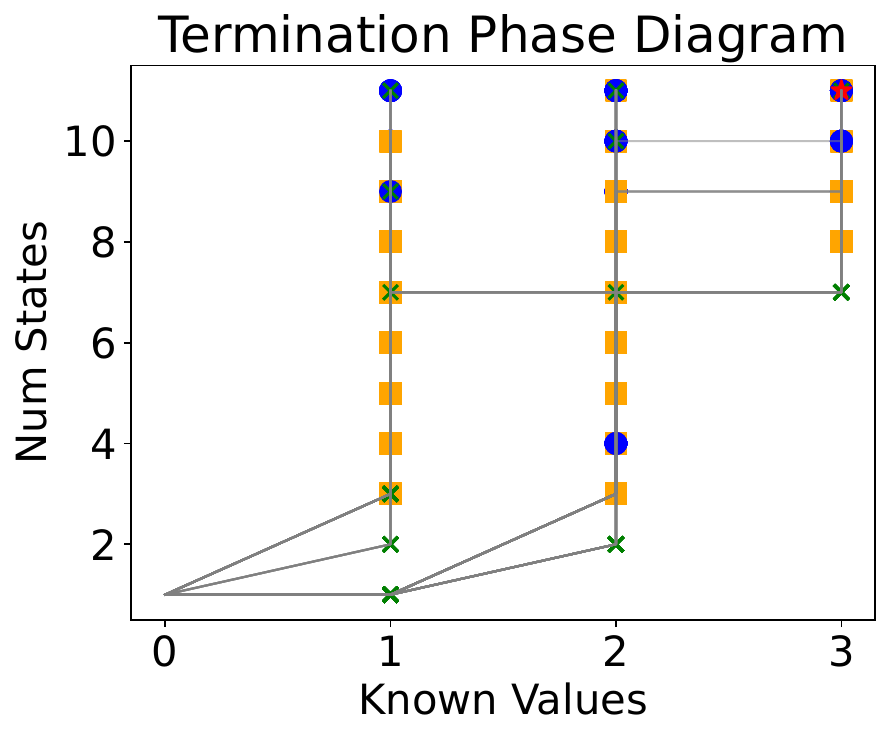}\\\rule{\textwidth}{1pt}
    \includegraphics[height=0.245\textwidth]{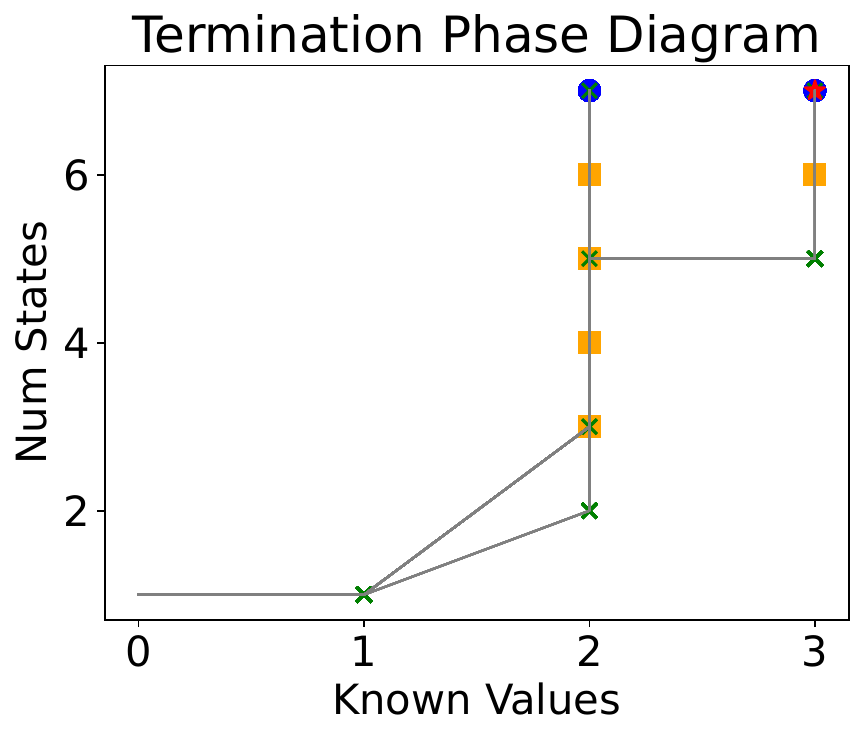}
    \includegraphics[height=0.245\textwidth]{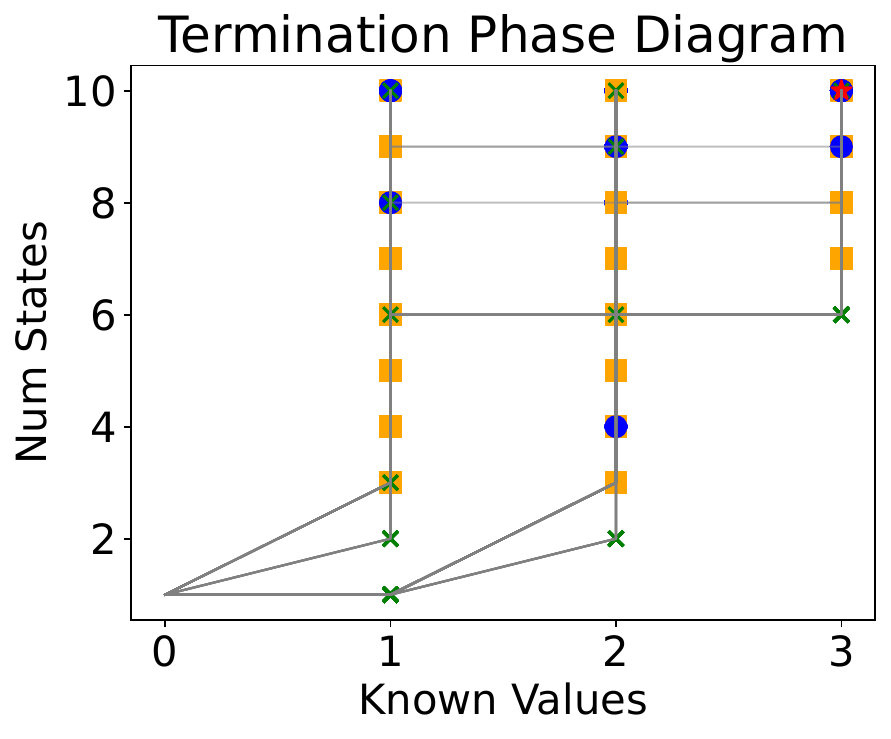}
    \includegraphics[height=0.245\textwidth]{figs/termination_plots/craft_termination_plot.50.t105.pdf}\\
    \includegraphics[height=0.245\textwidth]{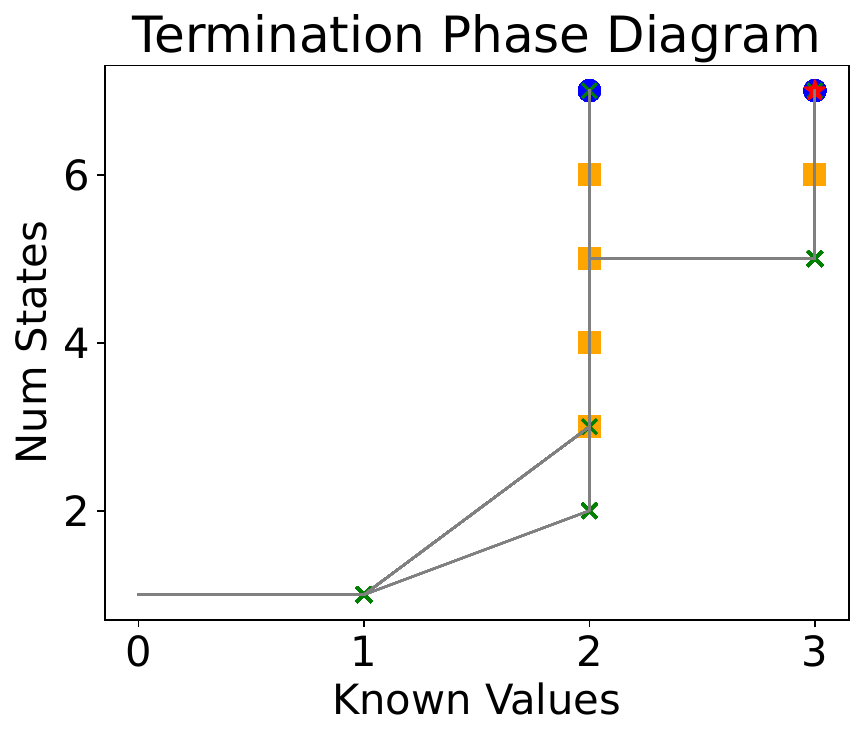}
    \includegraphics[height=0.245\textwidth]{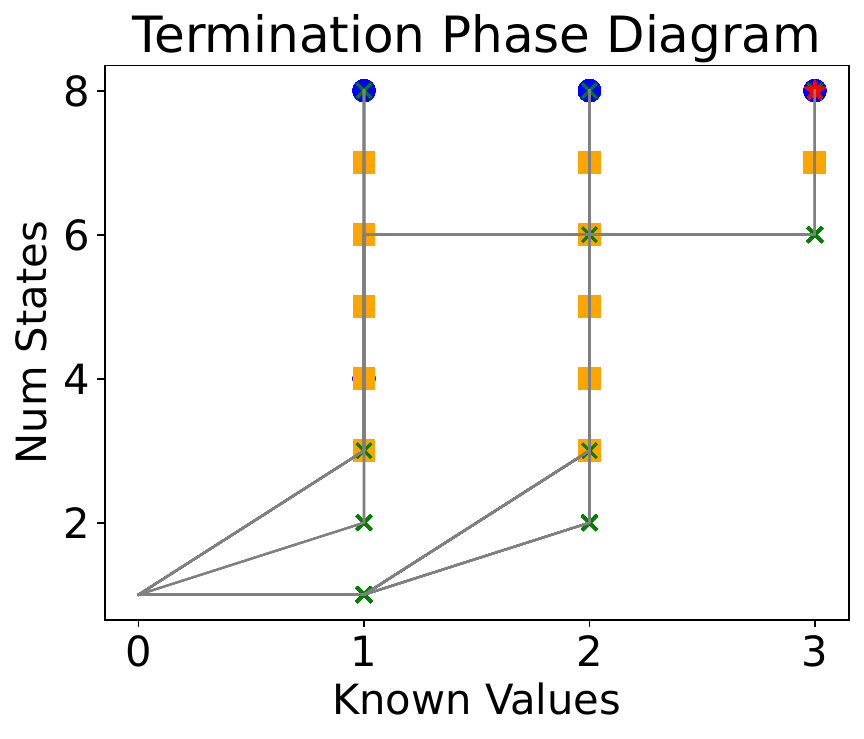}
    \includegraphics[height=0.245\textwidth]{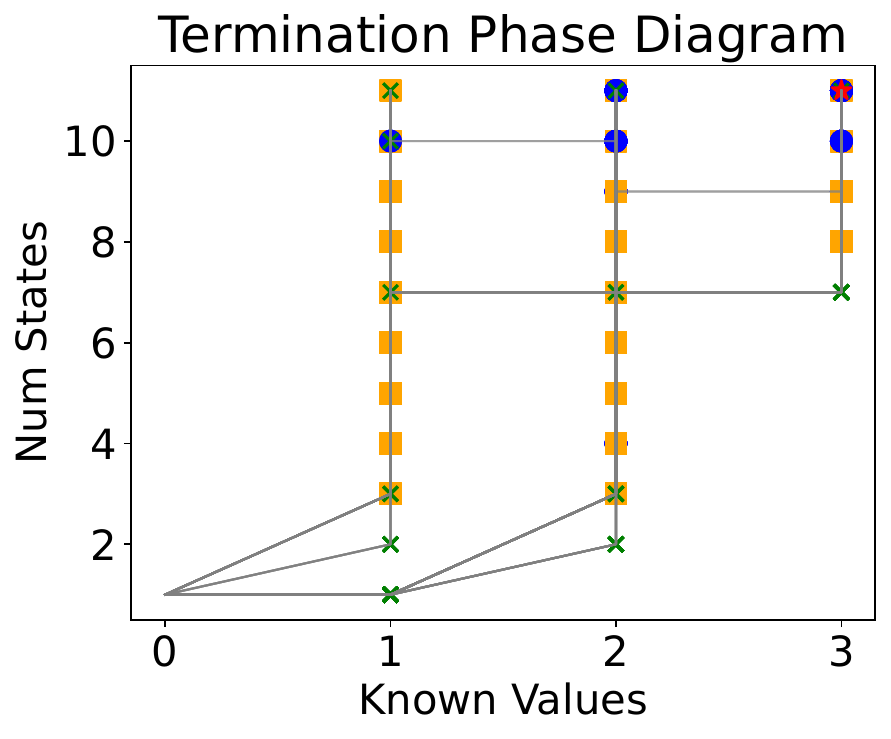}
    \caption{Phase diagram termination plots for CraftWorld tasks 5 through 10. Teacher tests 25 (top section) and 50 (bottom section) samples per equivalence query. Within each section: left to right (top row) Task 5, 6, 7; left to right (bottom row) Task 8, 9, 10. The X-axis is number of known representative values in the hypothesis, the Y-axis is the number of states in the hypothesis. Legend: blue circle is a closure operation, orange square is a consistency operation, green x is an equivalence query, and the red star represents the the upper bound termination condition. We observe it is possible to terminate early, prior to reaching the upper bound. Each plot contains 100 individual paths through phase space.}
    \label{fig:termination_plots_appendix_1}
\end{figure*}
\begin{figure*}
    \includegraphics[height=0.245\textwidth]{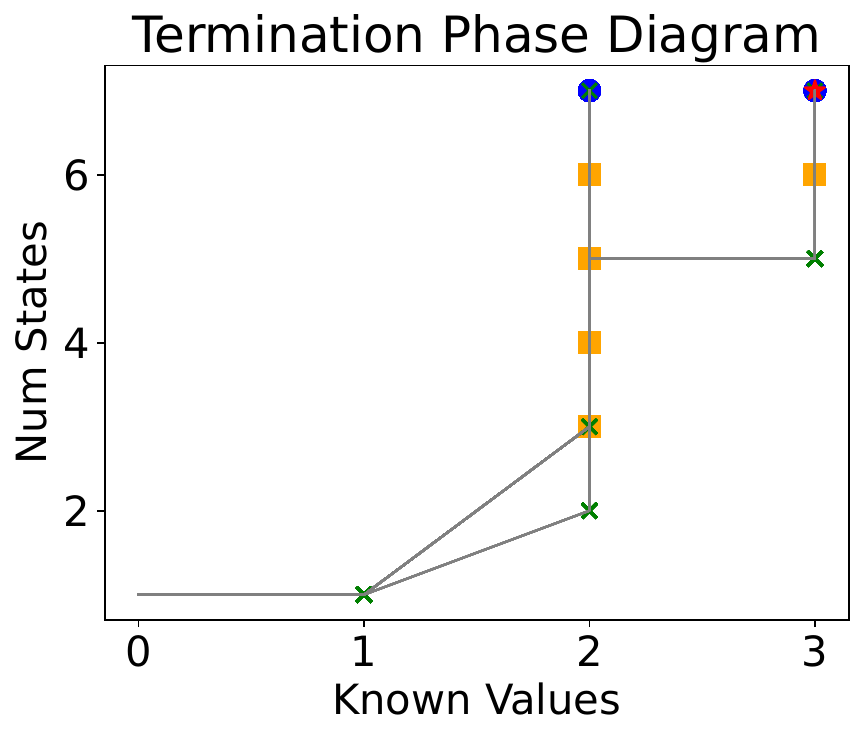}
    \includegraphics[height=0.245\textwidth]{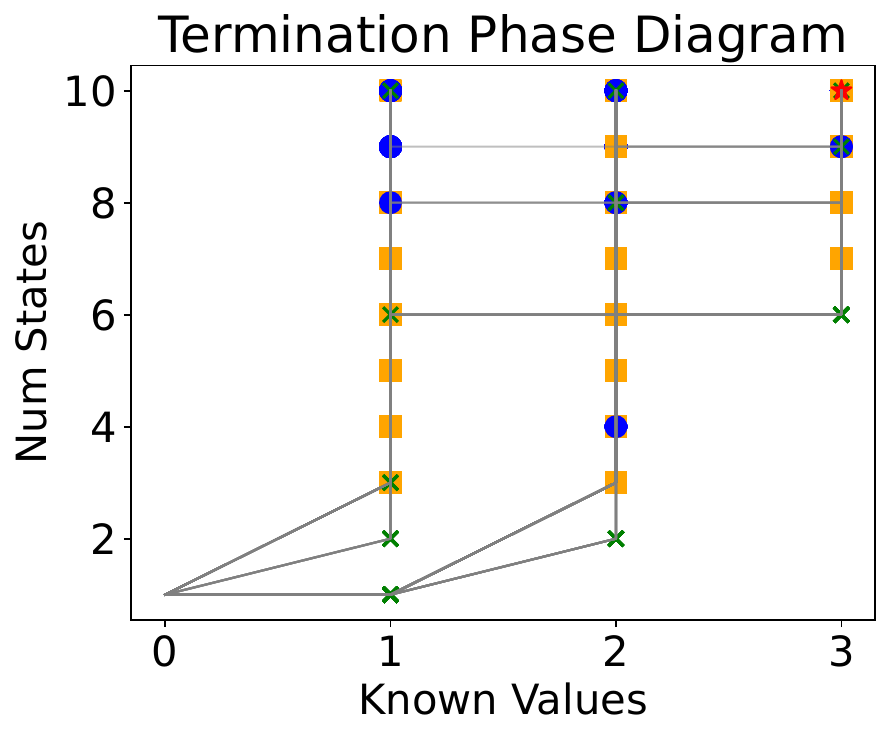}
    \includegraphics[height=0.245\textwidth]{figs/termination_plots/craft_termination_plot.100.t105.pdf}\\
    \includegraphics[height=0.245\textwidth]{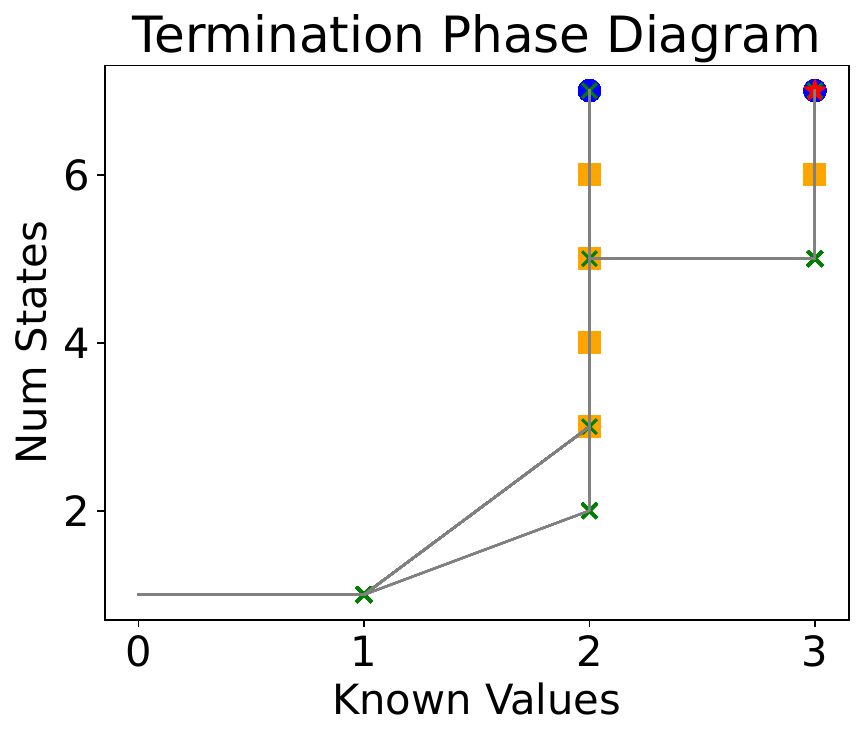}
    \includegraphics[height=0.245\textwidth]{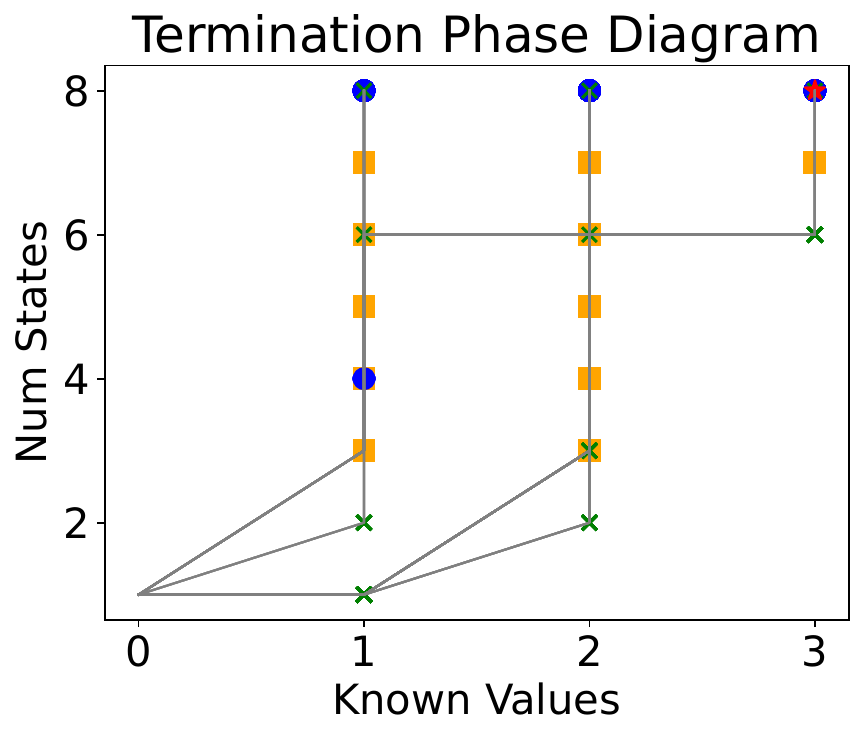}
    \includegraphics[height=0.245\textwidth]{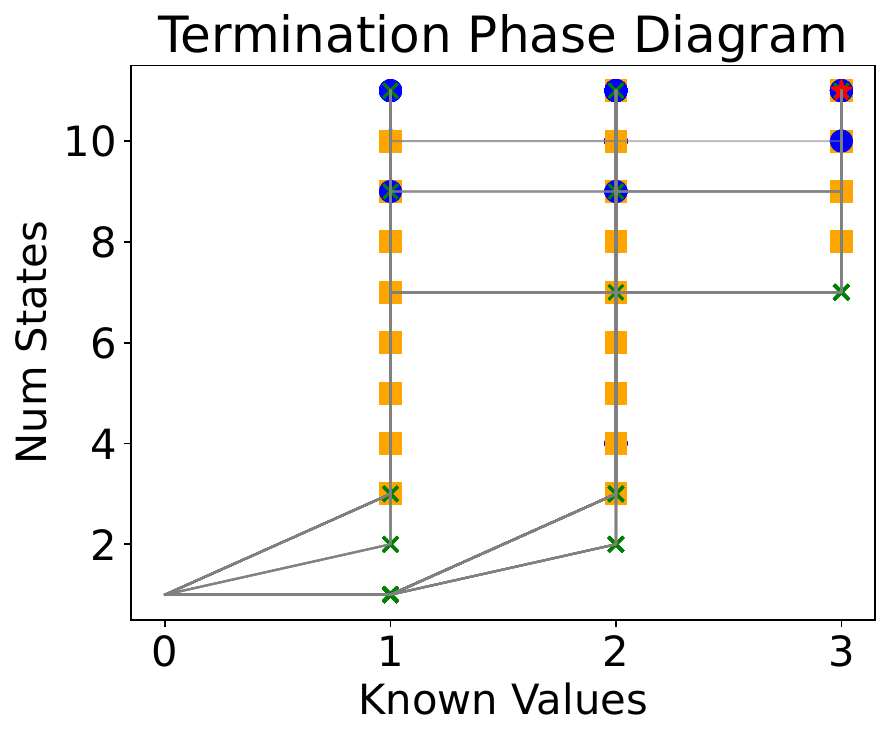}\\\rule{\textwidth}{1pt}
    \includegraphics[height=0.245\textwidth]{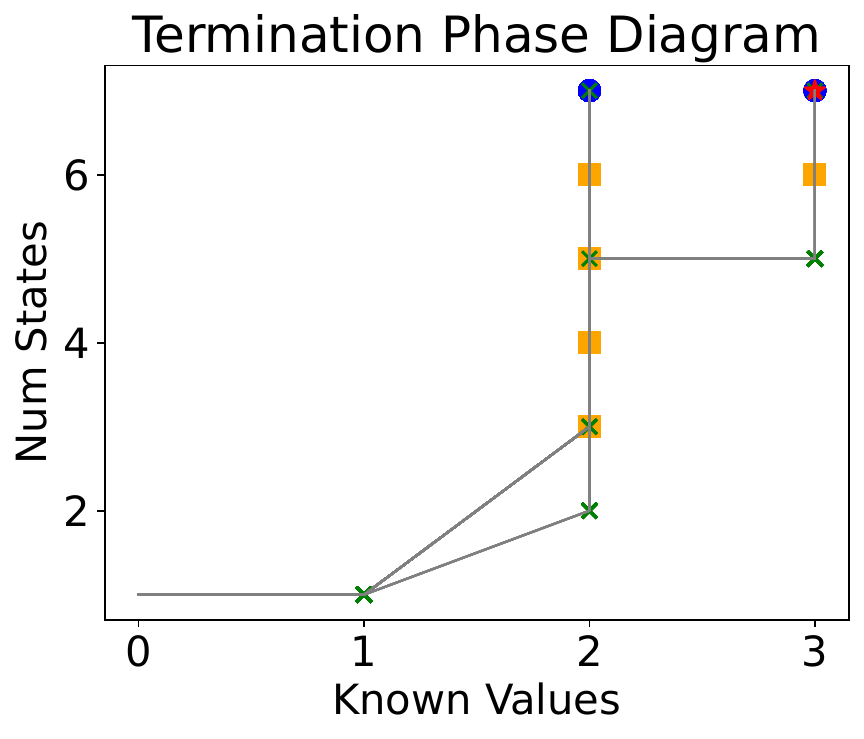}
    \includegraphics[height=0.245\textwidth]{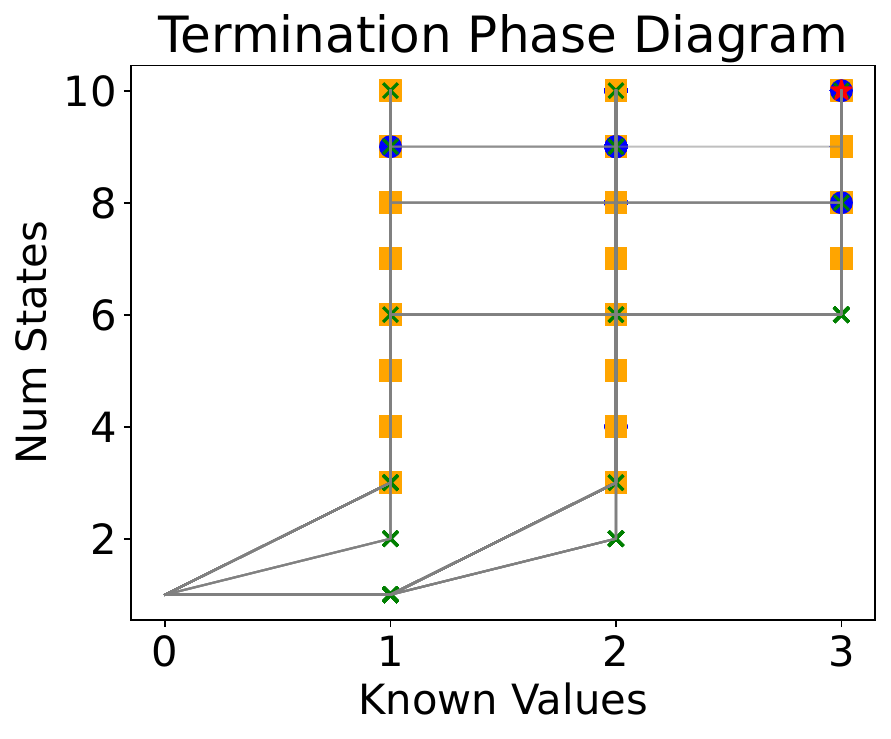}
    \includegraphics[height=0.245\textwidth]{figs/termination_plots/craft_termination_plot.200.t105.pdf}\\
    \includegraphics[height=0.245\textwidth]{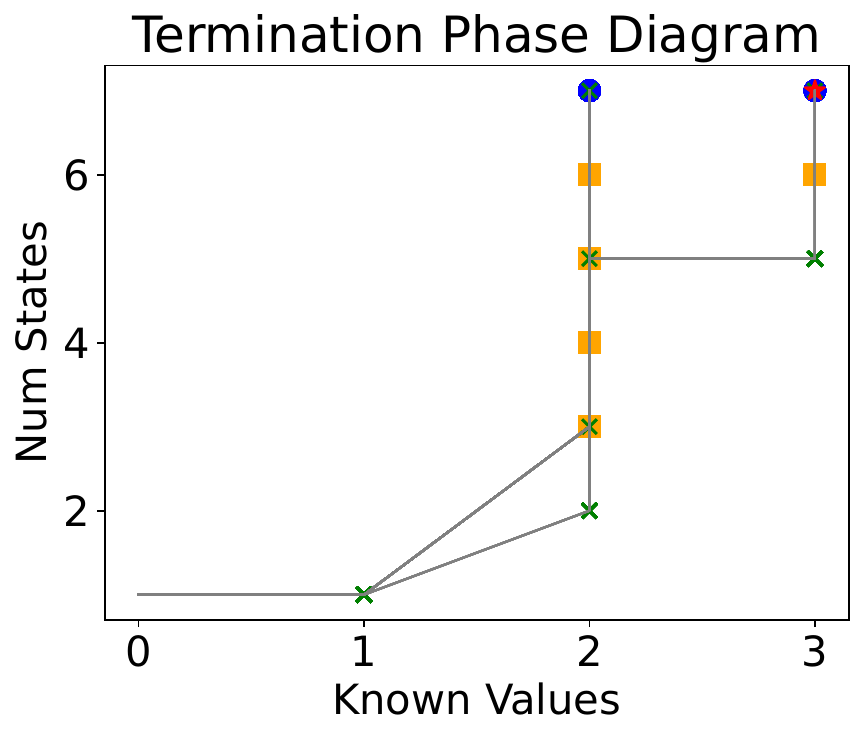}
    \includegraphics[height=0.245\textwidth]{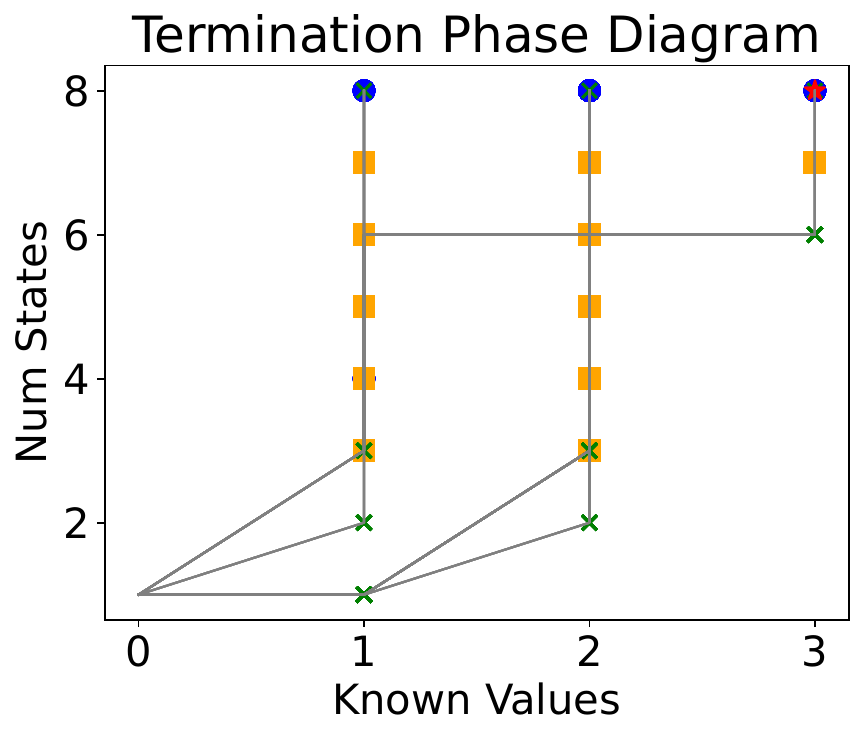}
    \includegraphics[height=0.245\textwidth]{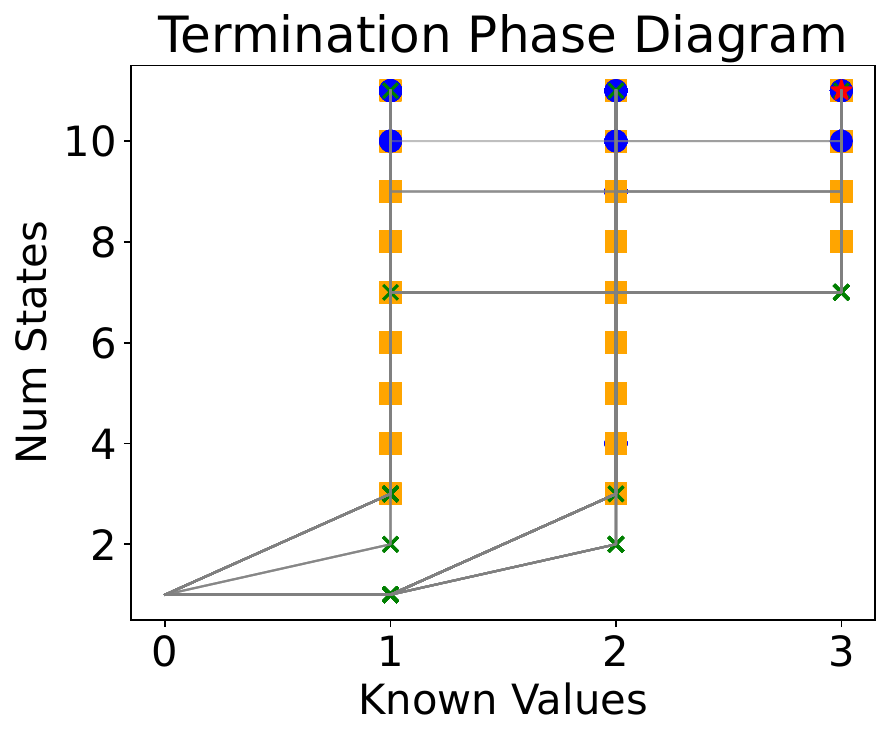}
    \caption{Phase diagram termination plots for CraftWorld tasks 5 through 10. Teacher tests 100 (top section) and 200 (bottom section) samples per equivalence query. Within each section: left to right (top row) Task 5, 6, 7; left to right (bottom row) Task 8, 9, 10. The X-axis is number of known representative values in the hypothesis, the Y-axis is the number of states in the hypothesis. Legend: blue circle is a closure operation, orange square is a consistency operation, green x is an equivalence query, and the red star represents the the upper bound termination condition. We observe it is possible to terminate early, prior to reaching the upper bound. Each plot contains 100 individual paths through phase space.}
    \label{fig:termination_plots_appendix_2}
\end{figure*}
\begin{figure*}
    \includegraphics[height=0.245\textwidth]{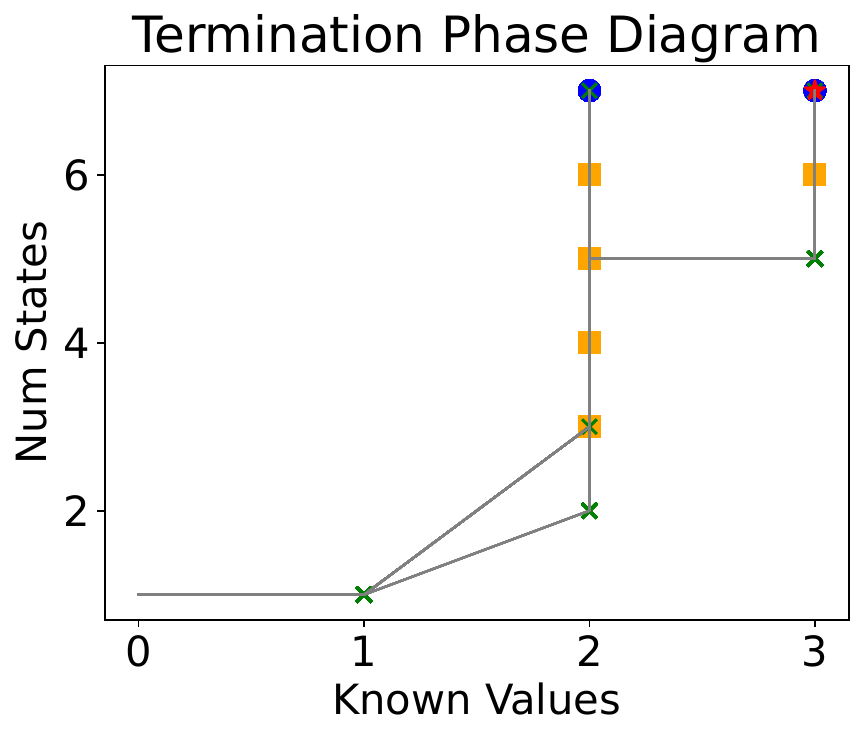}
    \includegraphics[height=0.245\textwidth]{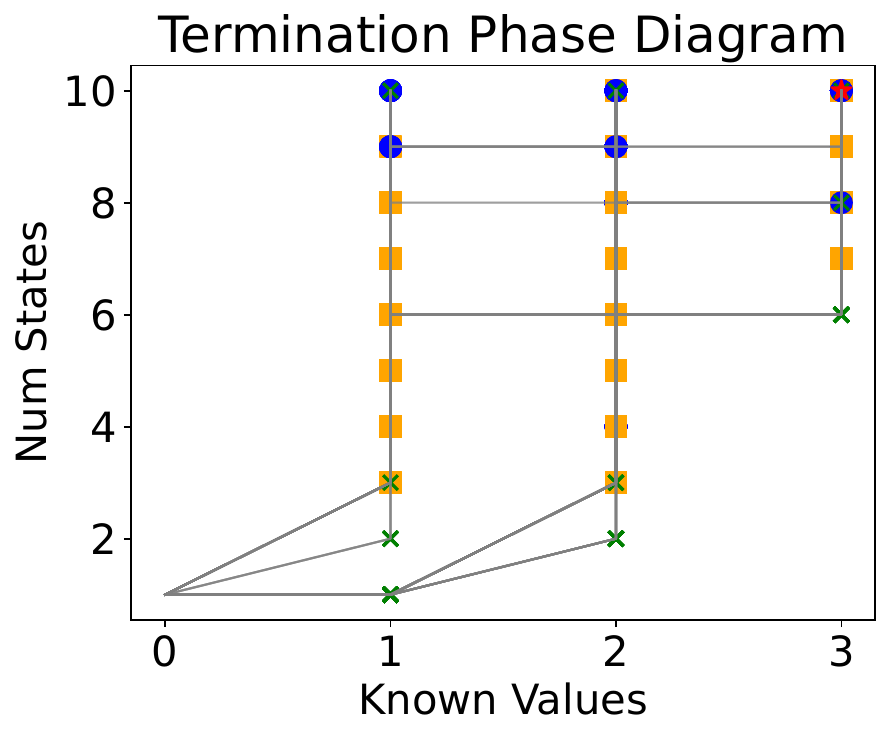}
    \includegraphics[height=0.245\textwidth]{figs/termination_plots/craft_termination_plot.300.t105.pdf}\\
    \includegraphics[height=0.245\textwidth]{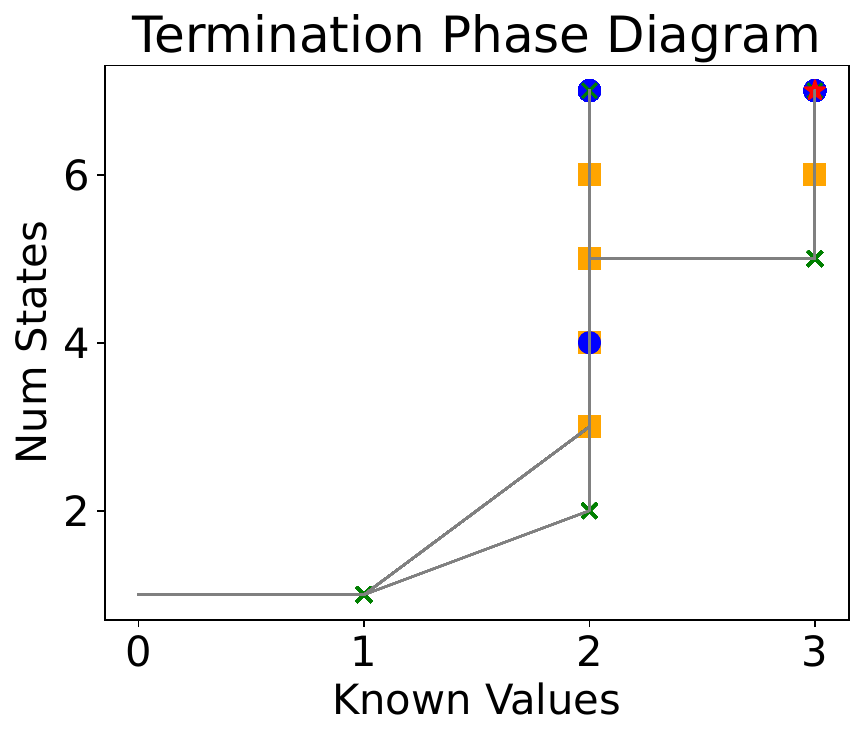}
    \includegraphics[height=0.245\textwidth]{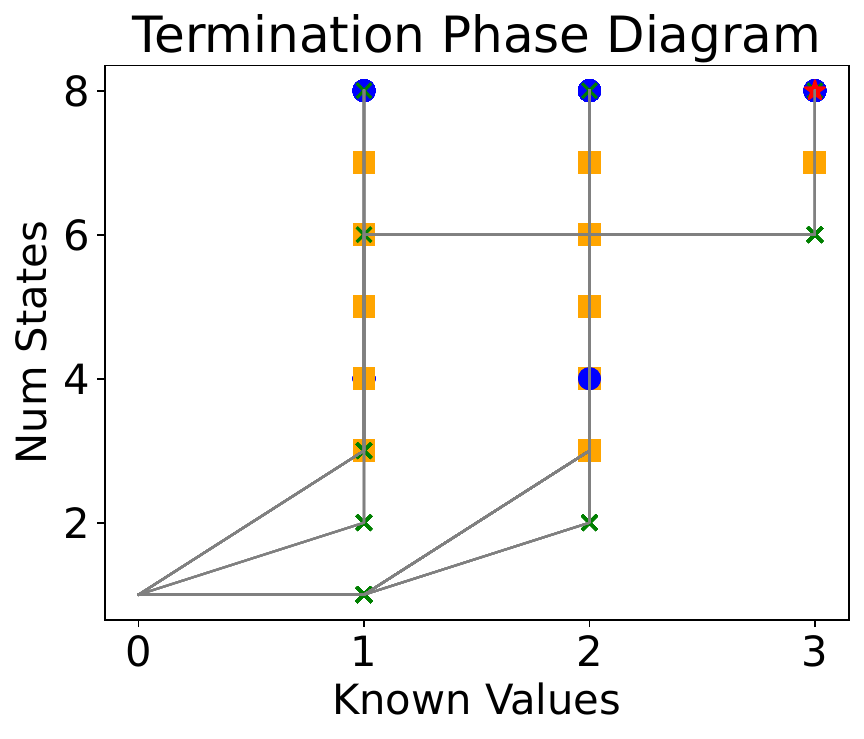}
    \includegraphics[height=0.245\textwidth]{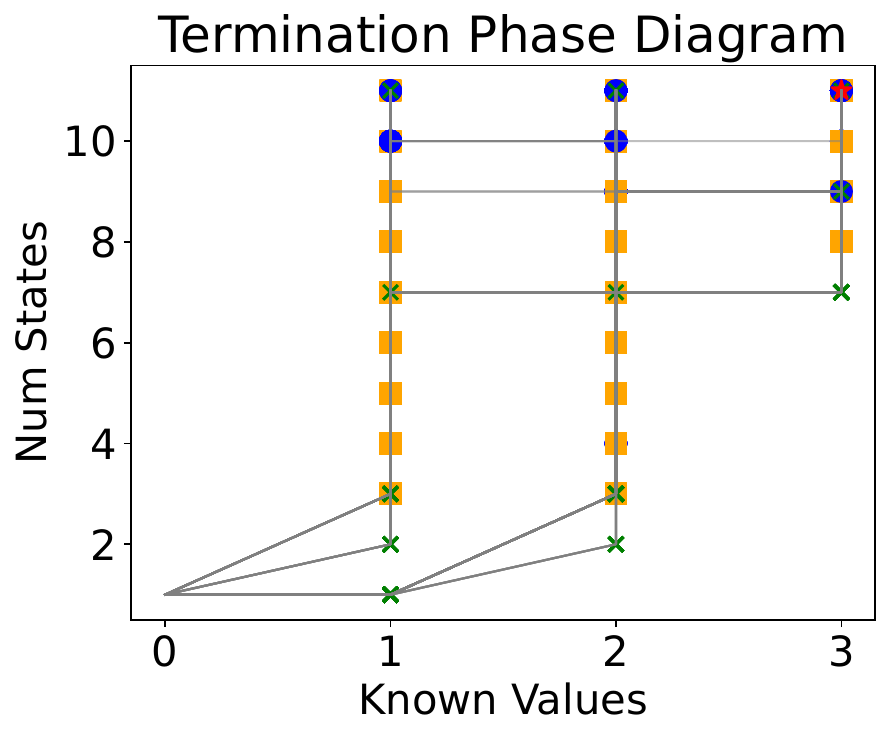}\\\rule{\textwidth}{1pt}
    \includegraphics[height=0.245\textwidth]{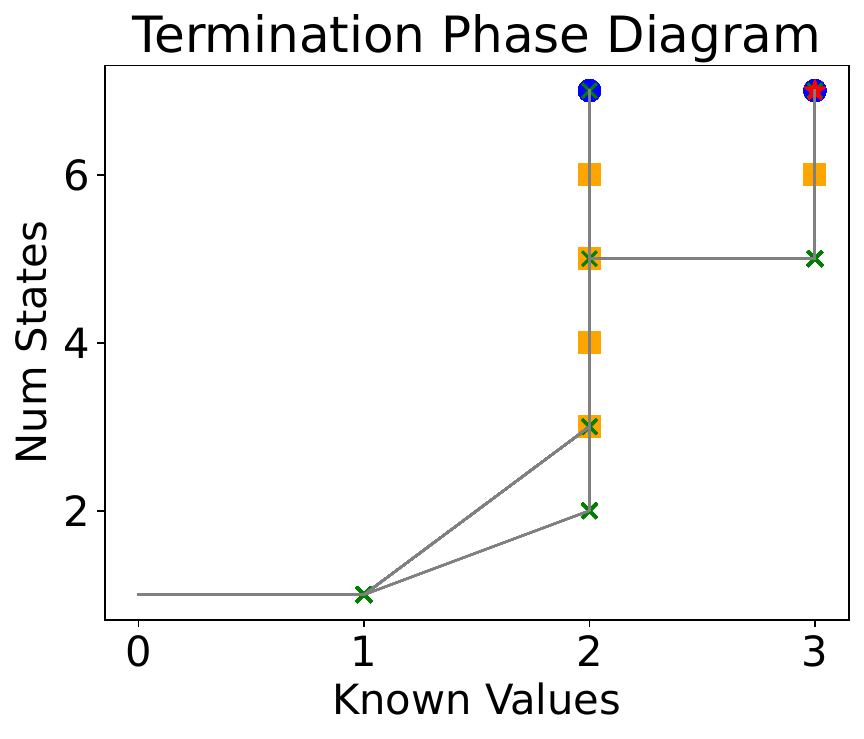}
    \includegraphics[height=0.245\textwidth]{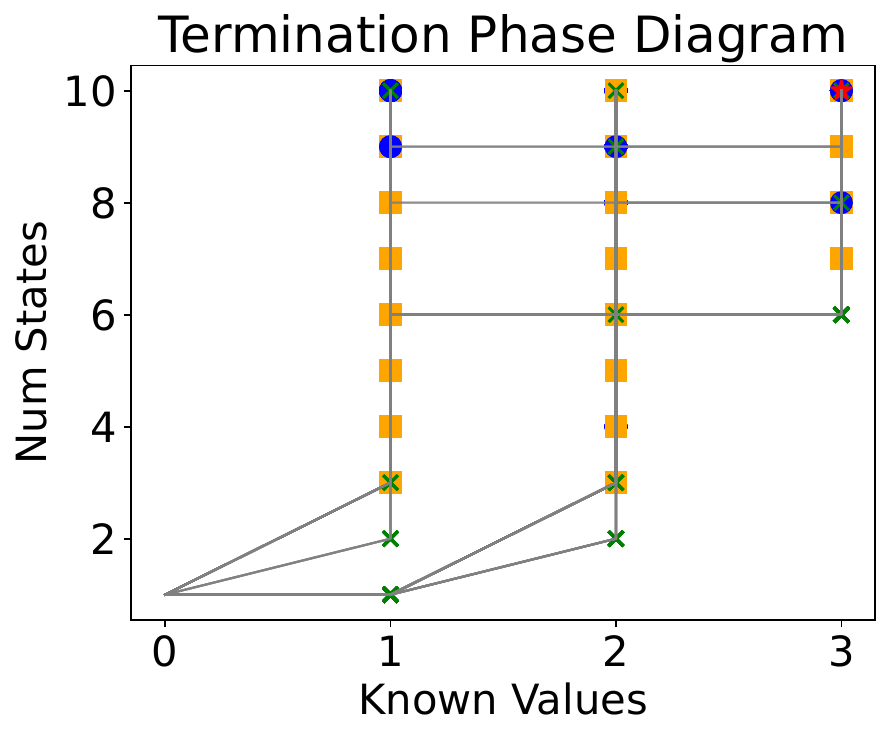}
    \includegraphics[height=0.245\textwidth]{figs/termination_plots/craft_termination_plot.400.t105.pdf}\\
    \includegraphics[height=0.245\textwidth]{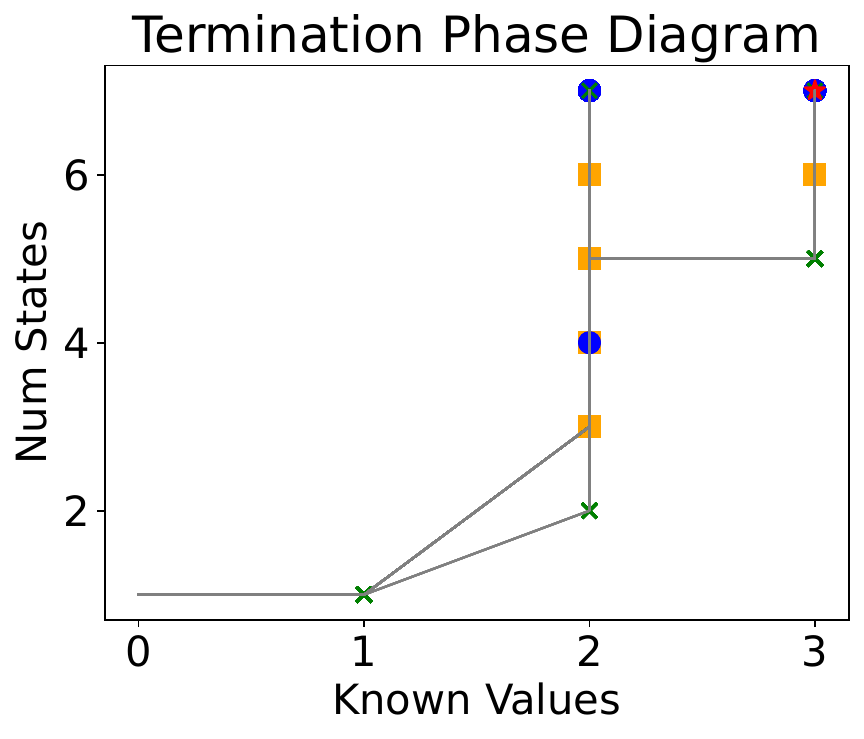}
    \includegraphics[height=0.245\textwidth]{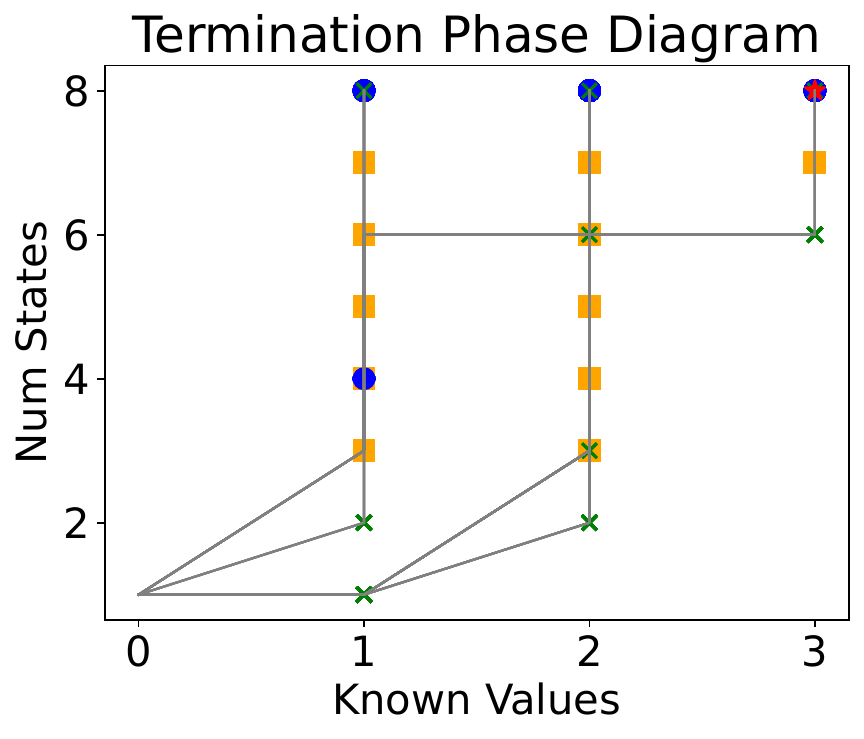}
    \includegraphics[height=0.245\textwidth]{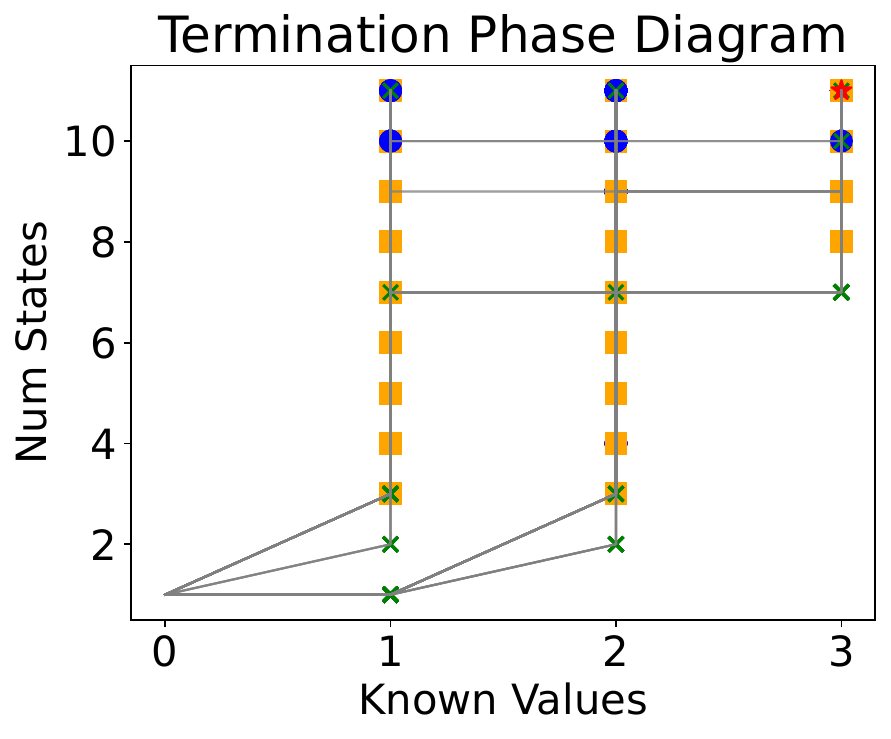}
    \caption{Phase diagram termination plots for CraftWorld tasks 5 through 10. Teacher tests 300 (top section) and 400 (bottom section) samples per equivalence query. Within each section: left to right (top row) Task 5, 6, 7; left to right (bottom row) Task 8, 9, 10. The X-axis is number of known representative values in the hypothesis, the Y-axis is the number of states in the hypothesis. Legend: blue circle is a closure operation, orange square is a consistency operation, green x is an equivalence query, and the red star represents the the upper bound termination condition. We observe it is possible to terminate early, prior to reaching the upper bound. Each plot contains 100 individual paths through phase space.}
    \label{fig:termination_plots_appendix_3}
\end{figure*}

\begin{figure*}
    \includegraphics[height=0.245\textwidth]{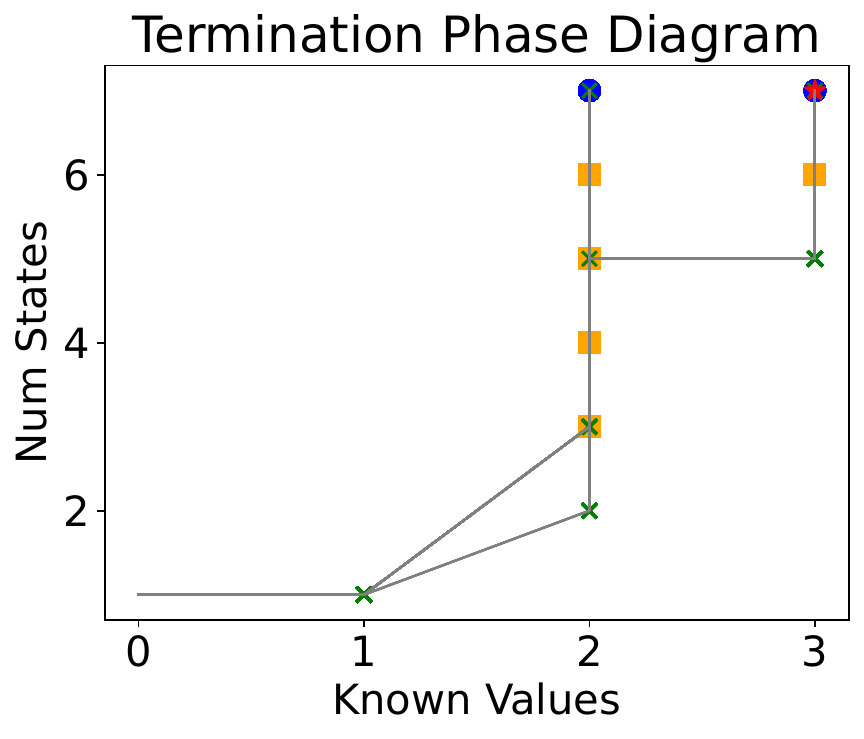}
    \includegraphics[height=0.245\textwidth]{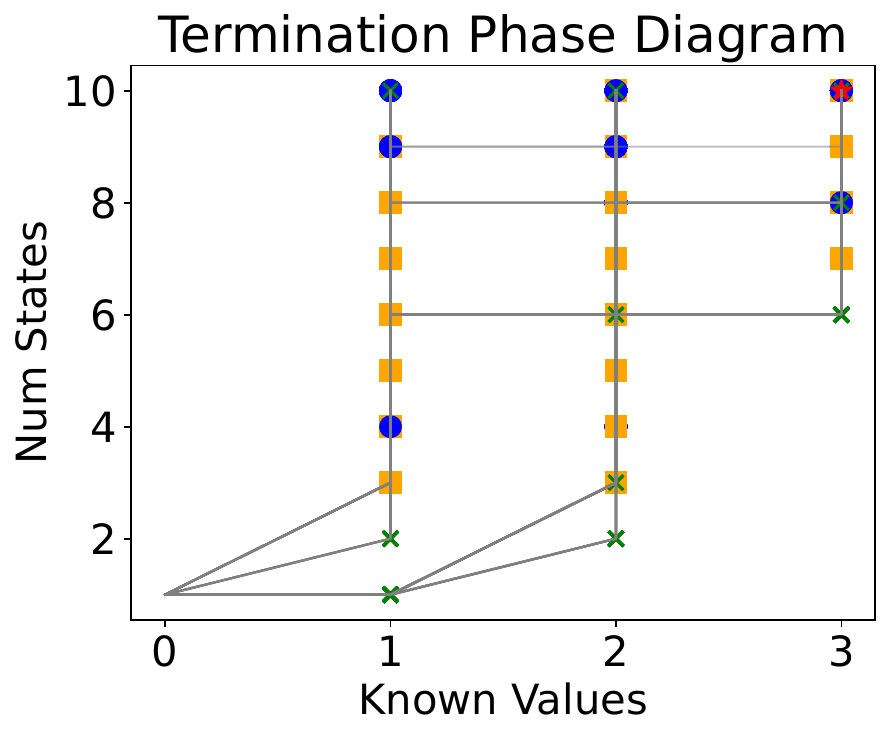}
    \includegraphics[height=0.245\textwidth]{figs/termination_plots/craft_termination_plot.500.t105.pdf}\\
    \includegraphics[height=0.245\textwidth]{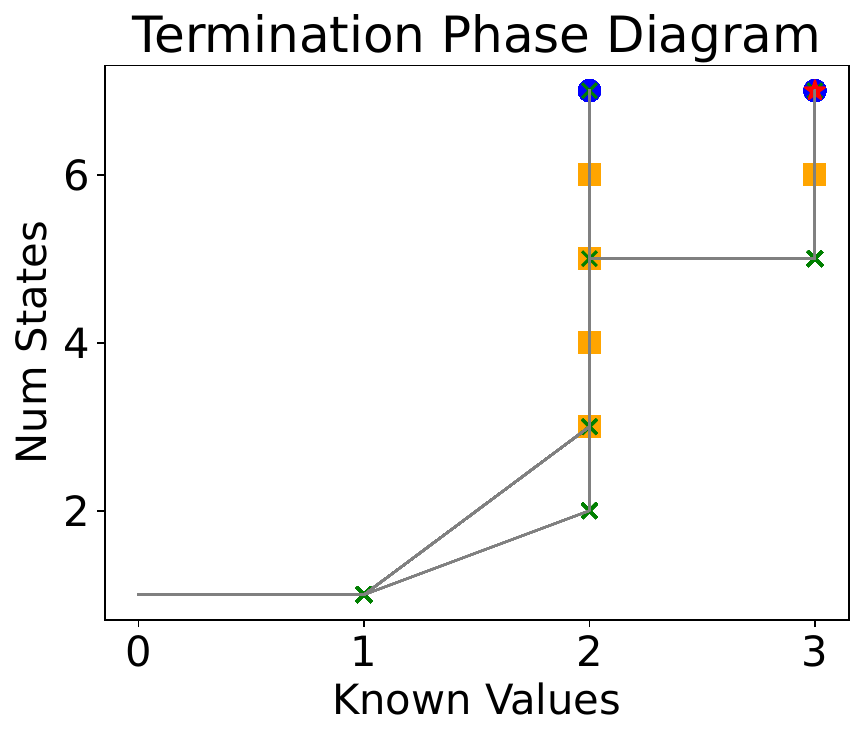}
    \includegraphics[height=0.245\textwidth]{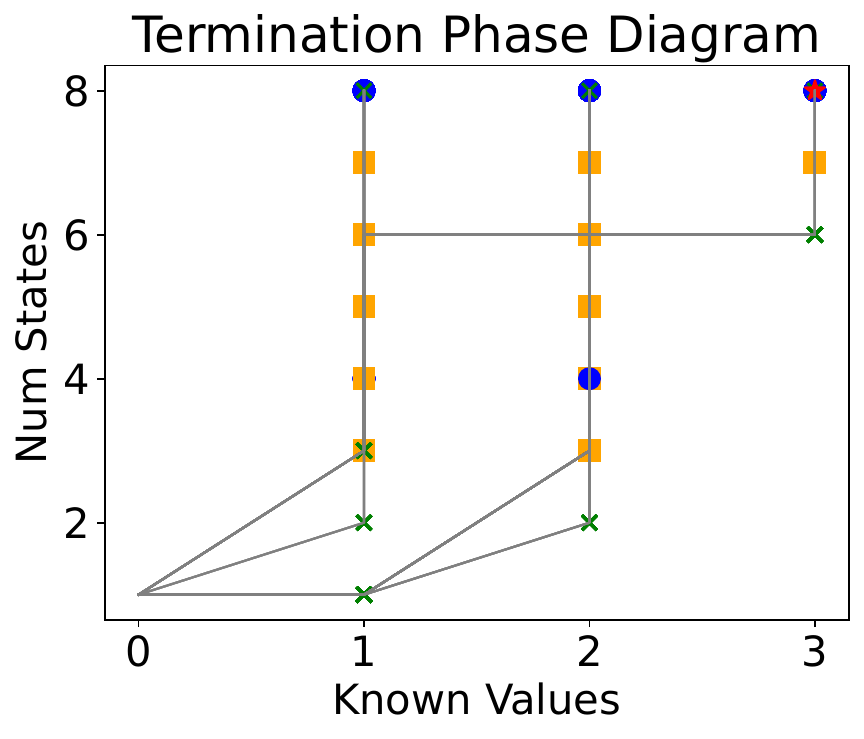}
    \includegraphics[height=0.245\textwidth]{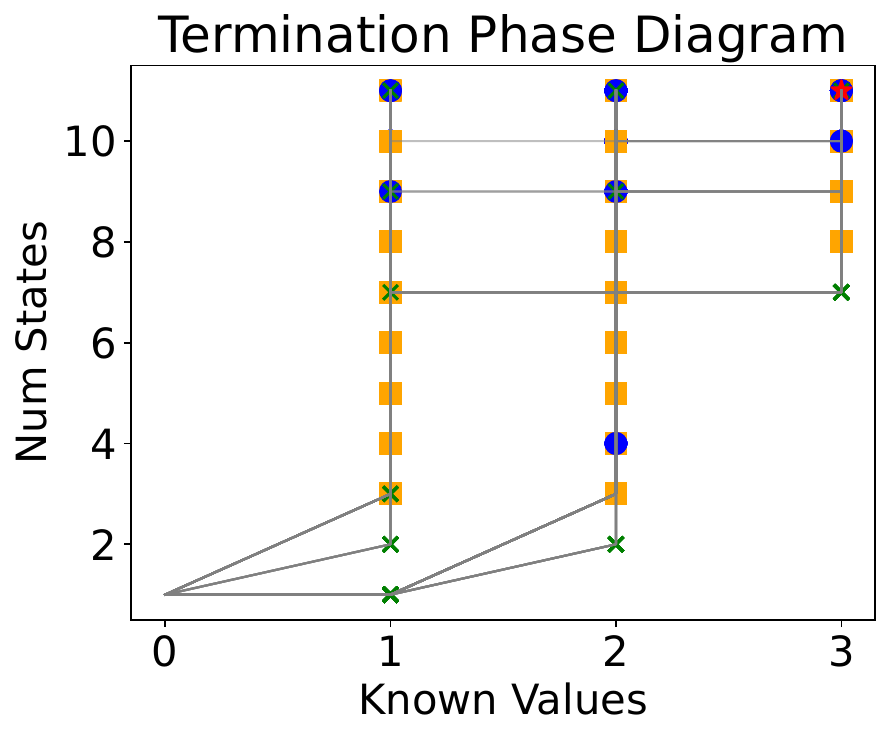}
    \caption{Phase diagram termination plots for CraftWorld tasks 5 through 10. Teacher tests 500 samples per equivalence query. Left to right (top row) Task 5, 6, 7; left to right (bottom row) Task 8, 9, 10. The X-axis is number of known representative values in the hypothesis, the Y-axis is the number of states in the hypothesis. Legend: blue circle is a closure operation, orange square is a consistency operation, green x is an equivalence query, and the red star represents the the upper bound termination condition. We observe it is possible to terminate early, prior to reaching the upper bound. Each plot contains 100 individual paths through phase space.}
    \label{fig:termination_plots_appendix_4}
\end{figure*}
\subsection{\ouralgorithm{} Algorithms}
In this section, we present the algorithms used in \ouralgorithm{}. Algorithm \ref{alg:pref-demo} presented in the body of the paper is an abbreviated version of Algorithm \ref{alg:full-pref-demo} (\ouralgorithm{}). In Algorithm \ref{alg:full-pref-demo}, we expand the function \textsc{MakeClosedAndConsistent} into loop for making the symbolic observation table unified, closed, and consistent. We also present the Symbolic Fill Procedure in Algorithm \ref{alg:symbolic-fill}, which is responsible for (1) creating fresh variables for sequences which have not been queried, (2) performing preference queries, and (3) performing unification. Algorithm \ref{alg:preference-query} illustrates how preference queries are performed as comparisons, and also shows how the constraint set is updated with the return information. Algorithm \ref{alg:make-hypothesis} is responsible for constructing a hypothesis from a unified, closed, and consistent observation table, and it includes the \textsc{FindSatisfyingSolution} (abbreviated as \textsc{FindSolution} in the main paper) procedure which encodes all the collected constraints, known representative values, and global constraints, and sends them to the solver. Algorithm \ref{alg:io-equiv-query} illustrates the generic equivalence query used by the teacher for probably approximately correct (PAC) identification. Note that in practice, for both the equivalence query and for obtaining classification accuracy, we collect all sampled sequences into a set first, then perform evaluation over the set. For the equivalence query, the first counterexample encountered is returned (and the remainder is untested). For classification accuracy, we evaluate all the sample sequences. Finally, in Algorithm \ref{alg:unification-a}, we present the unification algorithm used in \ouralgorithm{}. Overall, the algorithm has three sections: (1) creation and merging of equivalence classes and electing representatives, which converts all collected equality relations into equivalence classes (sets of variables which are equivalent), (2) performing unification on the collection of inequalities, performed via substituting each variable with its elected representative, and (3) replacing each variable in the symbolic observation table with the elected representative of the equivalence class of the variable.
\begin{algorithm*}[t]
\caption{\ouralgorithm{}}
\label{alg:full-pref-demo}
\textbf{Input}: Input alphabet $\Sigma^I$, output alphabet $\Sigma^O$, and a teacher $\mathcal{T}$\\
\textbf{Output}: Moore Machine $\mathcal{H} = \langle \hat{Q}, \Sigma^I, \Sigma^O, \hat{q}_0, \hat{\delta}, \hat{L}\rangle$ \\
\begin{algorithmic}[1]
\STATE Initialize observation table $(\langle S,E,T\rangle)$, constraint set $\mathcal{C}$, and context $\Gamma$ as  $\obstable$ with $S=\{\varepsilon\}, E=\{\varepsilon\}$, $\mathcal{C}=\{\}$, $\Gamma=\emptyset$
\STATE $\obstable \longleftarrow$ \textsc{SymbolicFill}$\left(\obstable| \mathcal{T}\right)$
\REPEAT
\WHILE{$\obstable$ is not closed and consistent}
\IF{$\obstable$ not consistent}
    \STATE Find $s_1, s_2\in S$, $\sigma \in \Sigma^I$, and $e\in E$\\s.t. $\orow{s_1} \equiv \orow{s_2}$ and $\orow{s_1\cdot\sigma}\not\equiv\orow{s_2\cdot\sigma}$ and $T(s_1\cdot \sigma \cdot e)\not\equiv T(s_2\cdot \sigma \cdot e)$
    \STATE Add $\sigma \cdot e$ to $E$ via $E := E \cup \{\sigma\cdot e\}$
\ENDIF
\IF{$\obstable$ not closed}
    \STATE Find $s\in S$ and $\sigma\in\Sigma^I$ s.t. $\orow{s\cdot\sigma}\in \orows{S\cdot\Sigma^I}$ and $\orow{s\cdot\sigma} \notin \orows{S}$
    \STATE Add $s\cdot \sigma$ to $S$ via $S := S \cup \{s\cdot \sigma\}$
\ENDIF
\STATE $\obstable \longleftarrow$ \textsc{SymbolicFill}$\left(\obstable| \mathcal{T}\right)$
\ENDWHILE
\STATE $\mathcal{H}\longleftarrow$\textsc{MakeHypothesis}$\left(\obstable, \Sigma^I, \Sigma^O\right)$
\STATE \textit{result} $\longleftarrow$ \textsc{EquivalenceQuery}$(\mathcal{H}| \mathcal{T})$
\IF{\textit{result} $\neq$ \textit{correct}}
    \STATE $(s', r)\longleftarrow$ \textit{result}
    \FORALL{$t\in\text{\textbf{\textit{prefixes}}}(s')$}
        \STATE Add $t$ to $S$ via $S:=S\cup\{t\}$
    \ENDFOR
    \STATE $\obstable \longleftarrow$ \textsc{SymbolicFill}$\left(\obstable| \mathcal{T}\right)$
    \STATE Update $\mathcal{C}$ to include constraint via $\mathcal{C} := \mathcal{C} \cup \{ \text{\textsc{GetVar}$(s', \obstable)$} = r\}$
\ENDIF
\UNTIL{\textit{result} $=$ \textit{correct}}
\RETURN $\mathcal{H}$
\end{algorithmic}
\end{algorithm*}
\begin{algorithm*}[t]
\caption{A Query Efficient Symbolic Fill Procedure}
\label{alg:efficient-symbolic-fill}
\textbf{procedure} \textsc{SymbolicFill}$\left(\obstable| \mathcal{T}\right)$
\begin{algorithmic}[1]
    \STATE \textit{seqs} $= \{\}$; Let $\mathcal{O}=\obstable$; let \textit{oldsortedseqs} be a sorted list of sequences.
    \FORALL{$s\cdot e \in (S\cup (S\cdot \Sigma^I))\cdot E$}
        \IF{$s\cdot e \notin \Gamma$}
            \STATE $\Gamma\left[s\cdot e\right], seqs\longleftarrow$ \textsc{FreshVar}(), $\text{\textit{seqs}}\cup \{s\cdot e\}$
            \STATE $T(s\cdot e)\longleftarrow$ $\Gamma\left[s\cdot e\right]$
        \ENDIF
    \ENDFOR
    \STATE \textit{sortseqs}, $\mathcal{O}$ $\longleftarrow$ \textsc{PrefQsViaRandQuicksort}(\textit{seqs}, $\mathcal{O}|\mathcal{T}$)
    \STATE \textit{oldsortedseqs},~$\mathcal{O}\longleftarrow$~\textsc{PrefQsViaLinearMerge}(\textit{sortedseqs, oldsortedseqs}, $\mathcal{O}|\mathcal{T}$)
    \STATE $\mathcal{O}\longleftarrow$ \textsc{Unification}$(\mathcal{O})$
    \RETURN $\mathcal{O}$
\end{algorithmic}
\end{algorithm*}

\begin{algorithm*}[t]
\caption{A Query Inefficient Symbolic Fill Procedure}
\label{alg:symbolic-fill}
\textbf{procedure} \textsc{SymbolicFill}$\left(\obstable| \mathcal{T}\right)$
\begin{algorithmic}[1]
    \STATE \textit{newentries, oldentries} $= \{\}, \{\}$
    \FORALL{$s\cdot e \in (S\cup (S\cdot \Sigma^I))\cdot E$}
        \IF{$s\cdot e \notin \Gamma$}
            \STATE $\Gamma\left[s\cdot e\right]\longleftarrow$ \textsc{FreshVar}()
            \STATE $T(s\cdot e) \longleftarrow \Gamma\left[s\cdot e\right]$
            \STATE $\text{\textit{newentries}}\longleftarrow\text{\textit{newentries}}\cup \{s\cdot e\}$
        \ELSE
            \STATE $\text{\textit{oldentries}}\longleftarrow\text{\textit{oldentries}}\cup \{s \cdot e\}$
        \ENDIF
    \ENDFOR
    \FORALL{$(s_1\cdot e_1), (s_2 \cdot e_2)\in$ \textsc{PairCombinations}(\textit{newentries}, \textit{newentries} $\cup$ \textit{oldentries})}
        \STATE $p\longleftarrow$ \textsc{PrefQuery}$(s_1\cdot e_1,s_2 \cdot e_2| \mathcal{T})$
        \STATE $\mathcal{C}\longleftarrow$ \textsc{UpdateConstraintSet}$(p, s_1\cdot e_1, s_2\cdot e_2; \obstable)$
    \ENDFOR
    \STATE $\obstable\longleftarrow$ \textsc{Unification}$(\obstable)$
    \RETURN $\obstable$
\end{algorithmic}
\end{algorithm*}
\begin{algorithm*}[t]
\caption{Preference Query Procedure and Constraints Update Procedure}
\label{alg:preference-query}
Preference queries are checked by the teacher executing each sequence using the ground truth transition function $\delta$ and checking the ground truth output value using $L$.\\\\
\textbf{procedure} \textsc{PrefQuery}$(s_1, s_2| \mathcal{T})$
\begin{algorithmic}[1]
\IF{$L(\delta(q_0,s_1)) = L(\delta(q_0,s_2))$}
    \RETURN $0$
\ELSIF{$L(\delta(q_0,s_1)) < L(\delta(q_0,s_2))$}
    \RETURN $-1$
\ELSE
    \RETURN $+1$
\ENDIF
\end{algorithmic}

\textbf{procedure} \textsc{UpdateConstraintSet}$(p, s_1 \cdot e_1, s_2 \cdot e_2; \obstable)$

\begin{algorithmic}[1]
\IF{$p=0$}
    \STATE $\mathcal{C}:=\mathcal{C}\cup \left\{T(s_1\cdot e_1) = T(s_2\cdot e_2)\right\}$
\ELSIF{$p=-1$}
    \STATE $\mathcal{C}:=\mathcal{C}\cup \left\{T(s_1\cdot e_1) < T(s_2\cdot e_2)\right\}$
\ELSE
    \STATE $\mathcal{C}:=\mathcal{C}\cup \left\{T(s_1\cdot e_1) > T(s_2\cdot e_2)\right\}$
\ENDIF
\RETURN $\mathcal{C}$
\end{algorithmic}
\end{algorithm*}
\begin{algorithm*}[t]
\caption{Make Hypothesis from Observation Table Procedure}
\label{alg:make-hypothesis}
\textbf{procedure} \textsc{MakeHypothesis}$\left(\obstable, \Sigma^I, \Sigma^O\right)$
\begin{algorithmic}[1]
\STATE $\hat{Q} = \{\orow{s} | \forall s\in S\}$ is the set of states
\STATE $\hat{q}_0 = \orow{\varepsilon}$ is the initial state
\STATE $\hat{\delta}(\orow{s},\sigma) = \orow{s\cdot \sigma}$ for all $s\in S$ and $\sigma \in \Sigma^I$
\STATE $\Lambda\longleftarrow$\textsc{FindSatisfyingSolution}$(\obstable, \Sigma^O)$
\STATE $\hat{L}(\orow{s}) = \Lambda[T(s\cdot\varepsilon)]$ is the sequence to output function
\RETURN $\langle \hat{Q}, \Sigma^I, \Sigma^O, \hat{q}_0, \hat{\delta}, \hat{L}\rangle$
\end{algorithmic}

\textbf{procedure} \textsc{FindSatisfyingSolution}$(\obstable, \Sigma^O)$

\begin{algorithmic}[1]
\STATE $\mathcal{C}_{val} :=$ \textsc{Select}$(\mathcal{C}, Var \rightarrow Value)$
    \STATE $D = \left\{\left.\displaystyle\bigvee_{r\in\Sigma^O} T(s\cdot e) = r \right| \forall s\cdot e \in (S \cup (S \cdot \Sigma^I))\cdot E \text{ such that }\mathcal{C}_{val}[T(s\cdot e)]\equiv \bot\right\}$
    \STATE Submit to an SMT solver the constraint set $\mathcal{C}\cup D$
    \RETURN The model $\Lambda$ which maps variables to values.
\end{algorithmic}
\end{algorithm*}
\begin{algorithm*}[t]
\caption{Sampling-based Equivalence Query}
\label{alg:io-equiv-query}
Equivalence is checked through sampling by the teacher. We assume a probability distribution $\mathcal{D}$ over all possible sequences $\left(\Sigma^I\right)^*$, and that the teacher can sample random sequences $s\sim\mathcal{D}$. 
\\\\
\textbf{procedure} \textsc{EquivalenceQuery}$\left(\langle \hat{Q}, \Sigma^I, \Sigma^O, \hat{q}_0, \hat{\delta}, \hat{L}\rangle | \mathcal{T}\right)$
\begin{algorithmic}[1]
\REPEAT
\STATE Sample a random sequence $s\sim\mathcal{D}$.
\IF{$L\left(\delta\left(q_0, s\right)\right) \neq \hat{L}\left(\hat{\delta}\left(\hat{q_0}, s\right)\right)$}
\RETURN $\left(s, L\left(\delta\left(q_0, s\right)\right)\right)$
\ENDIF
\UNTIL{up to $r$ times}
\RETURN \textit{correct}
\end{algorithmic}
\end{algorithm*}
\begin{algorithm*}[hbt!]
\caption{Unification Procedure}
\label{alg:unification-a}
Unification is performed by computing equivalence classes.
\\\\
\textbf{procedure} \textsc{Unification}$\left(\obstable\right)$
\begin{algorithmic}[1]
\STATE $\mathcal{C}_{ec} :=$ \textsc{Select}$(\mathcal{C}, Var \rightarrow EquivalenceClass)$
\STATE $\mathcal{C}_{val} :=$ \textsc{Select}$(\mathcal{C}, Var \rightarrow Value)$
\STATE $\mathcal{C}_{eq} :=$ \textsc{Select}$(\mathcal{C}, Var = Var)$
\STATE $\mathcal{C}_{in} :=$ \textsc{Select}$(\mathcal{C}, Var < Var) \cup \text{\textsc{Select}}(\mathcal{C}, Var > Var)$

\WHILE{$|\mathcal{C}_{eq}| > 0$}
    \STATE \textbf{match} \textsc{Pop}$(\mathcal{C}_{eq})$ \textbf{with} $L = R$
    \IF{$L\in \mathcal{C}_{ec} \land R\in \mathcal{C}_{ec}$}
        \IF{$\mathcal{C}_{ec}\left[L\right] \not\equiv \mathcal{C}_{ec}[R]$}
            \STATE $L_{rep}\longleftarrow$ \textsc{GetRepresentative}$(\mathcal{C}_{ec}\left[L\right])$
            \STATE $R_{rep}\longleftarrow$ \textsc{GetRepresentative}$(\mathcal{C}_{ec}\left[R\right])$
            \STATE Update $\mathcal{C}_{ec}\left[L\right] := \mathcal{C}_{ec}\left[L\right] \cup \mathcal{C}_{ec}\left[R\right]$
            \FORALL{$v \in \mathcal{C}_{ec}\left[L\right]$}
                \STATE Set $\mathcal{C}_{ec}\left[v\right]:= \mathcal{C}_{ec}\left[L\right]$
            \ENDFOR
            \IF{$L_{rep}\not\equiv R_{rep} \land R_{rep} \in \mathcal{C}_{val}$}
                \STATE Remove redundant representative via \textsc{Del}$(\mathcal{C}_{val}\left[ R_{rep}\right])$
            \ENDIF
        \ENDIF
    \ELSIF{$L\in \mathcal{C}_{ec} \land R\not\in \mathcal{C}_{ec}$}
        \STATE Update $\mathcal{C}_{ec}\left[L\right] := \mathcal{C}_{ec}\left[L\right] \cup \left\{R\right\}$ and set $\mathcal{C}_{ec}\left[R\right] := \mathcal{C}_{ec}\left[L\right]$
    \ELSIF{$L\not\in \mathcal{C}_{ec} \land R\in \mathcal{C}_{ec}$}
        \STATE Update $\mathcal{C}_{ec}\left[R\right] := \mathcal{C}_{ec}\left[R\right] \cup \left\{L\right\}$ and set $\mathcal{C}_{ec}\left[L\right] := \mathcal{C}_{ec}\left[R\right]$
    \ELSE
        \STATE Set $\mathcal{C}_{ec}\left[L\right] \longleftarrow$ \textsc{EquivalenceClass}$(\{L,R\})$ and $\mathcal{C}_{ec}\left[R\right] := \mathcal{C}_{ec}\left[L\right]$
        \STATE Set $\mathcal{C}_{val}\left[\text{\textsc{GetRepresentative}}\left(\mathcal{C}_{ec}\left[L\right]\right)\right] \longleftarrow \bot$
    \ENDIF
\ENDWHILE
\STATE \textit{subineqs} $\longleftarrow \{\}$
\WHILE{$|\mathcal{C}_{in}| > 0$}
    \STATE \textit{constraint} $\longleftarrow$ \textsc{Pop}$(\mathcal{C}_{in})$
    \FORALL{$v\in$ \textsc{GetVars}$\left(\text{\textit{constraint}}\right)$}
        \IF{$v\not\in \mathcal{C}_{ec}$}
            \STATE Set $\mathcal{C}_{ec}\left[v\right] \longleftarrow$ \textsc{EquivalenceClass}$(\{v\})$
            \STATE Set $\mathcal{C}_{val}\left[\text{\textsc{GetRepresentative}}\left(\mathcal{C}_{ec}\left[v\right]\right)\right] \longleftarrow \bot$
        \ENDIF
        \STATE Substitution via \textit{constraint} $:=$ \textit{constraint}$[\text{\textsc{GetRepresentative}}\left(\mathcal{C}_{ec}\left[v\right]\right)/v]$
    \ENDFOR
    \STATE \textit{subineqs} $:= $ \textit{subineqs} $\cup$ $\{\text{\textit{constraint}}\}$
\ENDWHILE
\STATE $\mathcal{C}_{in} := $ \textit{subineqs}
\FORALL{$s\cdot e \in (S\cup (S\cdot \Sigma^I))\cdot E$}
    \STATE Update $T[s\cdot e]$ via $T[s\cdot e] :=  \text{\textsc{GetRepresentative}}\left(\mathcal{C}_{ec}\left[T[s\cdot e]\right]\right)$
\ENDFOR
\FORALL{$s\cdot e \in \Gamma$}
    \STATE Update $\Gamma[s\cdot e]$ via $\Gamma[s\cdot e] := \text{\textsc{GetRepresentative}}\left(\mathcal{C}_{ec}\left[\Gamma[s\cdot e]\right]\right)$
\ENDFOR
\end{algorithmic}
\end{algorithm*}

\subsection{Technical Proofs}\label{proofs}
We first establish the correctness of constructing a hypothesis deterministic finite automaton from a symbolic observation table via Algorithm \ref{alg:make-hypothesis}, in Theorem 1 via Propositions 1 through 5. The proof strategies for Propositions 3 through 5 generally follow those of \citet{Angluin87}, with appropriate adjustments to account for a set of possible hypotheses, compared to just a single hypothesis.

\begin{definition} Let $h_k=\langle Q_k, q_{0,k}, \Sigma^I,\Sigma^O, \delta_k, L_k\rangle$ be a Moore machine for some integer $k$. Two Moore machines $h_1$ and $h_2$ are \textbf{structurally identical} if a bijection $B:Q_1\rightarrow Q_2$ exists, and the transition functions $\delta_1$ and $\delta_2$ are consistent with $B$. That is, $B(\delta_1(q_{0,1},s))=\delta_2(q_{0,2},s)$ for all $s\in(\Sigma^I)^*$ where $\delta_1$ and $\delta_2$ are the extended transition functions.
\end{definition}
\begin{proposition}
Let $\mathcal{H}$ be the set of all hypotheses that can be returned from \\\textsc{MakeHypothesis}$(\obstable)$. All hypotheses in $\mathcal{H}$ are structurally identical to each other, and differ only by their labeling function.
\end{proposition}
\begin{proof}
We first show that all hypotheses in $\mathcal{H}$ must be \textit{structurally identical} by construction, but have different $\hat{L}$ functions. Identical structure for all hypotheses in $\mathcal{H}$ can be shown by observing that lines 1-3 of Algorithm \ref{alg:make-hypothesis} are the same for each possible hypothesis in $\mathcal{H}$ for a given $\obstable$ input. Line 1 is the construction for the set of states. This procedure is identical for all hypotheses in $\mathcal{H}$, so all hypotheses in $\mathcal{H}$ have the same set of states. Bijection between sets of states for any pair of hypotheses is satisfied by the identity bijection function. Furthermore, the transition function construction specified on Line 3 indicates that all transitions functions for all hypotheses within $\mathcal{H}$ are identical. Since the transition functions are identical, and since bijection between states is satisfied by the identify function, we have that $B(\delta_k(q_{0,k},s))=\delta_k(q_{0,k},s)$ by identity bijection, and furthermore, $\delta_k(q_{0,k},s)=\delta_{k'}(q_{0,k'},s)$ by Line 3 construction, since $\delta_k$ and $\delta_{k'}$ are identical. This shows that all hypotheses in $\mathcal{H}$ are structurally identical.

Next, we show that each hypothesis in $\mathcal{H}$ has a unique labeling function. Note that in constructing the hypotheses, the only differences in the output hypothesis are due to lines 4 and 5. Lines 4 and 5 together select a satisfying solution for the free variables in $\obstable$, subject to the constraints $\mathcal{C}$. Therefore, each unique hypothesis in $\mathcal{H}$ is corresponds to a unique satisfying solution; the mapping between unique satisfying solutions and unique hypotheses is bijective.\qed
\end{proof}
\begin{proposition}
States in a Moore machine are represented by sets of sequences. Specifically, a mapping can always be constructed from sets of sequences to states, given an initial state.
\end{proposition}
\begin{proof}
We proceed by showing that a bijection always exists, simply by proposing a valid bijection for every case in an inductive argument. Suppose we have a finite alphabet $\Sigma$, and a Moore machine, with set of states $Q$, initial state $q_0\in Q$, and transition function $\delta: Q \times \Sigma \rightarrow Q$. We define the initial state $q_0$ of the Moore machine to correspond with the empty sequence $\varepsilon$ which has length $0$, so in fact we have $q_0 = \delta(q_0, \varepsilon)$, and for any $q\in Q$, $q = \delta(q, \varepsilon)$. The transition function can be extended recursively to $\delta: Q \times \left(\Sigma^*\right)\rightarrow Q$ by observing that $\delta(q_0, \sigma \cdot \omega) = \delta(\delta(q_0, \sigma), \omega)$, where $\sigma$ is a sequence of length of 1, and $\omega$ is a sequence of length of at least $1$. Similarly, $\delta(q_0, \omega \cdot \sigma) = \delta(\delta(q_0, \omega), \sigma)$. Based on this extended transition function, we can now consider the set $Q' = \{\delta(q_0, \sigma) | \sigma \in \Sigma\}$ to be a subset of $Q$.

First, we consider the case where $|Q'| = |\Sigma|$, where all transitions from $q_0$ have led to unique states. We can therefore construct a mapping: $M: 2^\Sigma \rightarrow Q'$. We know that $|2^\Sigma| = 2^{|\Sigma|} > |Q'|$ if $|Q'| = |\Sigma|\ge 1$, so we can consider a subset $K\subset 2^\Sigma$ such that $M$ is bijective. By constructing $K$, we can show that the elements of $K$ uniquely correspond to the elements of $Q'$ because $M: K \rightarrow Q'$ will be bijective. We note that if we construct $K = \{\{\sigma\} | \sigma \in \Sigma\}$, then clearly $|K| = |\Sigma| = |Q'|$, satisfying the current case. Furthermore, the mapping $M(\{\sigma\}) = \delta(q_0, \sigma) \forall \sigma \in \Sigma$ is clearly a bijective mapping between $K$ and $Q'$.

Next, we consider the case where $|Q'| < |\Sigma|$, which implies that there exists at least one pair $\sigma_i, \sigma_j \in \Sigma$ which lead to the same state (via the pidgeonhole principle). That is, $\delta(q_0, \sigma_i) = \delta(q_0, \sigma_j)$. If this is the case, then in order to make $M$ bijective, let $K=\{ k | k=\bigcup_{i,j} \{\sigma_i\}\cup\{\sigma_j\}\forall (\sigma_i,\sigma_j) \in \Sigma\times\Sigma \text{ such that } \delta(q_0, \sigma_i)=\delta(q_0,\sigma_j)\}$. The following mapping for $M$ is bijective: $M(k) = \delta(q_0, \sigma) \forall k \in K \text{ and } \forall \sigma \in k$, and that $K\subseteq 2^\Sigma$. Note each element of $K$ is a set of sequences of length 1.

Now, consider $Q'_n=\{\delta(q_0,\omega)|\omega\in (\Sigma)^n\}$ where $(\Sigma)^n$ denotes sequences of length at most $n$. Assume we can construct a bijective mapping $M:2^{(\Sigma)^n}\longrightarrow Q'_n$ for all $1\le n \le N$ for some fixed $N$. We have already shown this for $N=1$ above. We now proceed to show that we can also construct a bijective mapping $M:2^{(\Sigma)^{N+1}}\longrightarrow Q'_{N+1}$. Similar to the $N=1$ case, consider the case for when all sequences $\omega\in(\Sigma)^{N+1}$ lead to unique states---that is, when $|Q'_{N+1}|=|\Sigma|^{N+1}$. Then the bijective mapping in this case is $M: K\longrightarrow Q'_{N+1}$ where $K=\{\{\omega\} | \omega \in (\Sigma)^{N+1}\}$, and where $M(\{\omega\}) = \delta(q_0,\omega)$ for all $\omega \in (\Sigma)^{N+1}$, and $K\subset 2^{(\Sigma)^{N+1}}$.

Now, if $|Q'_{N+1}| < |\Sigma|^{N+1}$, then by the pidgeonhole principle, multiple $\omega$ must lead to the same state. Hence, let $K=\{ k | k=\bigcup_{i,j} \{\omega_i\}\cup\{\omega_j\}\forall (\omega_i,\omega_j) \in (\Sigma)^{N+1}\times(\Sigma)^{N+1} \text{ such that } \delta(q_0, \omega_i)=\delta(q_0,\omega_j)\}$. Then the bijective function is $M:K\longrightarrow Q'_{N+1}$, where $M(k)=\delta(q_0,\omega)$ for all $k\in K$ and for all $\omega\in k$, and note that $K\subseteq 2^{(\Sigma)^{N+1}}$.

Since we can construct bijective mappings from $M:2^{(\Sigma)^n}\longrightarrow Q'_n$ for all $1\le n \le N$, and we have also shown this is the case for $n=N+1$, this must now also hold for all $n=N+d$, for integer $d\ge 0$, and thus this holds as $n$ tends towards infinity. Thus, states in a Moore machine are represented by sets of sequences, where each state corresponds to a set in $K$. Each state is therefore synonymous with an equivalence class of sequences---a set of sequences that are equivalent according to the extended transition function.\qed
\end{proof}
\begin{proposition}
The rows of a unified, closed, and consistent $\obstable$ represent the states in a Moore machine consistent with the constraints $\mathcal{C}$. The hypothesis \textsc{MakeHypothesis}$(\obstable)$ satisfies $\delta(q_0, s) = \textbf{row}(s)$ for all $s\in(S\cup(S\cdot\Sigma))$.
\end{proposition}
\begin{proof}
We utilize the inductive proof of Lemma 2 from \citet{Angluin87}. For the case of $s=\varepsilon$, with length $0$, we have $\delta(q_0, \varepsilon) = q_0 = \textbf{row}(\varepsilon)$, which is true by definition. Now, let us assume that $\delta(q_0, s) = \textbf{row}(s)$ holds for all $s \in (S \cup (S \cdot \Sigma))$ with lengths no greater than $k$. Let $s'\in (S \cup (S \cdot \Sigma))$ be a sequence of length $k+1$ such that $s'=s\cdot \sigma$. If $s' \in S$, then $s\in S$ because $S$ is prefix closed. If $s'\in S\cdot \Sigma$, then $s\in S$. We can then show that $\delta(q_0, s') = \delta(q_0, s\cdot\sigma) = \delta(\delta(q_0, s), \sigma) = \delta(\textbf{row}(s),\sigma) = \textbf{row}(s\cdot\sigma) = \textbf{row}(s')$. This sequence of equalities follows as specified in Lemma 2 from \citet{Angluin87}.\qed
\end{proof}
\begin{proposition}The entries of a unified, closed, and consistent symbolic observation table correspond to sequence classification consistent with constraints $\mathcal{C}$. Specifically, for a unified, closed, and consistent $\obstable$, let $\Sigma^*$ be partitioned into at most $k\leq|\Sigma^O|$ disjoint sets $F_1, F_2,..., F_k$, and let $\Sigma^O_k$ be a specific subset of $k$ distinct elements of $\Sigma^O$. A sequence $s\in \Sigma^*$ is a member of $F_j$ if and only if it is classified as $\sigma_j\in \Sigma^O_k$. The hypothesis output by \textsc{MakeHypothesis}$(\obstable)$ satisfies $\delta(q_0, s\cdot e) \in F_j$ if and only if $T(s\cdot e) = \sigma_j$ for every $s\in (S\cup (S\cdot\Sigma)) \text{ and every } e\in E$.
\end{proposition}
\begin{proof}
We adapt the inductive proof of Lemma 3 from \citet{Angluin87}, but make appropriate adjustments for sequence classification. Let $s\in (S\cup (S\cdot \Sigma))$ and $e=\varepsilon$ be the base case. Clearly, $\delta(q_0, s\cdot \varepsilon) = \delta(q_0, s) = \textbf{row}(s)$ as shown previously. If $s\in S$, then $\textbf{row}(s)\in F_j$ if and only if $T(s)=\sigma_j$. If $s\in S\cdot \Sigma$, then $\textbf{row}(s) \in \hat{Q}$, since the symbolic observation table is closed. A $\textbf{row}(s') \in \hat{Q}$ is in $F_j$ if and only if $T(s')=\sigma_j$. 

Next, without loss of generality, assume that $\delta(q_0, s\cdot e) \in F_j$ if and only if $T(s\cdot e) = \sigma_j$ for all sequences $e$ with length at most $k$. Let $e'\in E$ be a sequence with length $k+1$ and let $s\in S\cup (S\cdot \Sigma)$. The sequence $e'=\sigma \cdot e_0$ for some $e_0\in E$ of length $k$ and some $\sigma\in\Sigma$, because $E$ is suffix-closed. Furthermore, there is a sequence $s_0\in S$ such that $\textbf{row}(s)=\textbf{row}(s_0)$ because the symbolic observation table is closed. Next, we show that if two sequences share a common suffix, but have different prefixes that end at the same state, then the two sequences also end at the same state. We observe that $\delta(q_0, s\cdot e') = \delta(q_0, s\cdot \sigma \cdot e_0) = \delta(\delta(q_0, s), \sigma\cdot e_0) = \delta(\textbf{row}(s), \sigma\cdot e_0) = \delta(\textbf{row}(s_0), \sigma\cdot e_0) = \delta(\delta(\textbf{row}(s_0), \sigma) e_0) = \delta(\textbf{row}(s_0\cdot\sigma), e_0) = \delta(\delta(q_0,s_0\cdot \sigma),  e_0) = \delta(q_0, s_0\cdot\sigma\cdot e_0)$. This indicates that the sequences $s\cdot e'$ and $s_0\cdot e'$ share a common suffix, but potentially have different prefixes, the two sequences end at the same state because their prefixes end at the same state. Since $e_0$ has length $k$, and $s_0\cdot \sigma \in S \cup (S\cdot \Sigma)$, we can use our initial assumption that $\delta(q_0, s_0\cdot\sigma\cdot e_0)\in F_j$ if and only if $T(s_0\cdot \sigma \cdot e_0)=\sigma_j$. Because $\textbf{row}(s)=\textbf{row}(s_0)$, this means that $\textbf{row}(s)(e)=\textbf{row}(s_0)(e) \forall e \in E$. Note that by definition, $T(s\cdot e') = \textbf{row}(s)(e')$, so $\textbf{row}(s_0)(\sigma\cdot e_0) = \textbf{row}(s)(\sigma\cdot e_0)$ implies $T(s_0\cdot \sigma\cdot e_0)=T(s\cdot \sigma\cdot e_0)=T(s\cdot e')$, which means $\delta(s\cdot e')\in F_j$ if and only if $T(s\cdot e')=\sigma_j$.\qed
\end{proof}
\begin{proposition}Suppose $\obstable$ is a unified, closed, and consistent symbolic observation table. If the hypothesis $\hat{h}=\langle \hat{Q}, \Sigma^I, \Sigma^O, \hat{q}_0, \hat{\delta}, \hat{L}\rangle$ from the function \textsc{MakeHypothesis}$(\obstable)$ has $n$ states, and $\hat{h}$ was generated using a satisfying solution $\Lambda$ to constraints $\mathcal{C}$, then any other Moore machine $h=\langle Q, \Sigma^I, \Sigma^O, q_0, \delta, L\rangle$ with $n$ or fewer states that is also consistent with $\Lambda$ and $T$ is isomorphic to $\hat{h}$.
\end{proposition}
\begin{proof}
We make additions and appropriate adjustments to the proof of Lemma 4 from \citet{Angluin87}. First, let us partition $\hat{Q}$ into $k=|\Sigma^O|$ disjoint sets $\hat{F}_1,...,\hat{F}_k$ such that for all $\sigma_j\in\Sigma^O$, $\hat{q}\in \hat{F}_j$ if and only if $\hat{L}(\hat{q}) = \sigma_j$. Note that it is possible for some of the $\hat{F}_j$ to be empty. Similarly, let us also partition $Q$ into $k$ disjoint sets $F_1,...,F_k$ such that for all $\sigma_j\in\Sigma^O$, $q\in F_j$ if and only if $L(q) = \sigma_j$. Next, we can define for all $q\in Q$ and for all $\sigma_j\in\Sigma^O$ the function $\textbf{row}^\wedge(q):E \rightarrow \Sigma^O$ such that $\textbf{row}^\wedge(q)(e) = \sigma_j$ if and only if $\delta(q, e) \in F_j$.

Because $h=\langle Q, \Sigma^I, \Sigma^O, q_0, \delta, L\rangle$ is consistent with $\Lambda$ and $T$, we know that for each $s\in S\cup (S\cdot \Sigma)$ and for each $e\in E$, $\delta(q_0, s\cdot e) \in F_j$ if and only if $T(s\cdot e)=\sigma_j$. Therefore, for all $e\in E$,  $\textbf{row}^\wedge(q)(e) = \textbf{row}(s)(e)$ if $q=\delta(q_0, s)$ implies $\textbf{row}^\wedge(\delta(q_0,s)) \equiv \textbf{row}(s)$. Next, because by definition $\hat{Q} = \{\textbf{row}(s) | \forall s\in S\}$, we have via substitution $\hat{Q} = \{\textbf{row}^\wedge(\delta(q_0,s)) | \forall s\in S\}$, which implies that $|Q|\geq n = |\hat{Q}|$. This is because in order for $\{\textbf{row}^\wedge(\delta(q_0,s)) | \forall s\in S\}$ to contain $n$ elements, $\delta(q_0,s)$ must range over at least $n$ states as $s$ ranges over $S$. However, because we have assumed that $h$ contains $n$ or fewer states, $h$ must contain exactly $n=|\hat{Q}|=|Q|$ states. 

Next, we can consider bijective mappings between $\hat{Q}$ and $Q$. Since $Q$ and $\hat{Q}$ have the same cardinality, and because $\textbf{row}^\wedge(\delta(q_0,s)) \equiv \textbf{row}(s)$, we know that for every $s\in S$ corresponds to a unique $q\in Q$, specifically, $q=\delta(q_0,s)$. We can define the bijective mapping $\textbf{row}^{-\wedge}: \hat{Q}\rightarrow Q$ via for all $s\in S$, $\textbf{row}^{-\wedge}(\textbf{row}(s)) = \delta(q_0, s)$. From this mapping, we observe that $\textbf{row}^{-\wedge}(\textbf{row}(\varepsilon)) = \textbf{row}^{-\wedge}(\hat{q}_0)=\delta(q_0,\varepsilon) = q_0$. Additionally, for all $s\in S$ and for all $\sigma\in \Sigma$, $\textbf{row}^{-\wedge}(\hat{\delta}(\textbf{row}(s),\sigma)) = \textbf{row}^{-\wedge}(\textbf{row}(s\cdot\sigma)) = \delta(q_0, s\cdot\sigma) = \delta(\delta(q_0, s), \sigma) = \delta(\textbf{row}^{-\wedge}(\textbf{row}(s)),\sigma)$, which implies for all $s\in S$ and $\sigma\in\Sigma$, $\textbf{row}^{-\wedge}(\hat{\delta}(\textbf{row}(s),\sigma)) = \delta(\textbf{row}^{-\wedge}(\textbf{row}(s)),\sigma)$.

Finally, we show that $\forall i \in \{k | k\in\mathbb{N}\And 1\leq k \leq |\Sigma^O|\}, \textbf{row}^{-\wedge}: \hat{F}_i \rightarrow F_i$. Specifically, if $s\in S$ has $\textbf{row}(s)\in \hat{F}_i$, this means that $\textbf{row}(s)(\varepsilon) = \sigma_i$, and therefore $T(s)=T(s\cdot\varepsilon) = \sigma_i$. Furthermore, suppose $\textbf{row}^{-\wedge}(\textbf{row}(s)) = \delta(q_0, s) = q$, and therefore $\textbf{row}^{\wedge}(q) = \textbf{row}^{\wedge}(\delta(q_0, s)) \equiv \textbf{row}(s)$, which implies that $\textbf{row}^{\wedge}(q)(\varepsilon) = \textbf{row}(s)(\varepsilon) = T(s\cdot \varepsilon) = T(s) = \sigma_i$. This means that $q\in F_i$ and $\textbf{row}(s)\in \hat{F}_i$, and that each element of $\hat{F}_i$ bijectively maps to an element of $F_i$.\qed
\end{proof}
\begin{proposition}
Suppose $\obstable$ is a unified, closed, and consistent symbolic observation table. Let $n$ be the number of states in the hypothesis returned from \textsc{MakeHypothesis}$(\obstable)$. Any Moore machine consistent with $\Gamma$ and $\mathcal{C}$ must have at least $n$ states.
\end{proposition}
\begin{proof}
Let $h=\langle Q, \Sigma^I, \Sigma^O, q_0, \delta, L\rangle$ be any Moore machine consistent with $T, \Lambda$, and $\mathcal{C}$, for the satisfying solution $\Lambda$ that was used to generate \textsc{MakeHypothesis}$(\obstable)$. Then, $\delta(q_0, s)\in F_i$ if and only if $T(s)=\sigma_i$. Let $s_1$ and $s_2$ be distinct elements of $S$ such that $\orow{s_1}\neq\orow{s_2}$. This means there exists an $e\in E$ such that $T(s_1\cdot e)\neq T(s_2\cdot e)$. This means that if $T(s_1\cdot e)=\sigma_i$, then $T(s_2\cdot e)\neq \sigma_i$. Because $h$ is consistent with $T$, this means that $\delta(q_0, s_1\cdot e)\in F_i$ and $\delta(q_0, s_2\cdot e)\not\in F_i$. This means that $\delta(q_0, s_1\cdot e)$ and $\delta(q_0, s_2\cdot e)$ must be distinct since they must belong to different classes. Because $\delta(q_0, s_1\cdot e)\neq\delta(q_0, s_2\cdot e)$, and because the transition function must be consistent, then $\delta(q_0, s_1)$ and $\delta(q_0,s_2)$ must be distinct states: if $\delta(q_0, s_1)$ and $\delta(q_0,s_2)$ were the same state, then it is impossible to transition to two different states using the same sequence $e$. Since \textsc{MakeHypothesis}$(\obstable)$ contains $n$ states, there is at least one subset $P\subseteq S$ with $P=\{s_1,...,s_n\}$ containing $n$ elements such that all the elements of $P'=\{\orow{s_1},...,\orow{s_n}\}$ are distinct from one another. Since every pair $\orow{s_i}\neq\orow{s_j}$ for $i\neq j$ in $P'$, then it follows that each element of $\{\delta(q_0, s_1),...,\delta(q_0,s_n)\}$ must be distinct to remain consistent with $T$. However, for a given $T$, there is a corresponding satisfying solution $\Lambda$ that specifies the values for $T$. This means that the mapping from $\Lambda$ to $T$ is one-to-one, and so for a given Moore machine to be consistent with a $(\Lambda, T)$ pair, the Moore machine must contain at least $n$ states. But since $\Lambda$ can be any satisfying solution to $\mathcal{C}$ and $\Gamma$, the above statement about Moore machines must be true for any $\Lambda$ that satisfies $\mathcal{C}$ and $\Gamma$. Thus, any Moore machine consistent with $\mathcal{C}$ and $\Gamma$ must have at least $n$ states.\qed
\end{proof}

\begin{theorem}If the symbolic observation table $\obstable$ is unified, closed, and consistent, and $\mathcal{H}$ is the set of all hypotheses that can be returned from \textsc{MakeHypothesis}$(\obstable)$, then every hypothesis $h \in \mathcal{H}$ is consistent with constraints $\mathcal{C}$. Any other Moore machine consistent with $\mathcal{C}$, but not contained in $\mathcal{H}$, must have more states.
\end{theorem} We first provide a sketch, followed by the proof.
\begin{proof} (Sketch) A given unified, closed, and consistent symbolic observation table $\obstable$ corresponds to $(\mathcal{S}, \mathcal{R}, \mathcal{C})$, where $\mathcal{S}$ is a symbolic hypothesis, $\mathcal{R}$ is the set of representatives used in the table, and $\mathcal{C}$ are the constraints expressed over $\mathcal{R}$. All hypotheses in $\mathcal{H}$ have states and transitions identical to $\mathcal{S}$. Each satisfying solution $\Lambda$ to $\mathcal{C}$ corresponds to a unique concrete hypothesis in $\mathcal{H}$. Therefore every concrete hypothesis in $\mathcal{H}$ is consistent with $\mathcal{C}$. Let $|h|$ represent the number of states in $h$. We know for all $h\in \mathcal{H}$, $|h|=|\mathcal{S}|$. Let $\overline{\mathcal{H}}$ be the set of concrete hypotheses \emph{not in} $\mathcal{H}$. Note $\overline{\mathcal{H}}$ can be partitioned into three sets---concrete hypotheses with (a) fewer states than $\mathcal{S}$, (b) more states than $\mathcal{S}$, and (c) same number of states as $\mathcal{S}$ but inconsistent with $\mathcal{C}$. We ignore (c) because we care only about hypotheses consistent with $\mathcal{C}$. Consider any concrete hypothesis $h$ in $\mathcal{H}$ and its corresponding satisfying solution $\Lambda$. Suppose we desire another hypothesis $h'$ to be consistent with $h$. If $|h'|<|h|$, then $h'$ cannot be consistent with $h$ because at least one sequence will be misclassified. Therefore, if $h'$ must be consistent with $h$, then we require $|h'|\geq |h|$. Thus, any other hypothesis consistent with $\mathcal{C}$, but not in $\mathcal{H}$, must have more states.\qed\end{proof}
\begin{proof} Proposition 1 establishes that all hypotheses in $\mathcal{H}$ have equivalent states and equivalent transitions, which implies that every hypothesis in $\mathcal{H}$ has the same number of states. Propositions 2 and 4 together establish that in a hypothesis, if a specific sequence $s$ is classified correctly, then all other sequences which start at the same initial state as $s$ and end in the same state as $s$ will be classified the same as $s$. Thus, if a state is classified correctly, then all sequences terminating at that state are classified correctly. Finally, Proposition 5 establishes that any Moore machine $d$ consistent with $\Lambda$ and $T$ from $\hat{h}=\textsc{MakeHypothesis}(\obstable)$ must be isomorphic to $\hat{h}$; otherwise $d$ must contain at least one more state than $\hat{h}$. This implies that $\hat{h}$ is the smallest Moore machine consistent with $\Lambda$ and $T$, and because $\Lambda$ is taken from all possible satisfying solutions to $\mathcal{C}$, then by Proposition 1, every hypothesis $h\in\mathcal{H}$ consistent with $\mathcal{C}$ has that same smallest number of states. Therefore, any other Moore machine not in $\mathcal{H}$, but still consistent with $\mathcal{C}$, must have more states.\qed
\end{proof}

Theorem 1 establishes that the output of Algorithm \ref{alg:make-hypothesis} will be consistent with all the constraints in $\mathcal{C}$. In order to establish correctness and termination, we will present a few additional lemmas, and then the main result.

\begin{proposition}
The number of unique constraints in $\mathcal{C}$ is finite, and is upper bounded by $\binom{|\Gamma|}{2}$.
\end{proposition}
\begin{proof}
We will show that the number unique constraints in $\mathcal{C}$ is finite. In fact, the upper bound on the total number of constraints in $\mathcal{C}$ is a function of $|\Gamma|$, the total number of unique sequences recorded in the observation table. The total number of preference queries that are executed is $\binom{|\Gamma|}{2}$, which is the maximum number of unique constraints in $\mathcal{C}$.

This is because whenever a \textsc{SymbolicFill} is performed at some time $t$, there will be $N$ new, unique sequences in the table which have never been used in a preference query, and there are $O$ old sequences, representing sequences which have already been queried. At all times, $|\Gamma|=N+O$, and therefore, during a \textsc{SymbolicFill}, $\binom{N}{2} + NO$ preference queries are performed. Now, consider $r$ to represent the round number---the $r$th time that a \textsc{SymbolicFill} has been performed. Let $n_r$ represent the number of new, unique sequences in round $r$. Let $o_{r+1}$ represent all sequences in $\Gamma$ which have been queried over the past $r$ rounds; that is, \[o_{r+1}=\sum_{k=0}^r n_k,\] where the initial conditions are $o_0=0$, and $n_0\ge 1$ is the initial size of $\Gamma$. Suppose that $R$ rounds have occurred. How many preference queries have occurred? Clearly, the total number of unique sequences after $R$ rounds is $n_0 + n_1 + \cdots + n_R = |\Gamma|$, which is the same as $o_R + n_R = |\Gamma|$. We can show that $$\binom{n_0+n_1+\cdots+n_R}{2} = \sum_{k=0}^R \binom{n_r}{2} + n_ro_r.$$
Clearly, the RHS is justified, because in round $r$, $\binom{n_r}{2} + n_ro_r$ preference queries are performed. Therefore, we just need to show the LHS is equivalent to the RHS via some algebra:
\begin{align*}
\binom{n_0+\cdots+n_R}{2} &= \binom{o_R+n_R}{2}\\
&=\frac{(o_R+n_R)(o_R+n_R-1)}{2}\\
&=\frac{(o_R+n_R)^2 - (o_R+n_R)}{2}\\
&=\frac{o_R^2 + n_R^2 + 2o_Rn_R-o_R-n_R}{2}\\
&=\frac{(o_R^2-o_R) + (n_R^2-n_R) + 2o_Rn_R}{2}\\
\binom{o_R+n_R}{2}&=\binom{o_R}{2} + \binom{n_R}{2} + o_Rn_R
\end{align*}
\begin{align*}
\binom{o_R+n_R}{2}&=\binom{o_{R-1}+n_{R-1}}{2} + \binom{n_R}{2} + o_Rn_R\\
&=\left[\binom{o_{R-2}+n_{R-2}}{2} + o_{R-1}n_{R-1}\right.\\&\phantom{{}=}\left.+ \binom{n_{R-1}}{2}\right] + \binom{n_R}{2} + o_Rn_R
\end{align*}
Now, note
\begin{align*}
\binom{o_{R-d} + n_{R-d}}{2} &= \binom{o_{R-d-1} + n_{R-d-1}}{2}\\&\phantom{{}=}+ \binom{n_{R-d}}{2} + o_{R-d}n_{R-d}
\end{align*}
for $0\le d\le R-1$ is an integer. Hence, by recursively splitting the above binomial term, we end up with the terms $o_0n_0 + \cdots + o_Rn_R$, $\binom{n_0}{2} + \cdots + \binom{n_R}{2}$, and $\binom{o_0}{2}$ in the expression, and in since $o_0=0$, we have:
\begin{align*}
\binom{o_R+n_R}{2}&=\binom{o_0}{2} + \left[\binom{n_0}{2} + \cdots + \binom{n_R}{2}\right]\\&\phantom{{}=}+ (o_0n_0 + \cdots + o_Rn_R)\\
&=\binom{o_0}{2} + \sum_{k=0}^R \binom{n_r}{2} + o_rn_r\\
&=\sum_{k=0}^R \binom{n_r}{2} + o_rn_r
\end{align*}

The number of unique sequences in the observation table is a function of the number of table closure and table consistency operations performed; but eventually, when the table becomes closed and consistent, a finite number of prefixes and suffixes will exist in the table. This implies there is only a finite number of unique possible sequences that can be created from the prefixes and suffixes in the table. The upper bound on the number of unique constraints is $\binom{|\Gamma|}{2}$.\qed
\end{proof}

\begin{proposition}
The number of representatives, $|\mathcal{R}|$, in a unified, closed, and consistent symbolic observation table $\obstable$ is finite, and is bounded by $1 \leq |\mathcal{R}| \leq |\Sigma^O|$. Also, the number of \textbf{unique inequalities} in $\mathcal{C}$ is $\binom{|\mathcal{R}|}{2}$.
\end{proposition}
\begin{proof}
Since the size of the output alphabet $\Sigma^O$ is finite, this means there is an upper bound of $|\Sigma^O|$ possible classification classes for all the sequences. If all sequences from $(S \cup (S\cdot \Sigma^I))\cdot E$ must be classified into $\Sigma^O$, then there are at most $|\Sigma^O|$ equivalence classes of sequences.

Because each entry of a unified, closed, and consistent symbolic observation table $\obstable$ is a variable, and that variable must have as its value one of the elements in $\Sigma^O$, the number of representatives $|\mathcal{R}|$ is bounded via $1 \leq |\mathcal{R}| \leq |\Sigma^O|$. We define unique representatives to be representatives that are known have distinct values; in other words, the number of unique representatives at any given in time is the same as the number of equivalence classes over variables at that point in time; the unique representatives are the representatives of the equivalence classes of variables. Thus, if there are $|\mathcal{R}|$ unique representatives, and they satisfy a total ordering, then $\mathcal{C}$ will eventually contain a total of $\binom{|\mathcal{R}|}{2}$ unique inequalities. This total number of unique inequalities is populated in $\mathcal{C}$ via preference queries over pairs of unique sequences in the table. The quantity $|\Gamma|$ of unique sequences in the table is increased by prefix expansions (via closure tests and counterexamples), and suffix expansions (via consistency tests). Furthermore, during an equivalence query, whenever a counterexample $c$ is presented, $c$ and all its prefixes are added to the prefix set.\qed
\end{proof}

We now define properties of a satisfying solution $\Lambda$ to the constraint set $C$ involving $|\mathcal{R}|$ representatives, and how many of those representatives have correct assignments. If $|\mathcal{R}| < V^*$ for some upper bound $V^*\le |\Sigma^O|$, then we say that $\Lambda$ is \textit{\textbf{incomplete}}. If $\Lambda$ contains at least 1 incorrect representative-value assignment, then we say $\Lambda$ is \textit{\textbf{incorrect}}. If $\Lambda$ has assigned correct values to all $|\mathcal{R}|$ representatives, then we say $\Lambda$ is \textit{\textbf{correct}}. If $\Lambda$ is both incomplete and correct, then we say $\Lambda$ is \textit{\textbf{partially correct}}.

\begin{lemma}
Whenever a counterexample $c$ is processed, either $0$ or $1$ additional representative values becomes known.
\end{lemma}
\begin{proof}
    This is proven in the proof for Theorem 2.\qed
\end{proof}
\begin{theorem}
Suppose $\obstable$ is a unified, closed, and consistent symbolic observation table. Let $\hat{h}=\textsc{MakeHypothesis}(\obstable)$ be the hypothesis induced by $\Lambda$, a satisfying solution to $\mathcal{C}$. If the teacher returns a counterexample $c$ as the result of an equivalence query on $\hat{h}$, then at least one of the following statements about hypothesis $\hat{h}$ must be true: (a) $\hat{h}$ contains too few states, or (b) the satisfying solution $\Lambda$ inducing $\hat{h}$ is either incomplete or incorrect.
\end{theorem}
We first present a sketch, then a proof.
\begin{proof} (Sketch) This sketch applies to the above lemma, theorem, and below corollary about termination. Consider the sequence $\dots,h_{k-1}, h_k,\dots$ of hypotheses that \ouralgorithm{} makes. For a given pair of consecutive hypotheses $(h_{k-1}, h_k)$, consider how the number of states $n$, and the number of \emph{known} representative values $n_\bullet$ changes. Let $n^*$ be the number of states of the minimal Moore machine correctly classifying all sequences. Let $V^*\leq |\Sigma^O|$ be the upper bound on $|\mathcal{R}|$. Note that $0\leq n_\bullet \leq |\mathcal{R}| \leq V^* \leq |\Sigma^O|$ always holds. Through detailed case analysis on returned counterexamples, we can show that the change in $n_\bullet$, denoted by $\Delta n_\bullet$, must always be either $0$ or $1$, and furthermore, if $\Delta n_\bullet = 0$, then we must have $\Delta n \geq 1$. By the case analysis and tracking $n$ and $n_\bullet$, observe that if a counterexample $c$ is received from the teacher due to hypothesis $h$, then \emph{at least one of} (a) $n < n^*$ or (b) $n_\bullet < V^*$ must be true. Since $\Delta n_\bullet$ and $\Delta n$ cannot simultaneously be $0$, whenever a new hypothesis is made, progress must be made towards the upper bound of $(n^*, V^*)$. If the upper bound is reached, then the algorithm must terminate, since it is impossible to progress from the point $(n^*, V^*)$.\qed
\end{proof}
\begin{proof}
Our strategy for this proof is to show that the number of known representative values and number of states in the hypothesis will increase up to their upper limits, at which point the algorithm must terminate. Let $\mathbb{EC}[v]$ denote the set of variables in $\Gamma$ which are known to be equivalent to $v$ according to the equality constraints that were gathered so far in $\mathcal{C}$. Let $|\mathcal{R}|$ denote the number of representatives present in the unified, closed, and consistent observation table $\obstable$. That is, $|\mathcal{R}|$ is the number of elements in the set $\{\mathbb{EC}[\Gamma(s\cdot e)] | \forall s\cdot e \in (S\cup(S\cdot \Sigma))\cdot E\}$. Let $n_{\circ}$ represent the number of unknown representative values, and let $n_{\bullet}$ represent the number of known representative values. We have, at all times, the invariant $|\mathcal{R}|=n_{\circ}+n_{\bullet}$. Furthermore, at the very beginning of the algorithm, they have the initial values of $|\mathcal{R}|=1$, $n_{\circ}=|\mathcal{R}|$ and $n_{\bullet}=0$. Finally, let $V^*$ represent the ground truth number of classes that sequences can be classified into, according to the teacher. We will show that if the teacher returns a counterexample, then at least one of the following two scenarios must be true about the learner's current hypothesis $\hat{h}=\textsc{MakeHypothesis}(\obstable)$: (\textbf{A}) $\hat{h}$ contains too few states, or (\textbf{B}) the satisfying solution $\Lambda$ inducing $\hat{h}$ is incorrect. We enumerate cases by describing the pre-conditions and post-conditions when executing the equivalence query.
\begin{case} Prior to the equivalence query, $c\in (S \cup (S\cdot\Sigma^I))\cdot E$. This means there exists  $s_1 \in (S\cup (S\cdot \Sigma^I))$ and $e_1 \in E$ such that $c=s_1\cdot e_1$ and the variable located at $T(c)$ was assigned an incorrect value $w$. This also means that for every $s\cdot e \in (S\cup(S\cdot\Sigma))\cdot E$ such that $\Gamma[s\cdot e]\in\mathbb{EC}[T(c)]$, it follows that each $T(s\cdot e)$ is also the same incorrect $w$. However, after the equivalence query, the teacher returns the counterexample $c$ with feedback $f(c)=r$, where $r\neq~w$. Then, for each variable $v\in \mathbb{EC}[T(c)]$, the learner will assign $v$ the value $r$. If $c\in S$, then all of its prefixes are already in $S$ because $S$ is prefixed closed. If $c\in S\cdot \Sigma$, then all of its prefixes are already in $S$. If $c$ is not in $S\cdot \Sigma$, then the prefixes of $c$ which are not in $S$ will be added to $S$. In this latter case, there is a possibility that the number of unique representatives $|\mathcal{R}|$ will be increased by $k$, where $0\leq k \leq V^*-|\mathcal{R}|$. However, the number of known variables will always increase by exactly 1. Thus, under this case, we have the following transformation: $(|\mathcal{R}|, n_{\circ}, n_{\bullet}) \longrightarrow (|\mathcal{R}| + k, n_{\circ} - 1 + k, n_{\bullet} + 1)$, where $0\leq k \leq V^* - |\mathcal{R}|$.
\end{case}
\begin{case}
Prior to the equivalence query, $c\not\in (S \cup (S\cdot\Sigma^I))\cdot E$, which means $T(c)$ does not exist in the symbolic observation table. If $c$ was returned as a counterexample, then the learner incorrectly hypothesized that $\hat{\delta}(\hat{q}_0, c) \in \hat{F}_i$. The post-conditions for this case can be broken into three possibilities:
\begin{enumerate}[(a)]
    \item There is some $s\cdot e \in (S\cup (S\cdot \Sigma^I))\cdot E$ for which $\mathbb{EC}\left[\Gamma[s\cdot e]\right]\overset{\mathrm{val}}{=}r$ according to $\Lambda$, and $\left|\mathbb{EC}[\Gamma[s\cdot e]]\right| \geq 1$. The new variable $T(c)$ is added to $\mathbb{EC}[\Gamma[s\cdot e]]$, and subsequently, $\mathbb{EC}[T(c)]=\mathbb{EC}[\Gamma[s\cdot e]]$ and $|\mathbb{EC}[T(c)]| > 1$, and all the variables in that class have a value $r$. While $c$ itself does not induce $|\mathcal{R}|$ to increase, it is possible that its prefixes might induce an increase in $|\mathcal{R}|$. So while the number of unknown representative values might increase, the number of known representative values stays constant. Therefore, we have the following transformation: $(|\mathcal{R}|, n_{\circ}, n_{\bullet}) \longrightarrow (|\mathcal{R}| + k, n_{\circ} + k, n_{\bullet})$, where $0 \leq k \leq V^* - |\mathcal{R}|$.
    \item There is some $s\cdot e \in (S\cup (S\cdot \Sigma^I))\cdot E$ for which $\mathbb{EC}\left[\Gamma[s\cdot e]\right]\overset{\mathrm{val}}{\neq}r$ according to $\Lambda$ and $\left|\mathbb{EC}[\Gamma[s\cdot e]]\right| \geq 1$. The new variable $T(c)$ is added to $\mathbb{EC}[\Gamma[s\cdot e]]$, and subsequently, $\mathbb{EC}[T(c)]=\mathbb{EC}[\Gamma[s\cdot e]]$ and $|\mathbb{EC}[T(c)]| > 1$, and $\mathbb{EC}\left[T(c)\right]\overset{\mathrm{val}}{=}\mathbb{EC}\left[\Gamma[s\cdot e]\right]\overset{\mathrm{val}}{=}r$ . Again, the number of unknown representative values might increase, but this time, the number of known representative values increases by 1. Therefore, we have the following transformation: $(|\mathcal{R}|, n_{\circ}, n_{\bullet}) \longrightarrow (|\mathcal{R}| + k, n_{\circ} - 1 + k, n_{\bullet}+1)$, where $0 \leq k \leq V^* - |\mathcal{R}|$.
    \item The variable $T(c)$ is not added to any pre-existing equivalence class; therefore it is added to its own equivalence class, and its value is known to be $r$. Thus, we have the following transformation: $(|\mathcal{R}|, n_{\circ}, n_{\bullet}) \longrightarrow (|\mathcal{R}| + 1 + k, n_{\circ} + k, n_{\bullet}+1)$, where $0\leq k \leq V^* - |\mathcal{R}| - 1$.
\end{enumerate}
\end{case}
We note that cases (1), (2b), and (2c) can occur a finite number of times. This is because $n_{\bullet}$ increases by exactly $1$ each time one of those cases occurs. The maximum value $n_{\bullet}$ can take on is $|\mathcal{R}|$; the upper bound on $|\mathcal{R}|$ is $V^*$, and the upper bound on $V^*$ is $|\Sigma^O|$. Furthermore, cases (1), (2b), and (2c) cannot occur if $n_{\bullet} = V^*$. Thus, if case (1), (2b), or (2c) occur, they increase the number of known variables by 1, and they will cease to occur once $n_{\bullet} = V^*$. These cases fall under the umbrella of an incorrect satisfying solution $\Lambda$.

Clearly, case (1) falls under Scenario (\textbf{B}), an incorrect satisfying solution $\Lambda$, since the value of $T(c)$ for a known $c$ in the table is incorrect. Case (2b) also falls under Scenario (\textbf{B}), because $\hat{\delta}(\hat{q}_0,c) = \hat{\delta}(\hat{q}_0, s')$ for some $s'\in S$ was incorrectly classified, implying that the value of $T(s')$ was incorrect.

For cases (2a) and (2c), it is not necessarily true that the satisfying solution $\Lambda$ used to generate $\hat{h}$ was incorrect, with respect to the current number of representatives $|\mathcal{R}|$. It is possible that $\Lambda$ was indeed incorrect, where at least one of the $|\mathcal{R}|$ representatives was assigned the incorrect value, in which case Scenario (\textbf{B}) is true. It is also possible that $\Lambda$ was incomplete, but partially correct, where all existing known representatives had correct values, but where $|\mathcal{R}| < V^*$. If a hypothesis $\hat{h}$ is partially correct, then $\mathcal{C}$ contains only $|\mathcal{R}|$ representatives, with $|\mathcal{R}| < V^*$, $n_{\bullet}=|\mathcal{R}|$, and $n_{\circ}=0$. This means that $\hat{h}=\textsc{MakeHypothesis}(\obstable)$ is the \textit{only} partially correct hypothesis the algorithm could have generated under these conditions (any others must be incorrect), and is therefore unique; the size of $\mathcal{H}$ must be 1. Then, by Theorem 1, any other Moore machine consistent with $\mathcal{C}$, but not equivalent to $\hat{h}$ must contain more states; hence Scenario (\textbf{A}) is true. Thus, in our case work, we have shown that Scenario (A) or Scenario (B) holds.\qed
\end{proof}

\begin{corollary}[Termination] \ouralgorithm{} must terminate when the number of states and number of known representative values in a concrete hypothesis reach their respective upper bounds.
\end{corollary}
\begin{proof} Since cases (1-2c) cover all the cases, and in each case at least one of the scenarios, Scenario (\textbf{A}) and Scenario (\textbf{B}), is true, then for each pair of consecutive hypotheses $\hat{h}_{j-1}$ and $\hat{h}_{j}$ generated by consecutive equivalence queries in the algorithm, one of the following is true: (a) the latter hypothesis $\hat{h}_{j}$ contains at least one more state than the prior hypothesis $\hat{h}_{j-1}$, (b) the latter hypothesis $\hat{h}_{j}$ contains at least one more known variable value compared to the prior hypothesis $\hat{h}_{j-1}$, or (c) both (a) and (b) are true. Thus, in a sequence of hypotheses generated by equivalence queries, both the number of states and number of known representatives increase monotonically. There can be at most $V^*$ discoveries. If the $n^*$ is the number of states in the minimum Moore machine which correctly classifies all sequences, then the number of states $n$ in each hypothesis will increase monotonically to $n^*$. Then, clearly, the algorithm must terminate when $n=n^*$ and $|\mathcal{R}|=V^*$.\qed
\end{proof}

Corollary 1 indicates that \ouralgorithm{} must terminate, since every hypothesis made makes progress towards the upper bound.
\begin{theorem}[Query Complexity] If $n$ is the number of states of the minimal automaton isomorphic to the target automaton to be learned, and $m$ is the maximum length of any counterexample sequence that the teacher returns, then (a) the upper bound on the number of equivalence queries that \ouralgorithm{} executes is $n+|\Sigma^O|-1$, and (b) the preference query complexity is $\mathcal{O}(mn^2 \ln (mn^2))$, which is polynomial in the number of unique sequences that the learner performs queries on.
\end{theorem}
\begin{proof} Based on Theorem 2, we know that the maximum number of equivalence queries is the taxi distance from the point $(1,0)$ to $(n, |\Sigma^O|)$, which is $n+|\Sigma^O|-1$. From counterexample processing, we know there will be at most $m(n+|\Sigma^O|-1)$ sequences added to the prefix set $S$, since a counterexample $c$ of length $m$ results in at most $m$ sequences added to the prefix set $S$. The maximum number of times the table can be found inconsistent is at most $n-1$ times, since there can be at most $n$ states, and the learner starts with $1$ state. Whenever a sequence is added to the suffix set $E$, the maximum length of sequences in $E$ increases by at most $1$, implying the maximum sequence length in $E$ is $n-1$. Similarly, closure operations can be performed at most $n-1$ times, so the total number of sequences in $E$ is at most $n$; the maximum number of sequences in $S$ is $n+m(n+|\Sigma^O|-1)$. The maximum number of unique sequences queried in the table is the maximum cardinality of $(S\cup S\cdot \Sigma^I)\cdot E$, which is $$(n+m(n+|\Sigma^O|-1))(1+|\Sigma^I|)n = \mathcal{O}(mn^2).$$
%$$|(S\cup S\cdot \Sigma^I)\cdot E|=(n+m(n+|\Sigma^O|-1))(1+|\Sigma^I|)n = O(mn^2)$$
Therefore, the preference query complexity of \ouralgorithm{} is $\mathcal{O}(mn^2\ln(mn^2))$ due to randomized quicksort.\qed
\end{proof}
Next, we show that \ouralgorithm{} probably approximately correctly identifies the minimal automaton isomorphic to the target.
\begin{definition}[Probably Approximately Correct Identification] \ Given an arbitrary Moore machine $M=\langle Q, q_0, \Sigma^I, \Sigma^O, \delta, L\rangle$, let the regular language classification function $f:(\Sigma^I)^*\rightarrow\Sigma^O$ be represented by $f(s)=L(\delta(q_0,s))$ for all $s\in(\Sigma^I)^*$. Let $\mathcal{D}$ be an any probability distribution over $(\Sigma^I)^*$. An algorithm $\mathcal{A}$ probably approximately correctly identifies $f$ if and only if for any choice of $0< \epsilon < 1$ and $0 < d < 1$, $\mathcal{A}$ always terminates and outputs an $\epsilon$-approximate sequence classifier $\hat{f}:(\Sigma^I)^*\rightarrow\Sigma^O$, such that with probability at least $1-d$, the probability of misclassification is $P(\hat{f}(s)\neq f(s)) \le \epsilon$ when $s$ is drawn according to the distribution $\mathcal{D}$.
\end{definition}
\begin{theorem} \ouralgorithm{} achieves probably approximately correct identification of any Moore machine when the teacher $\mathcal{T}$ uses sampling-based equivalence queries with at least $m_k \geq \left\lceil\frac{1}{\epsilon}\left(\ln\frac{1}{d}+k\ln 2\right)\right\rceil$ samples drawn i.i.d. from $\mathcal{D}$ for the $k$th equivalence query.
\end{theorem}
We first present a proof sketch, then the full proof.
\begin{proof} (Sketch) The probability $1-\epsilon_k$ of a sequence sampled from an arbitrary distribution $\mathcal{D}$ over $(\Sigma^I)^*$ depends on the distribution and the intersections of sets of sequences of the teacher and the learner's $k$th hypothesis with the same classification values. The probability that the $k$th hypothesis misclassifies a sequence is $\epsilon_k$. If the teacher samples $m_k$ samples for the $k$th equivalence query, then an upper bound can be established for the case when $\epsilon_k\leq \epsilon$ for a given $\epsilon$. Since we know \ouralgorithm{} executes at most $n+|\Sigma^O|-1$ equivalence queries, one can upper bound the probability that \ouralgorithm{} terminates with an error by summing all probabilities of events that the teacher does not detect an error in at most $n+|\Sigma^O|-1$ equivalence queries. An exponential decaying upper bound can be found, and a lower bound for $m_k$ can be found in terms of $\epsilon,d,$ and $k$.\qed
\end{proof}
\begin{proof}
Suppose the teacher $\mathcal{T}$ has Moore machine $\langle Q, q_0, \Sigma^I, \Sigma^O, \delta, L\rangle$ representing $f$ which classifies all sequences in the set $(\Sigma^I)^*$ into $|\Sigma^O|$ disjoint sets $\{L_c|c\in\Sigma^O\}$ such that for all $c\in\Sigma^O, s \in L_c, f(s)=L(\delta(q_0,s))=c$, we have:
\begin{align*}
(\Sigma^I)^* = \bigcup_{c\in\Sigma^O} L_c.
\end{align*}
Similarly, suppose the learner has proposed a Moore machine $\langle \hat{Q}, \hat{q}_0, \Sigma^I, \Sigma^O, \hat{\delta}, \hat{L}\rangle$ representing the $k$th hypothesis $\hat{f}_k$ which classifies all sequences in the set $(\Sigma^I)^*$ into $|C^k_L|$ disjoint subsets $\{L'^k_c|c\in C^k_L\}$ where $C^k_L\subseteq \Sigma^O$ such that for all $c\in C^k_L$ and for all $s \in L'^k_c, \hat{f}(s)=\hat{L}(\hat{\delta}(\hat{q}_0,s))=c$ we have:
\begin{align*}
(\Sigma^I)^* = \bigcup_{c\in C^k_L} L'^k_c
\end{align*} 
If $\mathcal{D}$ is a distribution over $(\Sigma^I)^*$, and if $S\sim\mathcal{D}$ is a random variable representing the sequence $s$ drawn according to $\mathcal{D}$, and if the set of intersections $I_k$ is defined by $$I_k=\displaystyle\bigcup_{c\in C^k_L\cap \Sigma^O} L_c \cap L'^k_c$$
then the probability that the $k$th hypothesis $\hat{f}_k$ classifies $s$ correctly is 
$$P(\hat{f}_k(S)=f(S))=\displaystyle\sum_{s \in I_k} p_{S}(s)$$ and therefore the probability $\epsilon_k$ that $\hat{f}_k(s)\neq f(s)$ is 
$$\epsilon_k = P(\hat{f}_k(S)\neq f(S))=1-\displaystyle\sum_{s \in I_{k}} p_{S}(s).$$
Suppose for the $k$th equivalence query, the teacher $\mathcal{T}$ samples $m_k$ sequences i.i.d. according to the distribution $\mathcal{D}$ over $(\Sigma^I)^*$.
The probability $\mathcal{T}$ accepts $\hat{f}_k$ because it detects no misclassification for any of the $m_k$ samples (represented by the random variables $S_1,\cdots, S_{m_k}\sim\mathcal{D}$) is $p_k$, given by
$$p_k=P(\sum_{i=1}^{m_k}(\hat{f}_k(S_i)-f(S_i))=0)=(1-\epsilon_k)^{m_k},$$ so if $\epsilon_k\leq \epsilon$ for a chosen $\epsilon$, then 
$$P(\sum_{i=1}^{m_k}(\hat{f}_k(S_i)-f(S_i))=0|\epsilon_k\leq\epsilon)\geq(1-\epsilon)^{m_k}$$ and if $\epsilon_k\geq\epsilon$, then 
$$P(\sum_{i=1}^{m_k}(\hat{f}_k(S_i)-f(S_i))=0|\epsilon_k\geq\epsilon)\leq(1-\epsilon)^{m_k}.$$
We know that \ouralgorithm{} will execute at most $n+|\Sigma^O|-1$ equivalence queries, so the probability that \ouralgorithm{} terminates with an error in the set 
$\{\epsilon_1,\cdots,\epsilon_{n+|\Sigma^O|-1}\}$
is given by
\begin{align*}
p_1+(1-p_1)p_2+\cdots+p_{n+|\Sigma^O|-1}\prod_{k=1}^{n+|\Sigma^O|-2}(1-p_k)\\
=\sum_{k=1}^{n+|\Sigma^O|-1}p_k\prod_{i=0}^{k-1}(1-p_i)
\end{align*}
 where $p_0=0$. Each term in the summation on the right hand side is less than or equal to $p_k$, so
$$\sum_{k=1}^{n+|\Sigma^O|-1}p_k\prod_{i=0}^{k-1}(1-p_i)\leq\sum_{k=1}^{n+|\Sigma^O|-1}p_k.$$
Furthermore, we know that $\epsilon_k$ monotonically decreases as $k$ increases because $|C_{L}^{k}\cap \Sigma^O|$ monotonically increases. Considering the progress transformation from Theorem 2, if $|C_L^k\cap \Sigma^O|$ monotonically increases, then this means the cardinality of the set $\{{L'}_{c}^{k}|\forall c\in C_{L}^{k}\}$ monotonically increases. If any sequence in any of the intersections $L_{c}\cap {L'}_{c}^{k}$ is misclassified, then the feedback from the counterexample will move the misclassified sequence to the correct class, which means the error decreases monotonically. Specifically, by Theorem 2, the following events can occur: (1) the value of an existing state (or a set of existing equivalently valued states) will obtain their correct values; (2) at least 1 new states will be created, either increasing the size of one of the ${L'}_{c}^{k}$; or (3) increasing the cardinality of $|C_{L}^{k}\cap \Sigma^O|$ and creating an additional ${L'}_{c}^{k}$ set (corresponds with at least one new states and one new value). Other non-counterexample sequences which were previously misclassified either stay misclassified (perhaps with a different, but still incorrect value), or they will become classified correctly. A sequence which was already known to be classified correctly will not become misclassified. Thus, all three possibilities monotonically decrease the error. This means $\epsilon_j\geq \epsilon_{j+1}$, so the terminal hypothesis will have an error of at most the error of the previous hypothesis. If we choose a desired error $\epsilon$, then the probability that the error of the terminal hypothesis is greater than $\epsilon$ is at most
$$\sum_{k=1}^{n+|\Sigma^O|-1}p_k \leq \sum_{k=1}^{n+|\Sigma^O|-1}(1-\epsilon)^{m_k}\leq \sum_{k=1}^{n+|\Sigma^O|-1}e^{-\epsilon m_k}$$
since $$P(\sum_{i=1}^{m_k}(\hat{f}_k(S_i)-f(S_i))=0|\epsilon_k\geq\epsilon)\leq(1-\epsilon)^{m_k},$$ and since $1+x\leq e^x$ for all real $x$. If we take each term in the rightmost summation and require it to be at most $\frac{d}{2^k}$ for some chosen value of $0<d<1$, then we have a lower bound on the number of samples for the $k$th equivalence query for $k=1,\cdots,(n+|\Sigma^O|-1)$
$$m_k \geq \frac{1}{\epsilon}(\ln\frac{1}{d}+k\ln 2)$$
and
$$\sum_{k=1}^{n+|\Sigma^O|-1}e^{-\epsilon m_k} \leq \sum_{k=1}^{n+|\Sigma^O|-1}\frac{d}{2^k}\leq \sum_{k=1}^{\infty}\frac{d}{2^k} \leq d$$
which implies shows that the probability that \ouralgorithm{} terminates with a hypothesis with error at least $\epsilon$ is at most $d$.\qed
\end{proof}
\begin{theorem} To achieve PAC-identification under \ouralgorithm{}, if the teacher $\mathcal{T}$ has chosen parameters $\epsilon$ and $d$, and if $f$ can be represented by a minimal Moore machine with $n$ states and $|\Sigma^O|$ classes, then $\mathcal{T}$ needs to sample at least $$\mathcal{O}(n+|\Sigma^O| + \frac{1}{\epsilon}((n+|\Sigma^O|)\ln\frac{1}{d} + (n+|\Sigma^O|)^2))$$ sequences i.i.d. from $\mathcal{D}$ over the entire run of \ouralgorithm{}.
\end{theorem}
\begin{proof}
Since for the $k$th equivalence query, the teacher must sample at least $m_k\geq \left\lceil \frac{1}{\epsilon}(\ln\frac{1}{d} + k\ln2)\right\rceil$ sequences in order to achieve PAC-identification, if the total number of samples is to be minimized while still achieving PAC-identification, then the teacher can just sample a total of 
\begin{align*}&\sum_{k=1}^{n+|\Sigma^O|-1}\left[\frac{1}{\epsilon}(\ln\frac{1}{d} + k\ln2) + 1\right]\\=&n+|\Sigma^O|-1\\&+\frac{1}{\epsilon}\left[(\ln\frac{1}{d})(n+|\Sigma^O|-1) + \ln2\sum_{k=1}^{n+|\Sigma^O|-1}k\right]\\
=&\mathcal{O}\left((n+|\Sigma^O|)+\frac{1}{\epsilon}((n+|\Sigma^O|)\ln\frac{1}{d} + (n+|\Sigma^O|)^2)\right)
\end{align*}\qed
\end{proof}

\subsection{Miscellaneous Experiments}
In addition to learning reward machines, we also evaluated \ouralgorithm{} on sequence classification problems where the classification model can be represented as a Moore machine. Since reward machines are Mealy machines, and Moore machines and Mealy machines are equivalent, each of these Moore machines can be converted into a reward machine. We specified the reference Moore machines using regular expressions. The regular expressions we used to construct the Moore machines were: $a^*b$ represents sequences of zero or more $a$'s followed by a single $b$; $b^*a$ represents sequences of zero or more $b$'s followed by a single $a$; $(a|b)^*$ represents sequences with zero or more elements, where each element is either $a$ or $b$; $(ab)^*$ represents sequences of zero or of even lengths, where every $a$ is immediately followed by a $b$, and non-zero length sequences start with an $a$; $(ba^*)$ represents sequences of zero or of even lengths, where every $b$ is immediately followed by a $a$, and non-zero length sequences start with a $b$.

These components were OR'ed together. If $M_i$ represents a regular expression component, and $$M= \displaystyle\left|_{i=1}^N M_i\right. = M_1 | M_2 |\cdots | M_N$$ represents $N$ distinct regular expression components OR'ed together, then the classifier $f$ classifies a sequence $s$ in the following manner: $f(s)=k$ if $s\in M_k$, and is $0$ otherwise. If $s$ happens to be contained in multiple $M_{k_1}$,$M_{k_2},...$, then the lowest matching $k$ value prevails. This structure implies the size of the output alphabet is $1+N$.

We measured accuracy, number of preference queries, number of equivalence queries, and number of unique sequences as functions of samples per equivalence query (Figure \ref{fig:violin_plots}) and Figure \ref{fig:comparison_plots}. We also measured these attributes as functions of alphabet size (Figure \ref{fig:comparison_alpha_size_plots}). Finally, we also created some termination phase diagrams (Figure \ref{fig:termination_plots}).

\begin{figure*}
    \centering\scriptsize
    \begin{tblr}{colspec = {X[c]X[c,h]X[c,h]X[c,h]X[c,h]},
  stretch = 0,
  rowsep = 6pt,}
     \textbf{Regex} & \textbf{Accuracy} & \textbf{Num of EQ} & \textbf{Num of PQ} & \textbf{Num of Seq} \\
    $(a^*b)$ & \includegraphics[width=0.19\textwidth]{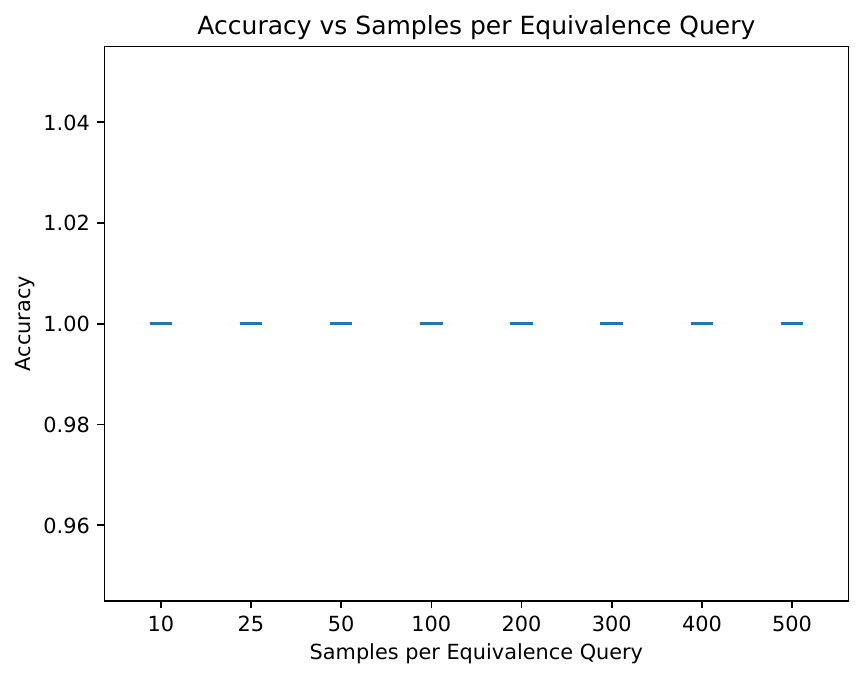} &
    \includegraphics[width=0.19\textwidth]{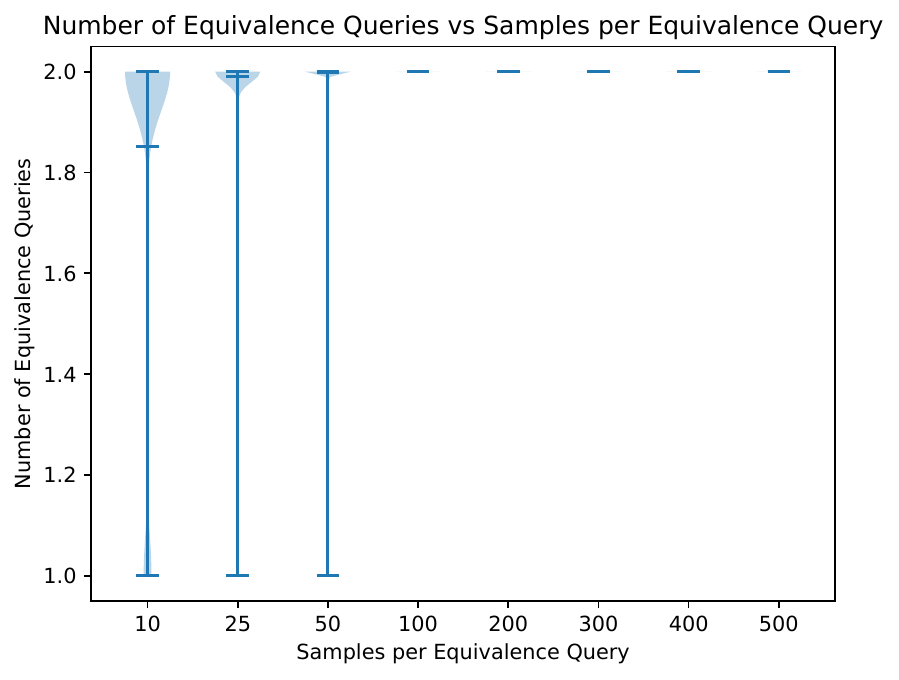} &
    \includegraphics[width=0.19\textwidth]{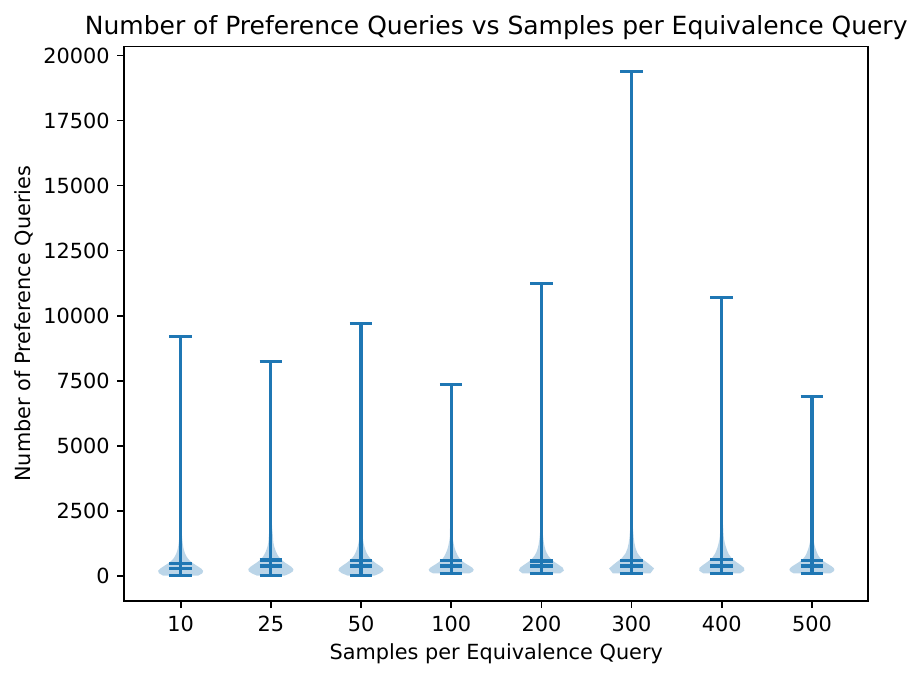} &
    \includegraphics[width=0.19\textwidth]{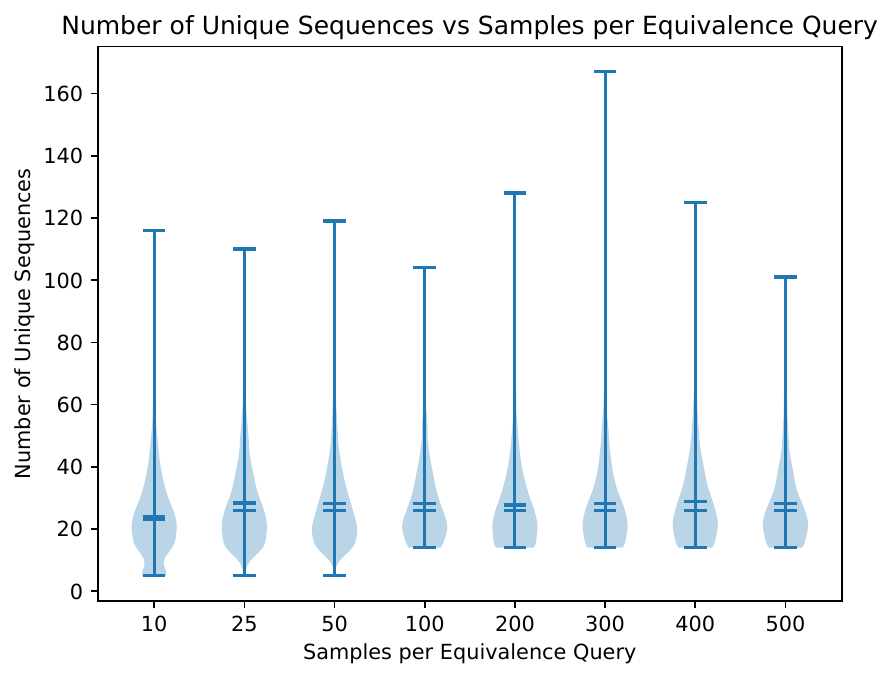}\\
    $(a^*b) | (b^*a)$ & 
    \includegraphics[width=0.19\textwidth]{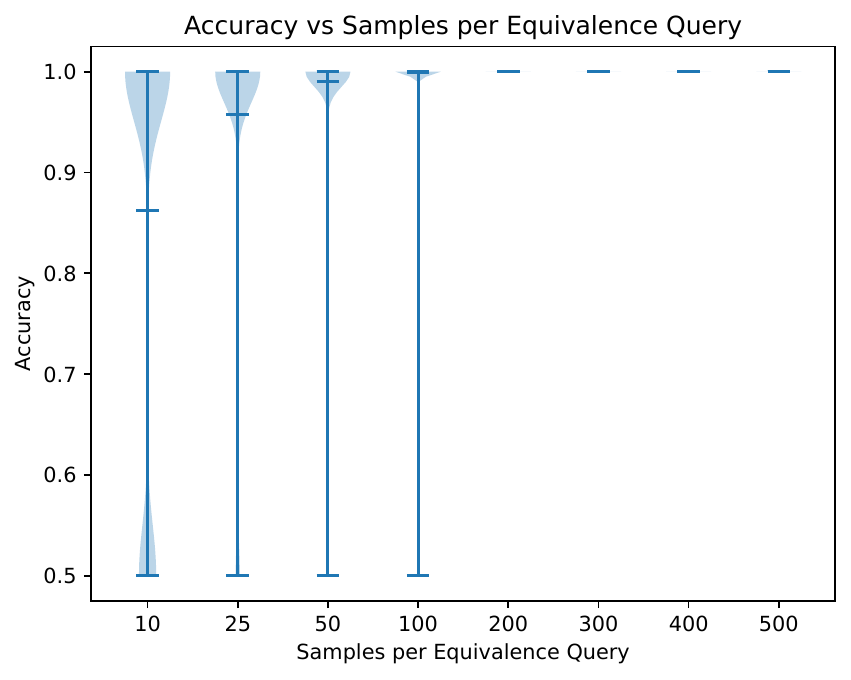} &
    \includegraphics[width=0.19\textwidth]{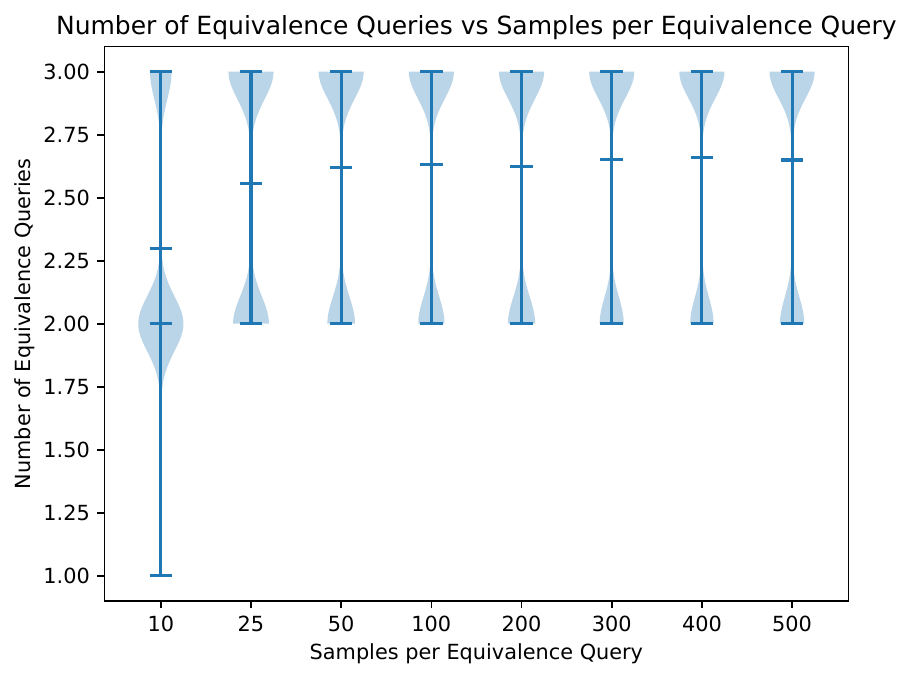} &
    \includegraphics[width=0.19\textwidth]{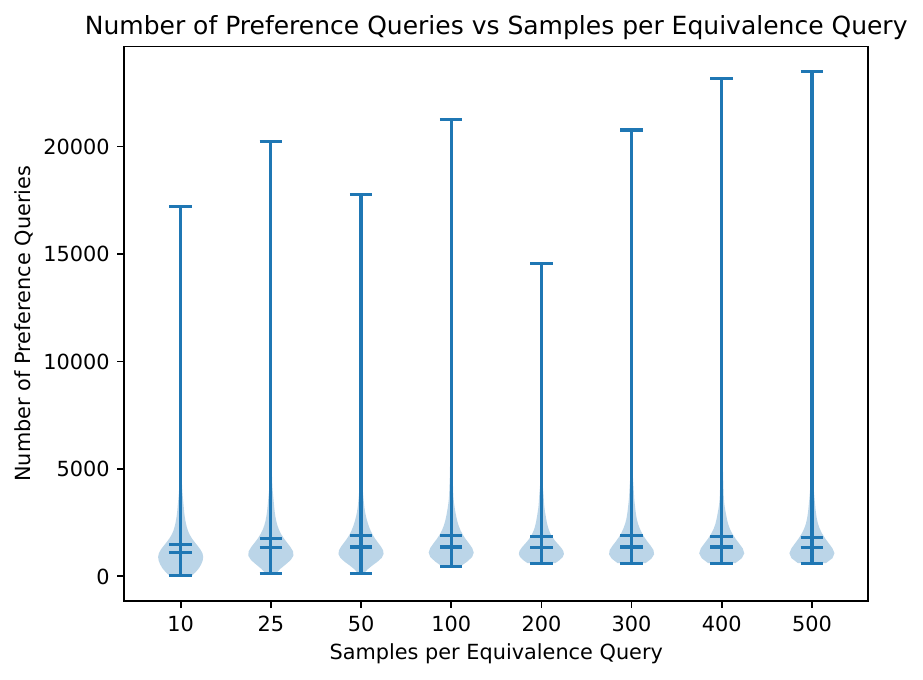} &
    \includegraphics[width=0.19\textwidth]{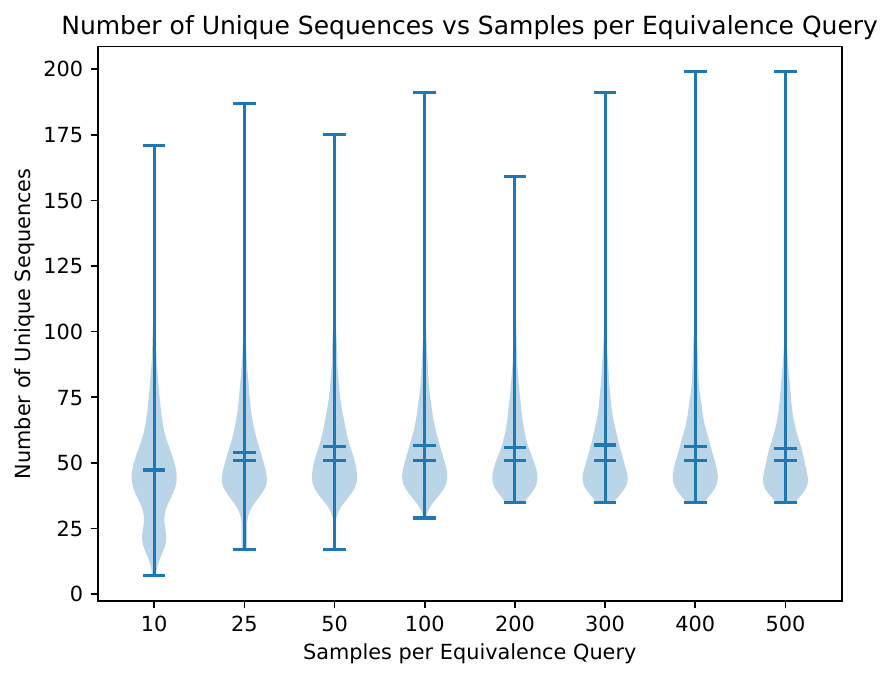}\\
    $(a^*b) | (b^*a) | (a|b)^*$ & 
    \includegraphics[width=0.19\textwidth]{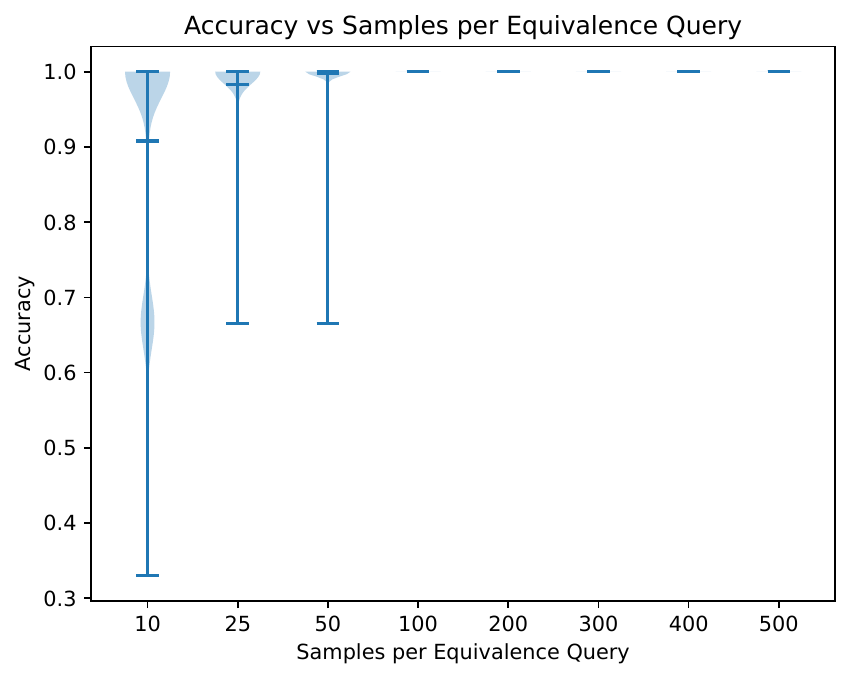} &
    \includegraphics[width=0.19\textwidth]{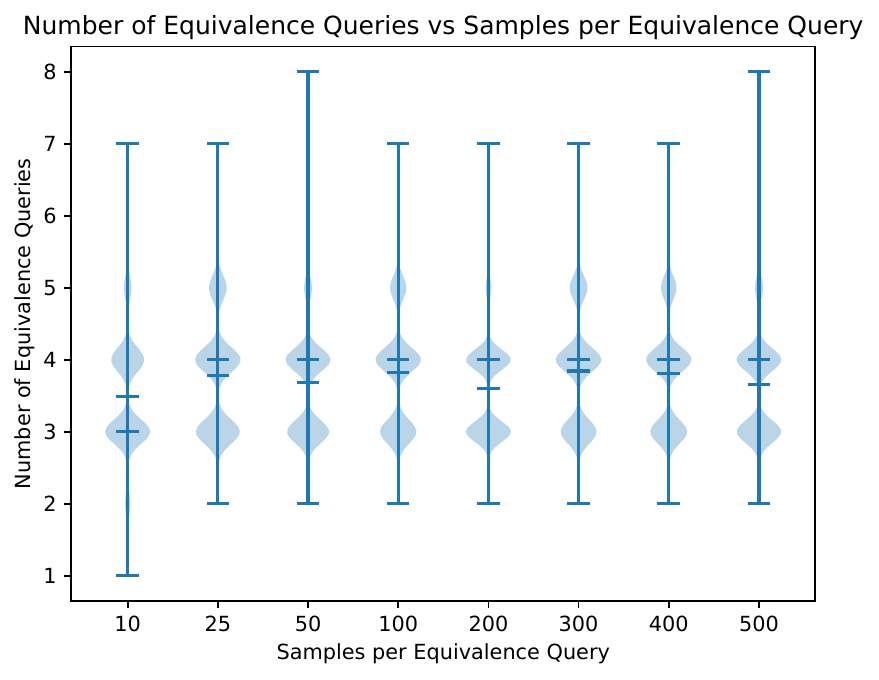} &
    \includegraphics[width=0.19\textwidth]{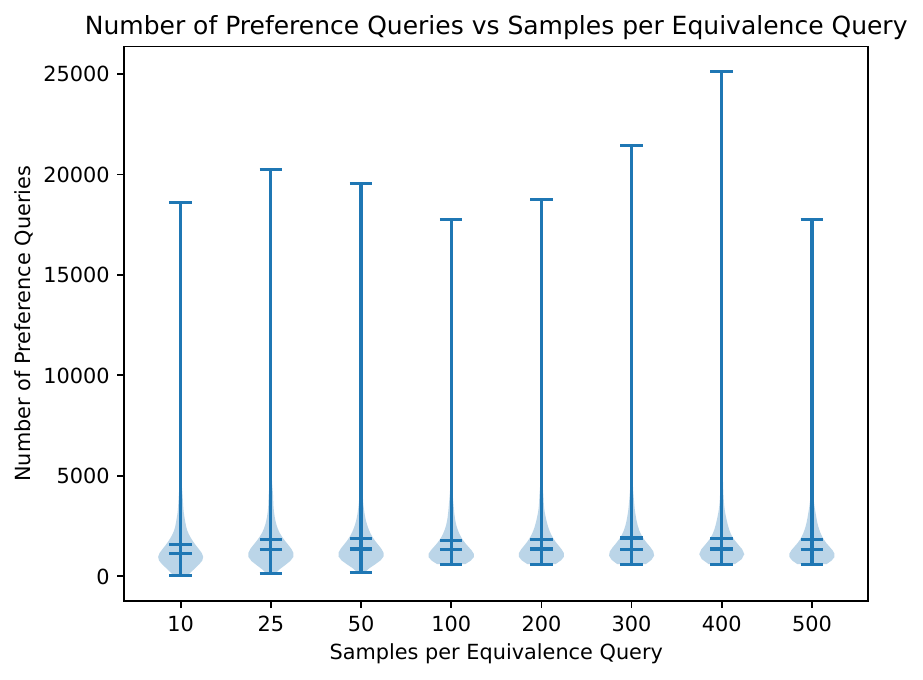} &
    \includegraphics[width=0.19\textwidth]{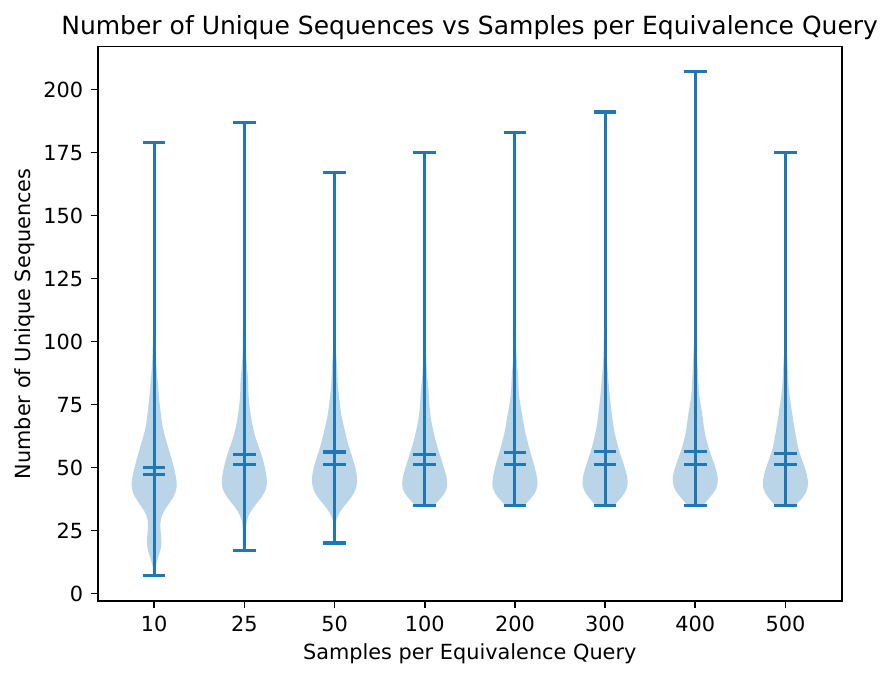}\\
    $(a^*b) | (b^*a) | (ab)^*$ & 
    \includegraphics[width=0.19\textwidth]{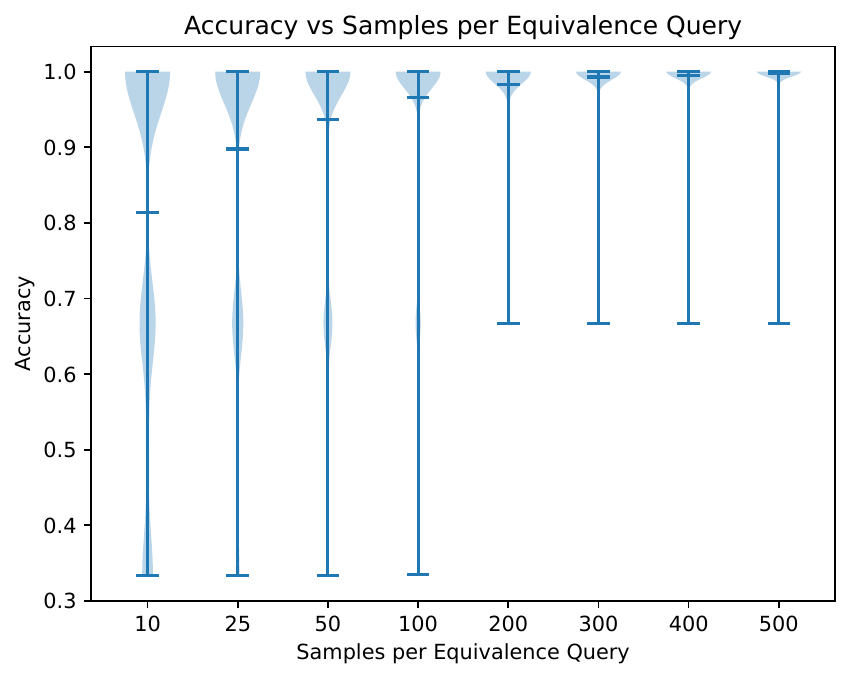} &
    \includegraphics[width=0.19\textwidth]{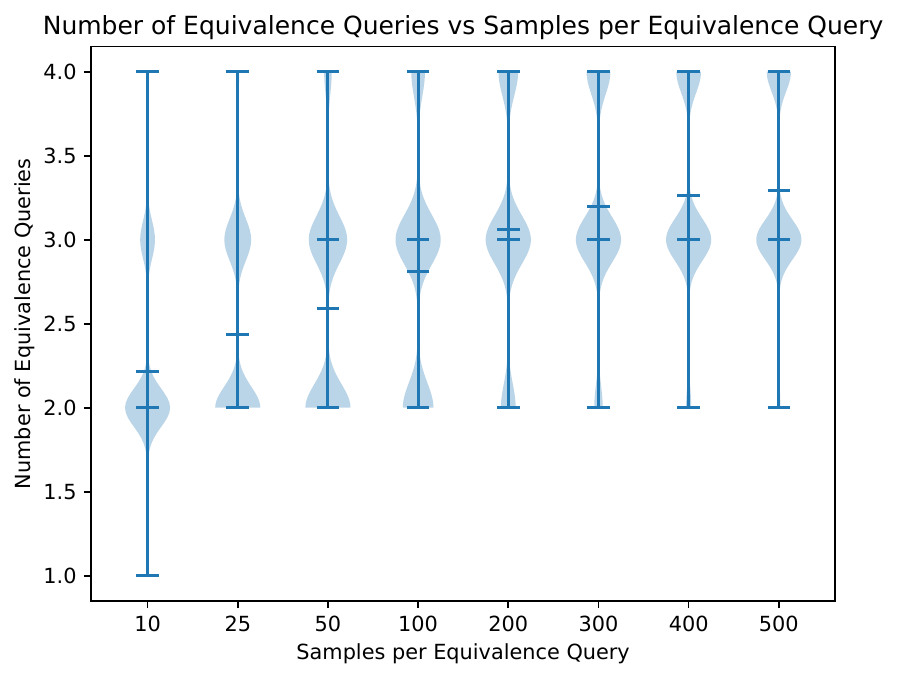} &
    \includegraphics[width=0.19\textwidth]{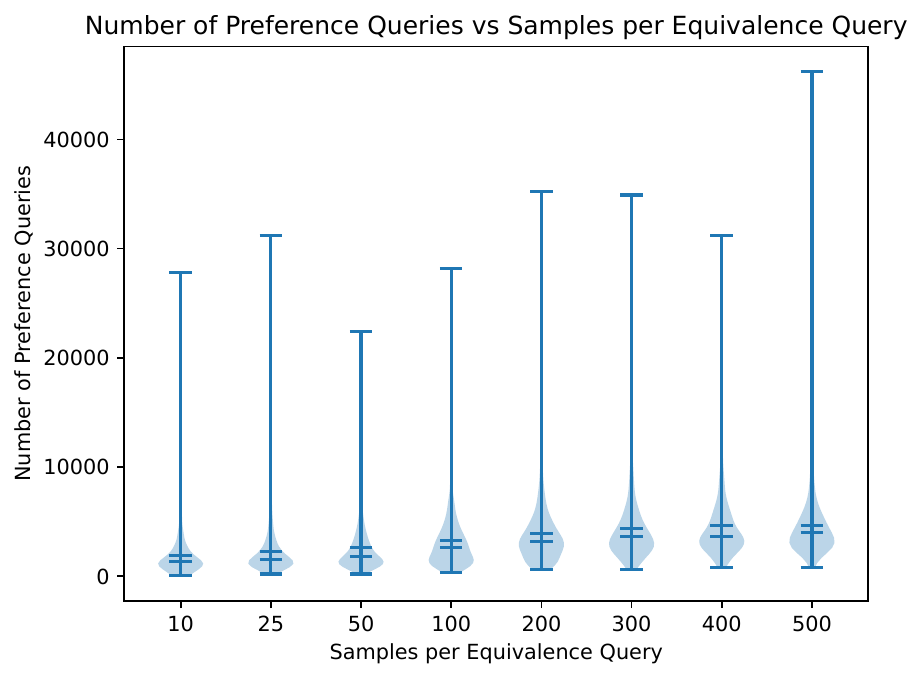} &
    \includegraphics[width=0.19\textwidth]{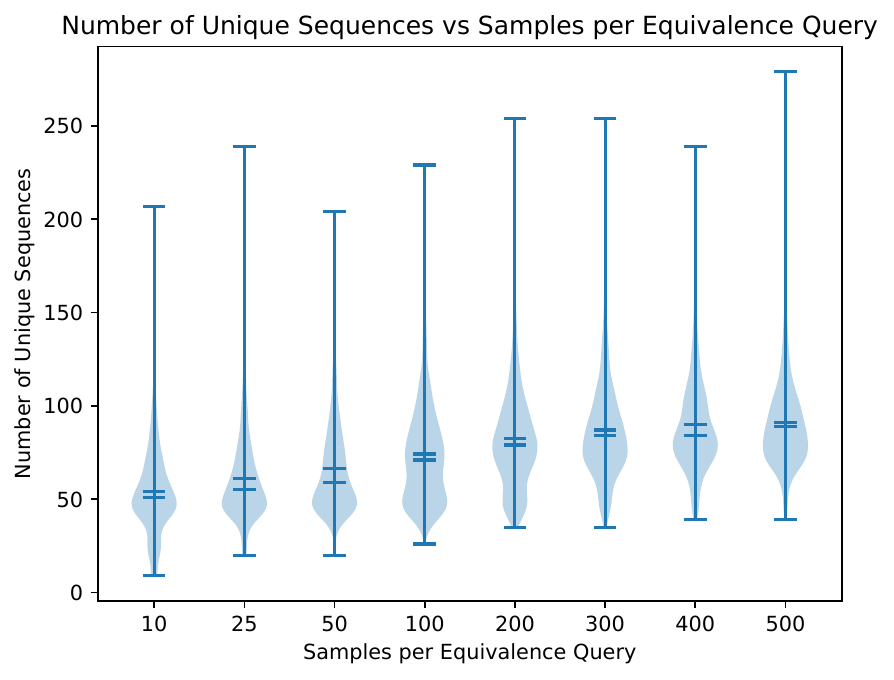}\\
    $(a^*b) | (b^*a) | (ab)^* | (ba)^*$ &
    \includegraphics[width=0.19\textwidth]{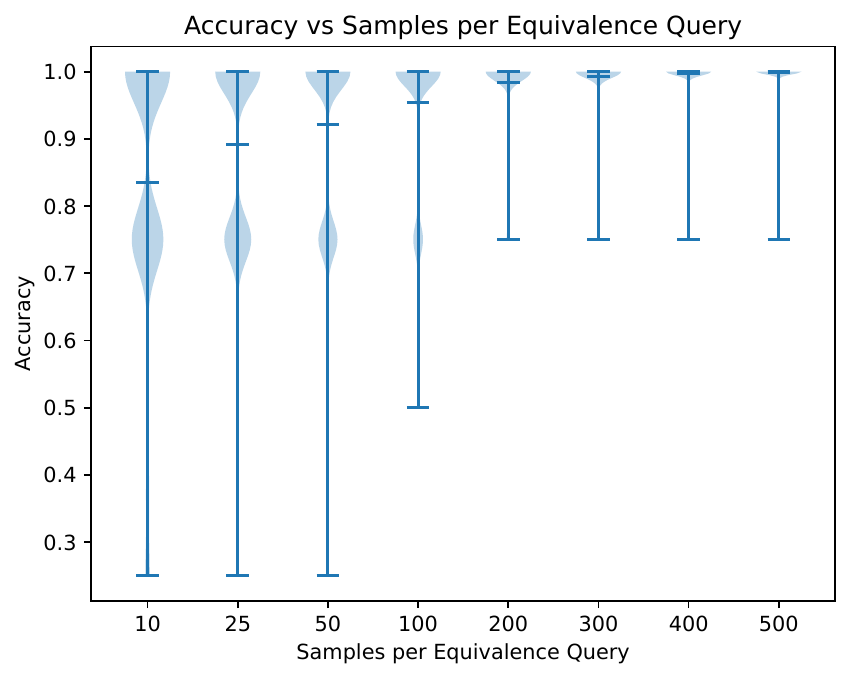} &
    \includegraphics[width=0.19\textwidth]{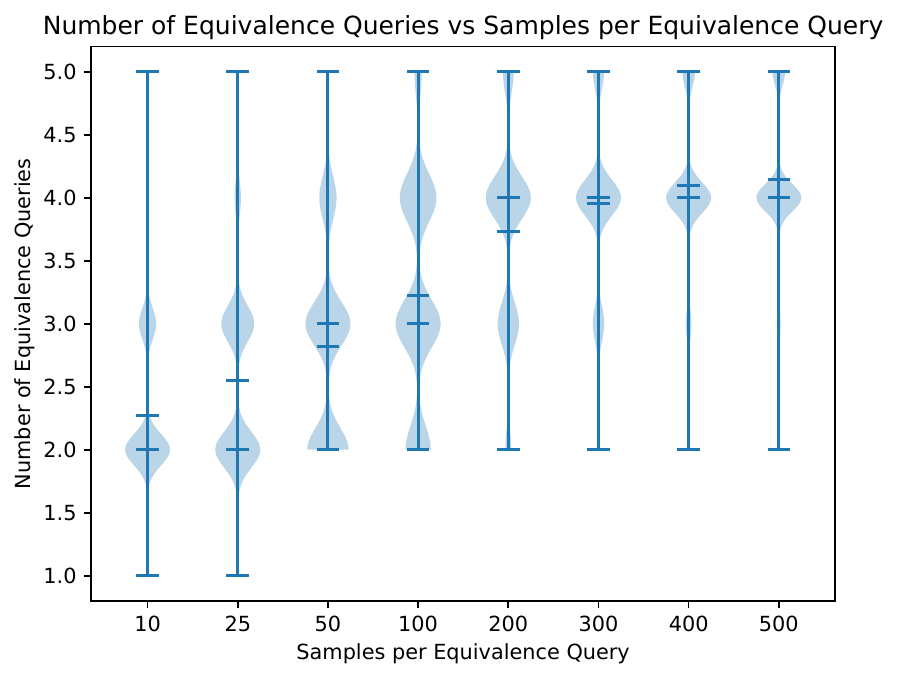} &
    \includegraphics[width=0.19\textwidth]{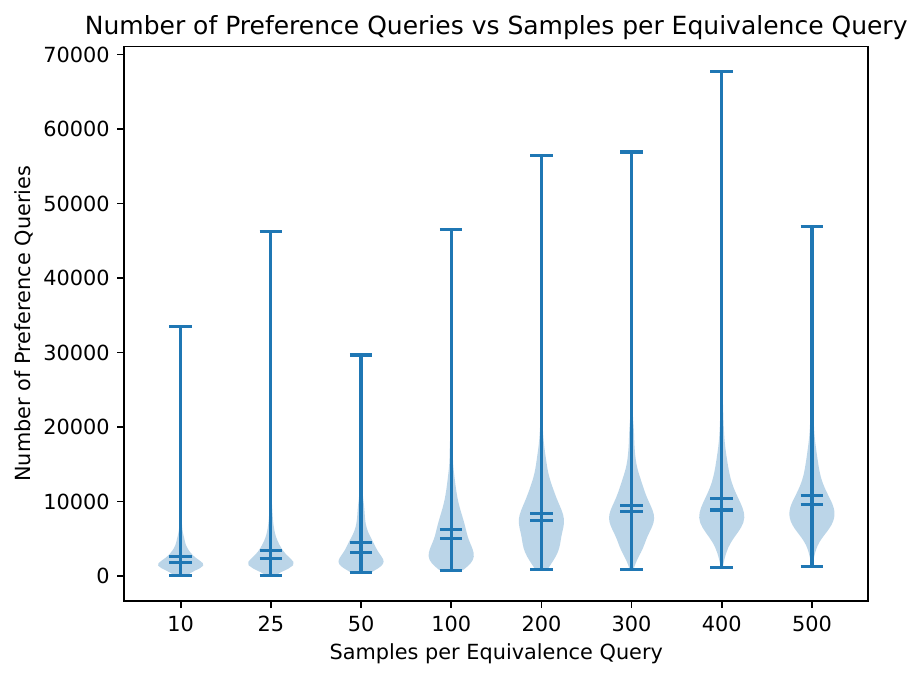} &
    \includegraphics[width=0.19\textwidth]{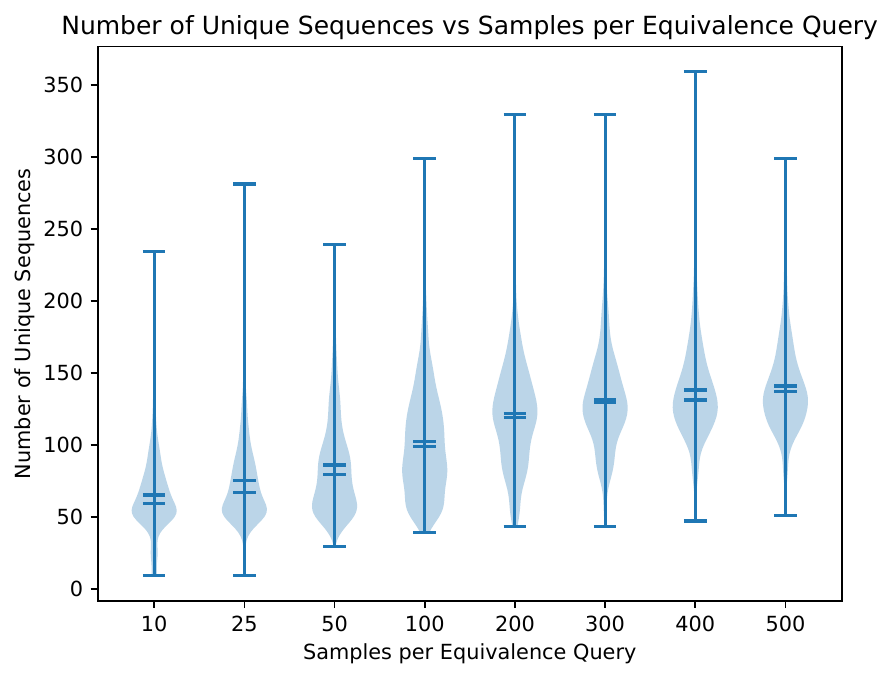}\\
    \end{tblr}
    \caption{Violin Plots of Terminal Hypothesis Accuracy, Number of Equivalence Queries executed, Number of Preference Queries executed, and Number of Unique Sequences Tested, as a function of Maximum Number of Random Tests performed per Equivalence Query, for learning a variety of regexes using a two-letter alphabet. During an Equivalence Query, the teacher samples $N$ random sequences, with the length of each sequence drawn iid from a geometric distribution with termination probability $0.2$. Given a sequence test length, each element of the sequence is chosen iid from a uniform distribution over the alphabet. To compute the terminal hypothesis accuracy, the terminal hypothesis is tested on $200$ random sequences per regex class. A total of $2000$ trials were conducted per $N$ value; the plots show the resulting empirical distributions of those trials.}
    \label{fig:violin_plots}
\end{figure*}

\begin{figure*}
    \centering
    \includegraphics[width=\textwidth]{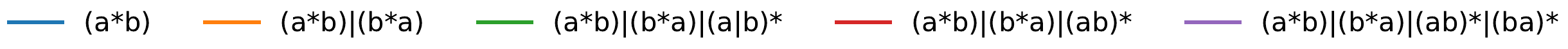}\\
    \includegraphics[height=0.245\textwidth]{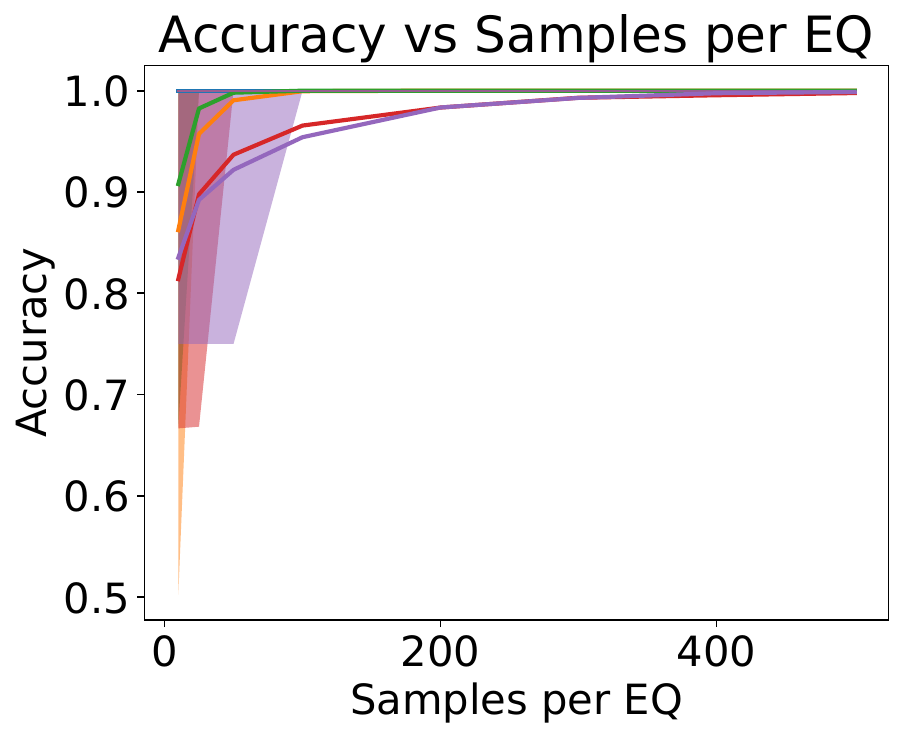}
    \includegraphics[height=0.245\textwidth]{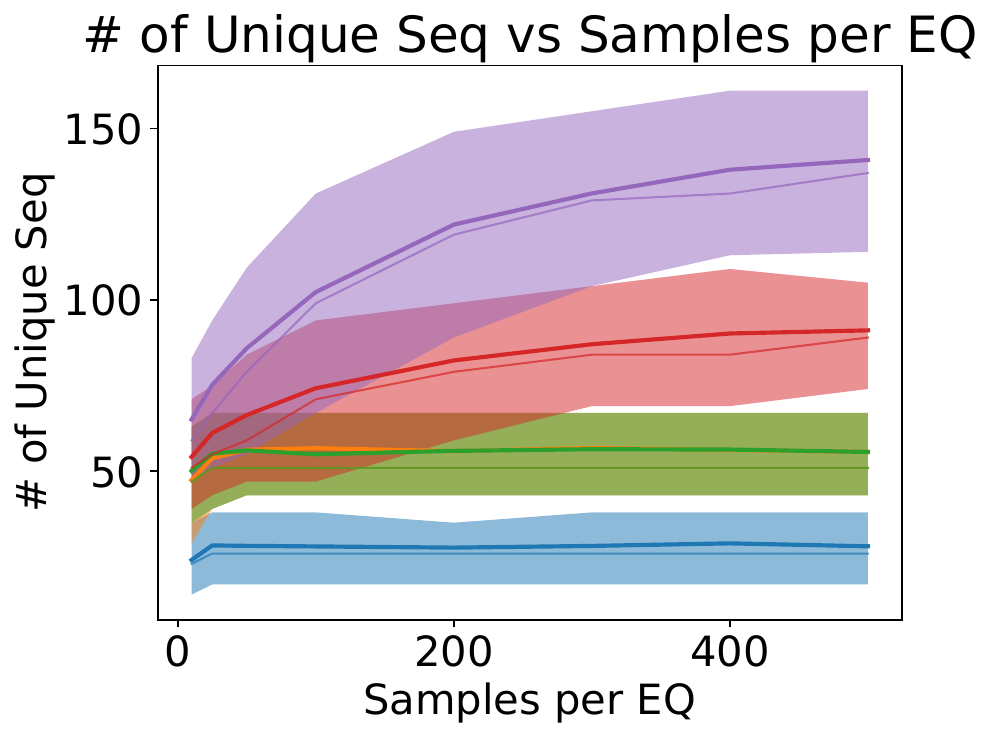}\\
    \includegraphics[height=0.245\textwidth]{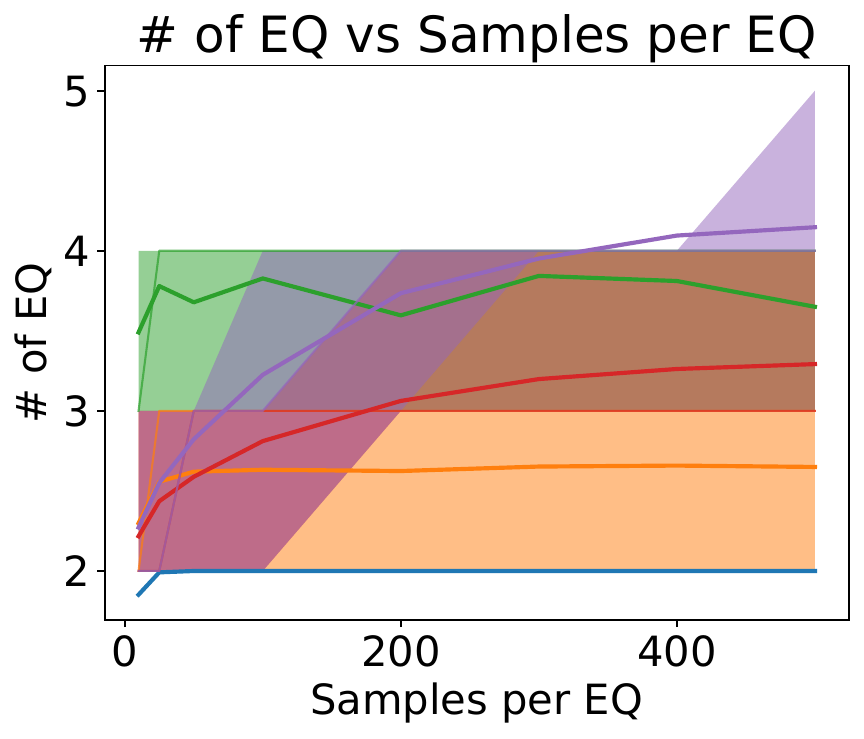}
    \includegraphics[height=0.245\textwidth]{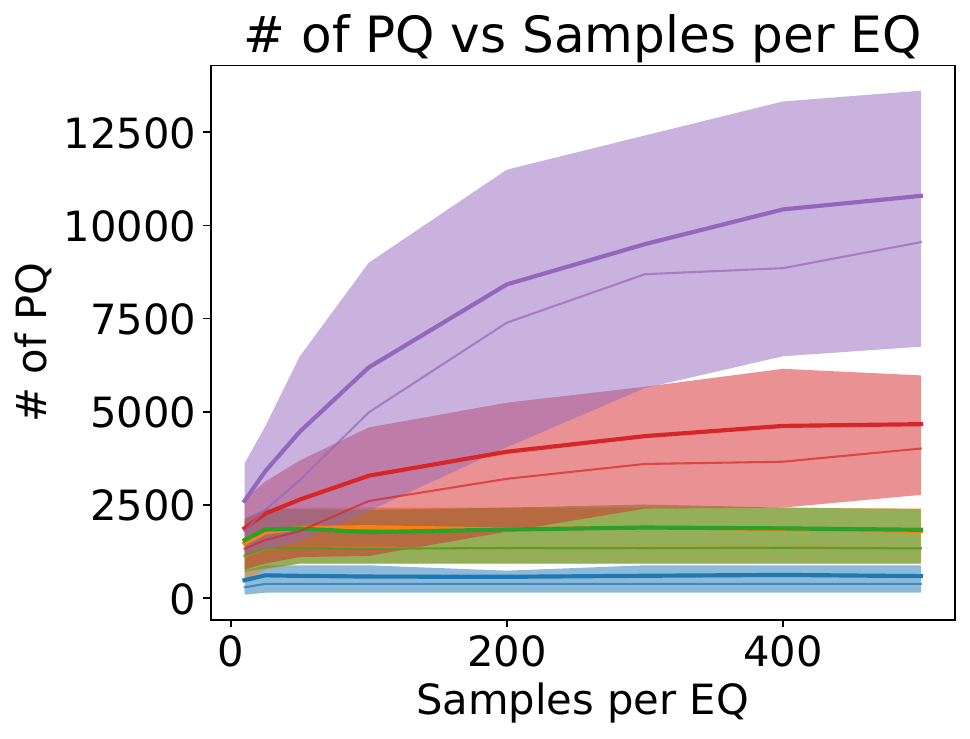}
    \caption{Plots of the same data as in Figure \ref{fig:violin_plots}. Mean values are represented by thick lines, median values are represented by thin lines, and values lying between the $20^{th}$ and $80^{th}$ percentiles are shown as shaded.}
    \label{fig:comparison_plots}
\end{figure*}

\begin{figure*}
    \centering
    \includegraphics[width=0.8\textwidth]{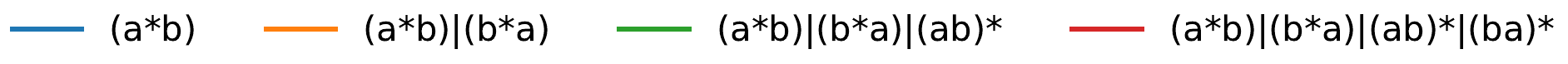}\\
    \includegraphics[height=0.245\textwidth]{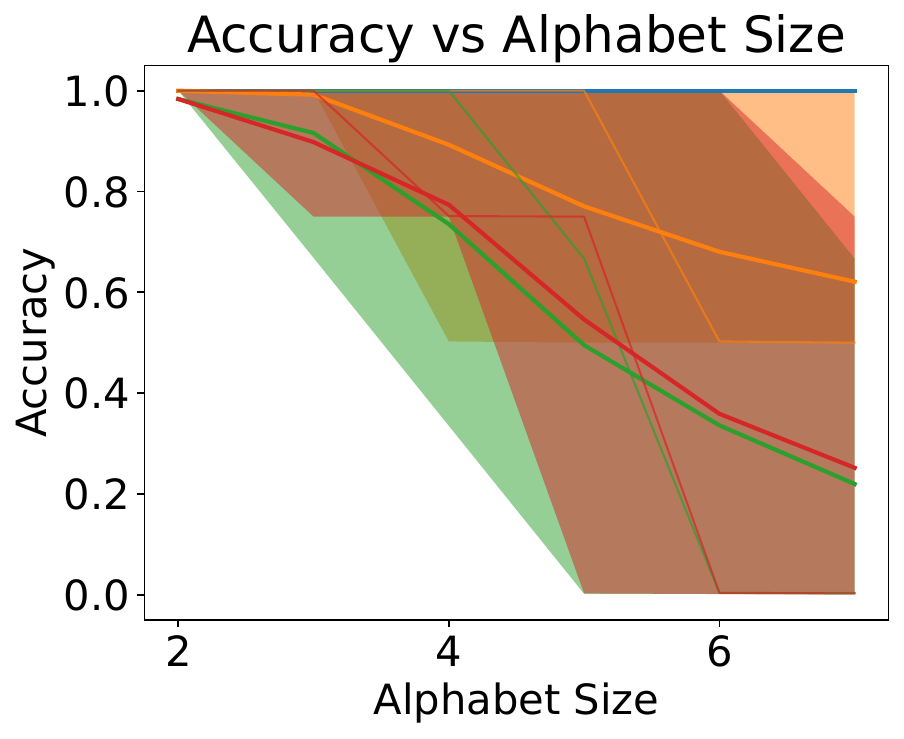}
    \includegraphics[height=0.245\textwidth]{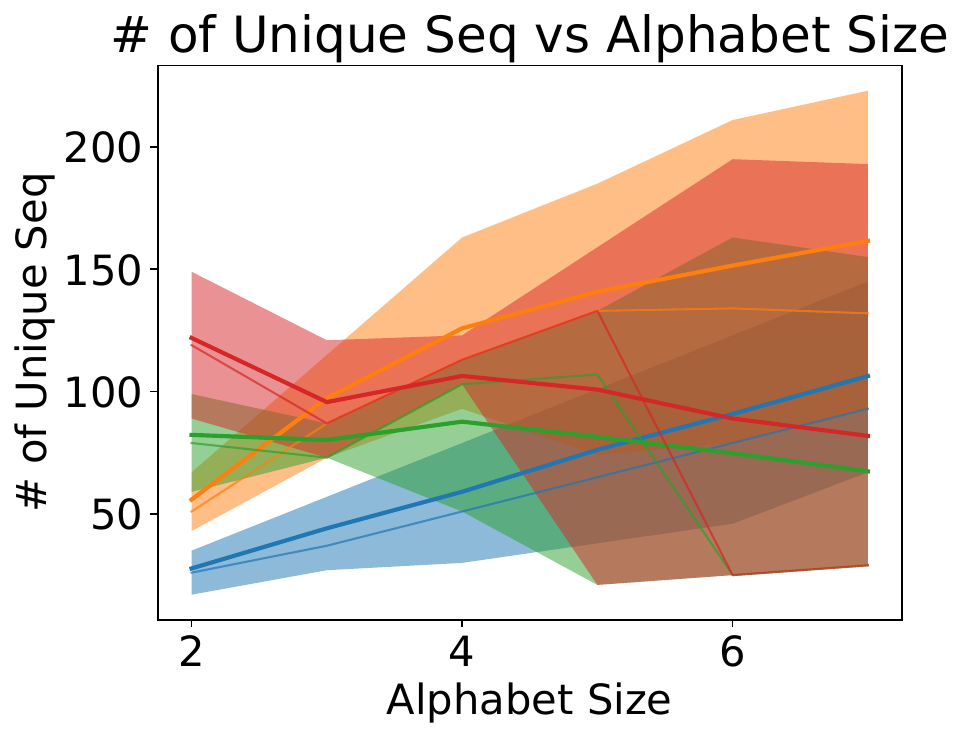}\\
    \includegraphics[height=0.245\textwidth]{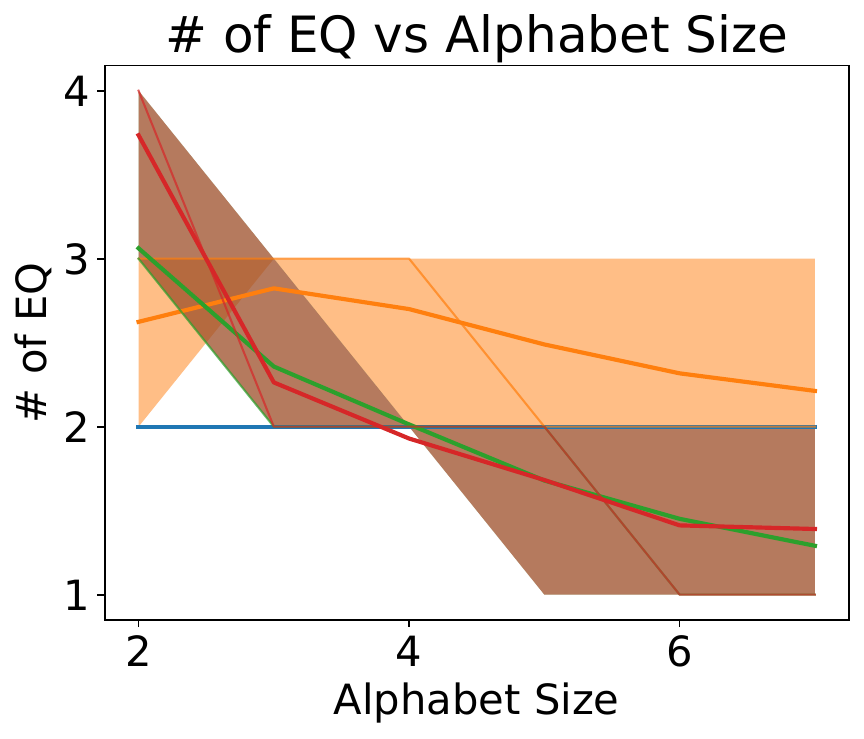}
    \includegraphics[height=0.245\textwidth]{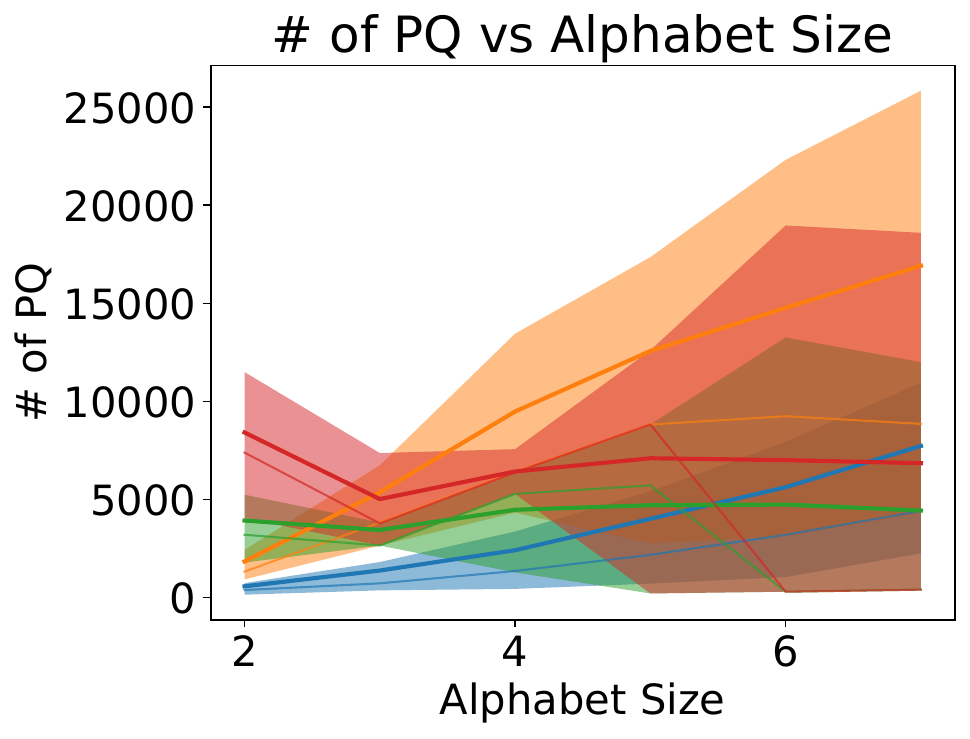}
    \caption{Plots accuracy, number of unique sequences tested, number of equivalence queries performed, and number of preference queries made, as a function of alphabet size. These plots are set at 200 samples per EQ. Mean values are represented by thick lines, median values are represented by thin lines, and values lying between the $20^{th}$ and $80^{th}$ percentiles are shown as shaded.}
    \label{fig:comparison_alpha_size_plots}
\end{figure*}

\begin{figure*}
    \centering
    \includegraphics[height=0.24\textwidth]{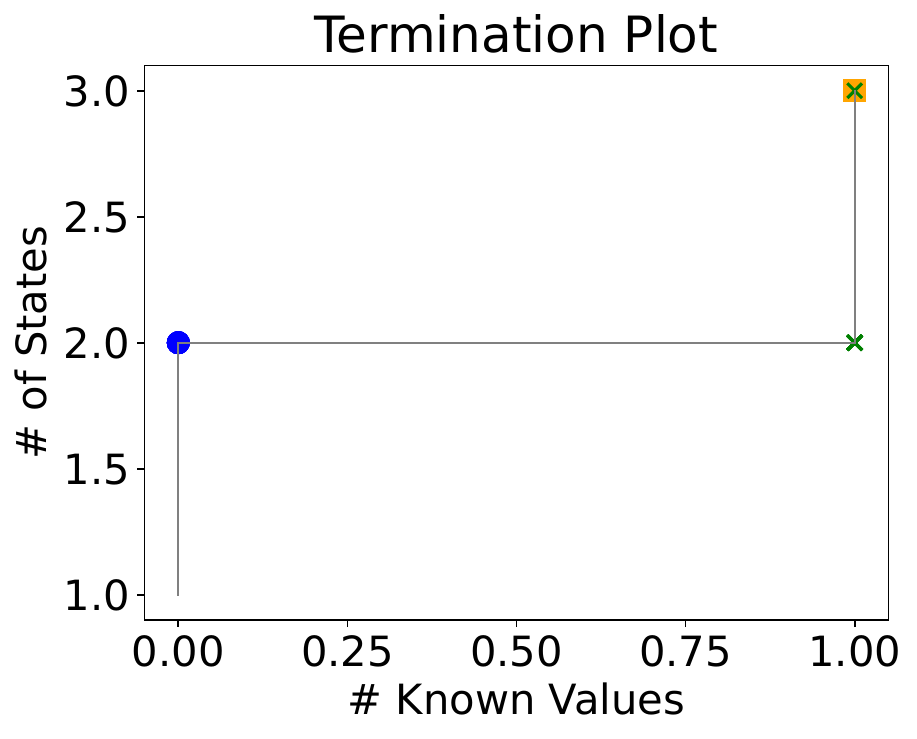}
    \includegraphics[height=0.24\textwidth]{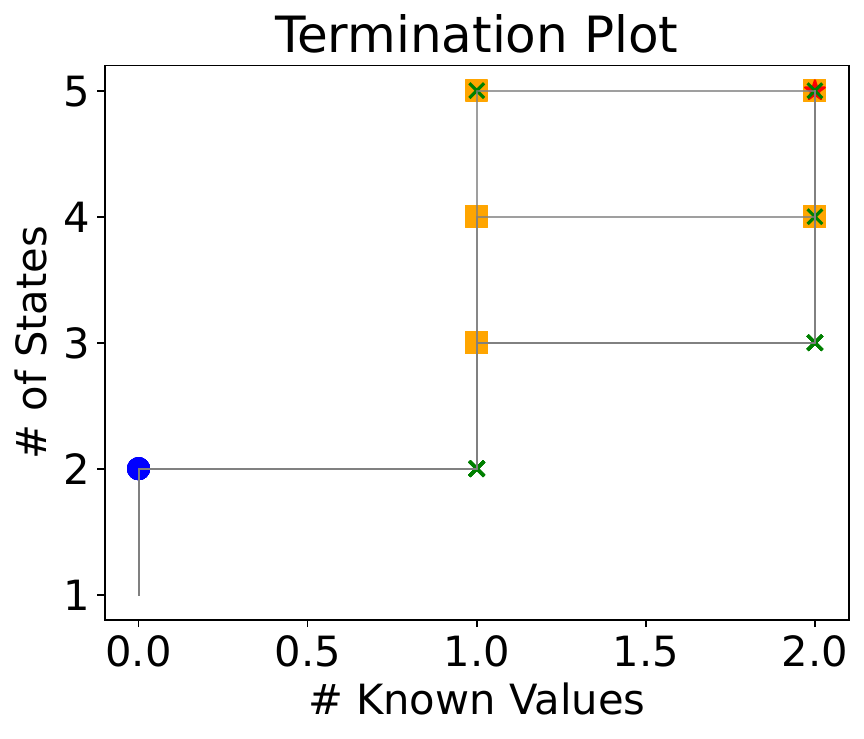}
    \includegraphics[height=0.24\textwidth]{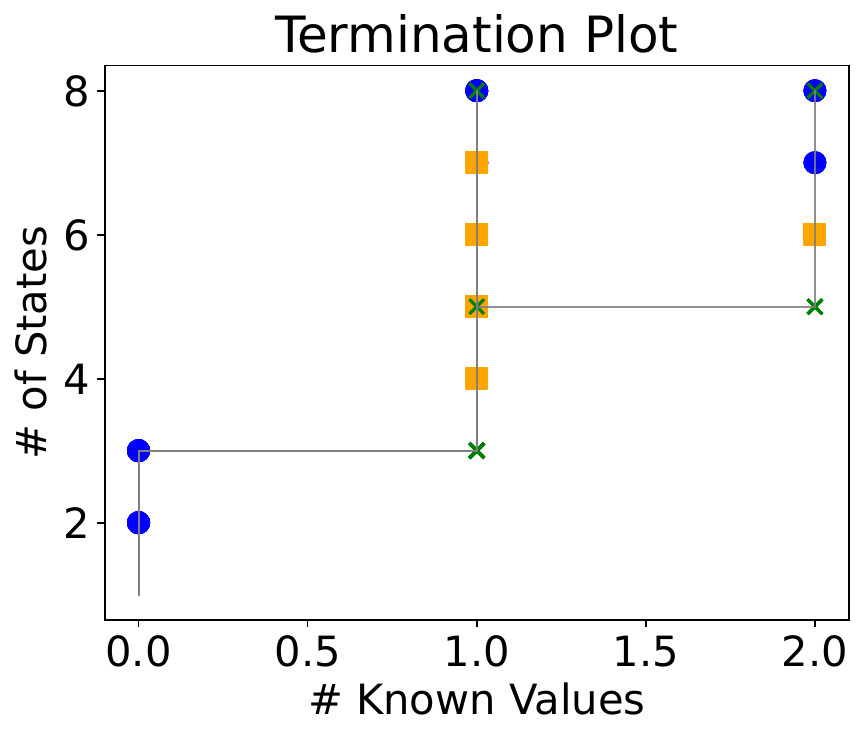}
    \includegraphics[height=0.24\textwidth]{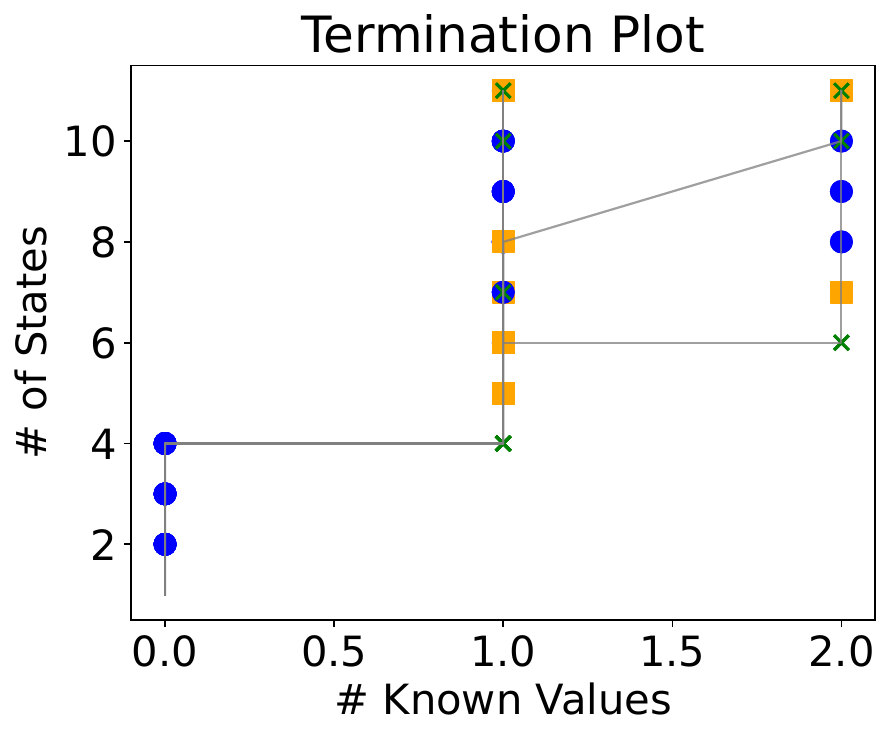}
    \includegraphics[height=0.24\textwidth]{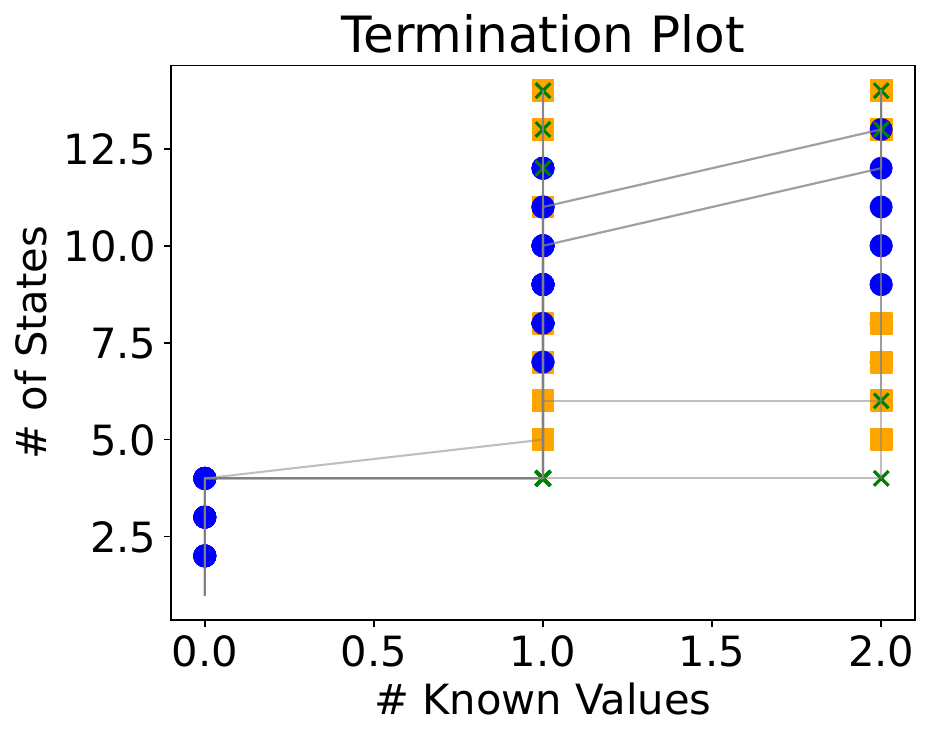}
    \caption{Termination plots illustrating that number of state and number of known variables increase monotonically. The blue circles represent closure operations, orange squares represent consistency operations, and green Xs represent equivalence queries. The number of states in the automata from left to right is 3, 5, 8, 11, 14. We observe that for termination to occur, the number of states must be correct. The number of explicitly known values does not need to match the number of classes to terminate, however, because the learner can simply ``guess'' the correct values for early termination.}
    \label{fig:termination_plots}
\end{figure*}
\end{document}